\documentclass{article}

\usepackage{microtype}
\usepackage{graphicx}
\usepackage{subfigure}
\usepackage{booktabs} % for professional tables

\usepackage{hyperref}

\usepackage[accepted]{icml2024}

\usepackage{amsmath}
\usepackage{amssymb}
\usepackage{mathtools}
\usepackage{amsthm}

\usepackage[capitalize,noabbrev]{cleveref}

\theoremstyle{plain}
\newtheorem{theorem}{Theorem}[section]

\newtheorem{lemma}[theorem]{Lemma}

\theoremstyle{definition}

\theoremstyle{remark}
\newtheorem{remark}[theorem]{Remark}

\usepackage[textsize=tiny]{todonotes}

\usepackage{amssymb}
\usepackage{xfrac}
\usepackage{upgreek}
\usepackage{cleveref}
\usepackage{multirow}
\usepackage{makecell}
\usepackage{enumitem}

\usepackage{titletoc}
\usepackage{hyperref}
\usepackage[page,header]{appendix}

\newcommand\TODO[1][]{{\color{orange}[TODO\ifthenelse{\equal{#1}{}}{}{: #1}]}}

\renewcommand\algorithmicrequire{\textbf{Input:}}
\renewcommand\algorithmicensure{\textbf{Output:}}

\theoremstyle{plain}
\newtheorem{ASS}{Assumption}

\newcommand\SEC\section
\newcommand\SSEC\subsection
\newcommand\SSSEC\subsubsection

\newcommand\RM[1]{\mathrm{#1}}

\newcommand\DD[1]{\mathop{\RM d{#1}}}

\newcommand\BM[1]{\boldsymbol{#1}}
\newcommand\BB[1]{\mathbb{#1}}
\newcommand\CAL[1]{\mathcal{#1}}
\newcommand\Cp{^\mathsf{C}}
\newcommand\TP{^\mathsf{T}}
\newcommand\Tp{\TP}

\newcommand\OP[1]{\operatorname{#1}}
\newcommand\Poly{\operatorname{poly}}

\newcommand\GA[1]{\begin{gather}#1\end{gather}}

\newcommand\AL[1]{\begin{align}#1\end{align}}
\newcommand\AM[1]{\begin{align*}#1\end{align*}}
\newcommand\ALN[2]{\AL{#2\label{#1}}}
\newcommand\HAT[1]{\widehat{#1}}

\newcommand\TLD[1]{\widetilde{#1}}
\newcommand\Prb{\mathbb{P}}
\newcommand\Exp{\mathbb{E}}
\newcommand\SM[1]{\setminus\{#1\}}

\DeclareMathOperator*{\ExpOp}{\Exp}
\DeclareMathOperator*{\argmin}{\arg\min}
\DeclareMathOperator*{\argmax}{\arg\max}

\makeatletter
\newcommand{\LABEL}[2]{\protected@write\@auxout{}{\string\newlabel{#1}{{#2}{\thepage}{#2}{#1}{}}}\hypertarget{#1}{}}
\makeatother
\newcommand\ALIST[1]{\begin{itemize}[topsep=0em,partopsep=0em,parsep=0em,itemsep=0em]
#1\end{itemize}}
\newcommand\AITEM[1]{\item[\ref{#1}]\ref{#1::name}\dotfill\pageref{#1}}
\newcommand\AOBJ[3]{\LABEL{#3::name}{#2}#1{#2}\label{#3}}
\newcommand\ASEC[2]{\AOBJ{\section}{#1}{#2}}
\newcommand\ASSEC[2]{\AOBJ{\subsection}{#1}{#2}}
\newcommand\Ours{GraCe}
\newcommand\OursFull{\underline{Gra}dient \underline{C}ompressed S\underline{e}nsing}

\icmltitlerunning{Gradient Compressed Sensing: A Query-Efficient Gradient Estimator for High-Dimensional Zeroth-Order Optimization}

\begin{document}

\twocolumn[
\icmltitle{Gradient Compressed Sensing: A Query-Efficient Gradient Estimator for High-Dimensional Zeroth-Order Optimization}



\icmlsetsymbol{equal}{*}

\begin{icmlauthorlist}
\icmlauthor{Ruizhong Qiu}{uiuc}
\icmlauthor{Hanghang Tong}{uiuc}
\end{icmlauthorlist}

\icmlaffiliation{uiuc}{Department of Computer Science, University of Illinois Urbana--Champaign, USA}

\icmlcorrespondingauthor{Ruizhong Qiu}{rq5@illinois.edu}
\icmlcorrespondingauthor{Hanghang Tong}{htong@illinois.edu}

\icmlkeywords{Zeroth-Order Optimization, High Dimensions, Compressed Sensing}

\vskip 0.3in
]



\printAffiliationsAndNotice{}  % leave blank if no need to mention equal contribution

\begin{abstract}
We study nonconvex zeroth-order optimization (ZOO) in a high-dimensional space $\BB R^d$ for functions with approximately $s$-sparse gradients. %Due to inaccurate gradient estimates, full-gradient ZOO methods suffer from a dimension-dependent rate of convergence in high-dimensional ZOO. %\TODO Due to inaccurate gradient estimates, existing \hh{high-dimensional? be specific}\RZ{added}high-dimensional ZOO methods either suffer from dimension-\emph{dependent} rates of convergence or even fail to converge to a stationary point. 
To reduce the dependence on the dimensionality $d$ in the query complexity, high-dimensional ZOO methods seek to leverage gradient sparsity to design gradient estimators. The previous best method needs $O\big(s\log\frac ds\big)$ queries per step to achieve $O\big(\frac1T\big)$ rate of convergence w.r.t.\ the number $T$ of steps.
In this paper, we propose \underline{Gra}dient \underline{C}ompressed S\underline{e}nsing (GraCe), a query-efficient and accurate estimator for sparse gradients that uses only $O\big(s\log\log\frac ds\big)$ queries per step and still achieves $O\big(\frac1T\big)$ rate of convergence. To our best knowledge, we are the first to achieve a \emph{double-logarithmic} dependence on $d$ in the query complexity, and our proof uses weaker assumptions than previous work.
Our proposed GraCe generalizes the Indyk--Price--Woodruff (IPW) algorithm in compressed sensing from linear measurements to nonlinear functions. Furthermore, since the IPW algorithm is purely theoretical due to its impractically large constant, we improve the IPW algorithm via our \emph{dependent random partition} technique together with our corresponding novel analysis and successfully reduce the constant by a factor of nearly 4300. %Our constant is nearly 4300 times smaller than that of the IPW algorithm.
%Instead of using simple, non-adaptive queries, our key idea is to carefully design queries to adaptively locate large-gradient dimensions. which is a generalization of the Indyk--Price--Woodruff algorithm in compressed sensing. %We then apply the proposed \Ours{} method to ZOO, yielding an algorithm achieving %$\Exp[\|\nabla f(\BM x_\tau)\|_2^2]\le O\big(\frac1T\big)$ rate \hh{i do not understand this sentence}\RZ{in nonconvex optimization, people typically care about finding a stationary point (including local minima and saddle points), so a common measure of the rate of convergence is $\|\nabla f(\BM x_T)\|_2^2$. How to escape from saddle points is another research direction}
%an $O\big(\frac1T\big)$. %\hh{looks like we never explain T}\RZ{added below} rate of convergence w.r.t.\ the number $T$ of steps for nonconvex ZOO. %, which is the same as that of first-order gradient descent.
%To our best knowledge, we are the first to achieve a dimension-\emph{independent} rate of convergence that matches that of first-order gradient descent with only \emph{double-logarithmic} dependence on $d$ in the query complexity.
Our GraCe is not only theoretically query-efficient but also achieves strong empirical performance. We benchmark our GraCe against 12 existing ZOO methods with 10000-dimensional functions and demonstrate that GraCe significantly outperforms existing methods. Our code is publicly available at \url{https://github.com/q-rz/ICML24-GraCe}.
\end{abstract}
\begin{table*}[t]
\caption{Comparison in nonconvex ZOO. ($q,b$: hyperparameters; $\tau:=\argmin_{t=1,\dots,T}\|\nabla f(\BM x_t)\|_2$; w.h.p.: with high probability.) To our best knowledge, we are the first to achieve a \emph{double-logarithmic} dependence on $d$ in the query complexity %with improved convergence
under weaker assumptions.}\label{tab:bounds}
\begin{center}
\begin{small}
\begin{tabular}{clll}
\toprule
\multicolumn{1}{c}{\textbf{Type}}&\multicolumn{1}{l}{\textbf{Method}}&\multicolumn{1}{l}{\textbf{Queries per step}}&\multicolumn{1}{l}{\textbf{Rate of convergence}}\\
\midrule
\multirow{6.8}*{\makecell[c]{Full\\Gradient}}&RS \cite{zo-rs}
&$O(1)$&$\Exp[\|\nabla f(\BM x_\tau)\|_2^2]\le O\big(\frac{\sqrt d}{\sqrt T}+\frac dT\big)$\\
&TPGE \cite{tpge}
&$O(1)$&$\Exp[\|\nabla f(\BM x_\tau)\|_2^2]\le O\big(\frac{\sqrt d}{\sqrt T}\big)$\\
&RSPG \cite{rspg}
&$O(q)$&$\Exp[\|\nabla f(\BM x_\tau)\|_2^2]\le O\big(\frac dq+\frac{d^2}{qT}\big)$\\
&ZO-signSGD \cite{zo-signsgd}
&$O(bq)$&$\Exp[\|\nabla f(\BM x_\tau)\|_2]\le O\big(\frac{\sqrt d\sqrt{q+d}}{\sqrt{bq}}+\frac{\sqrt d}{\sqrt T}\big)$\\
&ZO-AdaMM \cite{zo-adamm}
&$O(1)$&$\Exp[\|\nabla f(\BM x_\tau)\|_2^2]\le O\big(\frac d{\sqrt T}+\frac{d^2}T\big)$\\
%\midrule
%&SparseSZO \cite{sparseszo}&$\Exp[\|\nabla f(\BM x_T)\|_2^2]\le O\big(1?+\frac{s}T\big)$&\TODO\\
\midrule
\multirow{2.1}*{\makecell[c]{Sparse\\Gradient}}&ZORO \cite{zoro}
&$O\big(s\log\frac ds\big)$&$\|\nabla f(\BM x_\tau)\|_2^2\le O\big(\frac1T\big)$ w.h.p.\\
&\textbf{\Ours} (ours)
&$O\big(s\log\log\frac ds\big)$&$\|\nabla f(\BM x_\tau)\|_2^2\le O\big(\frac1{T}\big)$ w.h.p.\\
\bottomrule
\end{tabular}
\end{small}
\end{center}\end{table*}

\section{Introduction}
We study the problem of unconstrained optimization:
\AL{\min_{\BM x\in\BB R^d}f(\BM x),}
where $f:\BB R^d\to\BB R$ is a (possibly nonconvex) function over a high-dimensional space $\BB R^d$. %\hh{do we still plan to cover convex setting or we will focus on nonconvex only in this paper?}\RZ{I plan to focus on the nonconvex case}

Gradient-free optimization (GFO), also known as black-box optimization, was among the first schemes explored in the history of optimization theory \cite{matyas1965random}. In GFO, the function $f$ is unknown to the optimizer, and the optimizer can obtain information about the function $f$ only via \emph{queries} (i.e., function evaluations). The goal of GFO is to optimize the function $f$ using a minimal number of queries.

Zeroth-order optimization (ZOO), a paradigm of GFO, aims to apply first-order optimization methods to GFO with gradients estimated from queries. Through a long history of study, various \emph{full-gradient} ZOO methods have been proposed.  Early full-gradient ZOO methods such as the Kiefer--Wolfowitz method \cite{kiefer1952stochastic} use dimension-wise finite difference to approximate gradients, which suffers from an $O(d)$ dependence in the query complexity. Later methods achieve $O(1)$ queries per step via stochastic gradient estimators such as Gaussian smoothing \cite{nesterov2015random}, but they suffer from a $\Poly(d)$ factor in their rates of convergence. Hence, their overall query complexity still depends polynomially on $d$.%Furthermore, \citet{jamieson2012} have established an information-theoretic lower bound showing that, without additional assumptions, the dependence on $\Poly(d)$ is unavoidable in ZOO.

Meanwhile, the dimensionality $d$ can be very large in modern real-world applications. For instance, a high-resolution image can have millions of pixels. These real-world scenarios call for \emph{high}-dimensional ZOO. High-dimensional ZOO aims to develop gradient estimators with minimal dependence on the dimensionality $d$ in the query complexity under gradient sparsity assumptions \cite{zo-lasso,zoro}. In contrast to full-gradient ZOO, the research on \emph{sparse-gradient} ZOO is still in its infancy. Existing methods suffer from slow convergence and/or a suboptimal query complexity. For example, %\hh{cite or citet? one does not put a parenthesis outside year and the other does. let's choose one and use it consistently}\RZ{I use citet here because ``Wang et al.'' is the subject of the sentence. Grammatically, citet should be used when the author is part of the sentence; cite should be used when the author is not part of the sentence. See the ``short quotations'' section in \url{https://owl.purdue.edu/owl/research_and_citation/apa_style/apa_formatting_and_style_guide/in_text_citations_the_basics.html} }\hh{ok}
\citet{zo-lasso} proposed a LASSO-based method that uses $O(s^{2/3}\sqrt T)$ queries per step and has $O\big(\frac{s^2\sqrt{\log d}}{T^{1/3}}\big)+{\TLD O}\big(\frac1{T^{5/12}}\big)$ rate of convergence for stochastic convex ZOO. For nonconvex ZOO, the previous best method ZORO \cite{zoro}, which is based on CoSaMP \cite{needell2009cosamp}, needs $O\big(s\log\frac ds\big)$ queries per step to achieve $O\big(\frac1T\big)$ rate\footnote{This rate is for non-stochastic nonconvex ZOO. For stochastic nonconvex ZOO, ZORO has $O\big(1+\frac1{\sqrt T}\big)$ rate of convergence.} of convergence. %, which means that it might fail to converge to a stationary point \hh{is this because of the O(1) factor?}\RZ{yes. do I need to elaborate more on this?}.
The root cause of the technical difficulty here lies in the inaccurate gradient estimation in existing methods. 

In this paper, we propose \emph{\OursFull{}} (\Ours), a new sparse gradient estimator that uses only $O\big(s\log\log\frac ds\big)$ queries per step %for approximately $s$-sparse gradients
and still achieves $O\big(\frac1T\big)$ rate of convergence for nonconvex ZOO. It %inspired by an idea by \citet{ba2010lower} to encode dimension information into queries and is
generalizes the Indyk--Price--Woodruff (IPW) algorithm \cite{indyk2011power} in compressed sensing from linear measurements to nonlinear functions. Our main contributions are as follows:
\begin{itemize}
\item\textbf{Query-efficient gradient estimator.} We propose \emph{\OursFull{}} (\Ours), a new gradient estimator that uses only $O\big(s\log\log\frac ds\big)$ queries per step and still achieves $O\big(\frac1T\big)$ rate of convergence for nonconvex ZOO. To our best knowledge, we are the first to achieve a \emph{double-logarithmic} dependence on $d$ in the query complexity (see Table~\ref{tab:bounds}).
\item\textbf{Relaxed sparsity assumption.} Our analysis is based on a new assumption of approximate gradient sparsity, which is weaker than previous assumptions --- \emph{exact sparsity} \cite{zo-lasso} and \emph{compressibility} \cite{zoro}. %Besides that, our analysis does not assume Hessian sparsity while some previous works \cite{zo-lasso,zoro} do.
\item\textbf{Improvement of the IPW algorithm.} The IPW algorithm is purely theoretical due to its impractically large constant. To make the IPW algorithm practical, we improve the IPW algorithm via our \emph{dependent random partition} technique together with our corresponding novel analysis and successfully reduce the constant by a factor of nearly 4300.
%we improve the IPW algorithm via a new technique that we call \emph{dependent random partition} and provide corresponding novel analysis. Our constant is nearly 4300 times smaller than that of the IPW algorithm. %\TODO[constant] \Ours{} leverages a new technique called \emph{dependent random partition} to control the variance of the estimator and requires novel analysis on the effect of dependence on the variance. \hh{be  more specific of 'novel analysis', e.g., 'novel analysis in terms of xxx, and xxx'.}\RZ{added above}\TODO[discuss] %It is a generalization of the Indyk--Price--Woodruff method \cite{indyk2011power} in compressed sensing. 
%We show that our \Ours{} achieves $O\big(\frac1T\big)$ rate of convergence for nonconvex ZOO under our relaxed sparsity assumption, which is owing to the accurate gradient estimation by \Ours{}. This is in stark contrast to prior work such as \citet{zoro}, which fails to converge to a stationary point.\TODO[revise]
\item\textbf{Strong empirical performance.} Our \Ours{} is not only theoretically query-efficient but also achieves strong empirical performance. We benchmark our \Ours{} against 12 existing ZOO methods with 10000-dimensional functions and demonstrate that \Ours{} significantly outperforms existing methods.
\end{itemize}

%\hh{can we have a sentence (ideally in plain english without any equation) either in abs/intro and/or in the caption of table 1 to highlight our main contributions? e.g., 'All in all, under weaker assumptions, we are able to achieve a logarithmic factor reduction in query complexity with guaranteed covergence'}\RZ{added.}

%https://openreview.net/pdf?id=GXZ6cT5cvY
%https://proceedings.neurips.cc/paper_files/paper/2018/file/79a49b3e3762632813f9e35f4ba53d6c-Paper.pdf
%http://proceedings.mlr.press/v97/burt19a/burt19a.pdf
%15 Nesterov & S
\section{Preliminaries}

\subsection{Notation}
Throughout the paper, we use the bold font for vectors (e.g., $\BM x$) and the italic font for scalars (e.g., $x_i$). We use the same alphabet for a vector and its entries. 

For $i\in[d]$, let $\BM e_i:=[1_{[i'=i]}]_{i'\in[d]}\in\BB R^d$ denote the $i$-th standard basis of $\BB R^d$. For vectors $\BM u,\BM v\in\BB R^d$, let $\langle\BM u,\BM v\rangle:=\BM u\Tp\BM v$ denote the standard inner product. For a vector $\BM u\in\BB R^d$, let $\|\BM u\|_2:=\sqrt{\langle\BM u,\BM u\rangle}$ denote the Euclidean norm. 

For a dimension $i\in[d]$, let $\nabla_i$ denote the partial derivative operator w.r.t.\ the dimension $i$. For a subset $S\subseteq[d]$ of dimensions, let $\nabla_S:=[\nabla_i]_{i\in S}$ denote the partial derivative operator (as a column vector) w.r.t.\ dimensions $S$. Let $\nabla:=\nabla_{[d]}$ denote the gradient operator. 

For two finite sets $A,B$ with $|A|=|B|$, let $\CAL P_{A\to B}$ denote the set of bijections from $A$ to $B$. For instance, $\CAL P_{[d]\to[d]}$ is the set of permutations over $[d]$. For a finite set $A$, let $\mathop{\mathsf{Unif}}(A)$ denote the uniform distribution over $A$. 

\subsection{Assumptions}

We first introduce our new assumption on approximate gradient sparsity.

\begin{ASS}[Approximate gradient sparsity]\label{ASS:sparse}
The function $f$ has \emph{$\rho$-approximately $s$-sparse} gradients ($0<\rho\le1$, $1\le s\le d$):
\AM{\max_{\begin{subarray}{c}
I\subseteq[d]\,:\,|I|=s
\end{subarray}}{\|\nabla_If(\BM x)\|_2^2}\ge\rho\|\nabla f(\BM x)\|^2_2,\quad\forall\BM x.}
\end{ASS}

Our Assumption~\ref{ASS:sparse} is weaker than previous assumptions on gradient sparsity. The \emph{exact sparsity} assumption (i.e., $\|\nabla f(\BM x)\|_0\le s$) in \citet{zo-lasso} %simply 
corresponds to $\rho=1$ in our Assumption~\ref{ASS:sparse}. The \emph{compressibility} assumption (i.e., $\exists\kappa>1$ s.t.\ the $i$-th largest magnitude in $\nabla f(\BM x)$ is at most $i^{-\kappa}\|\nabla f(\BM x)\|_2$, $\forall i\in[d]$) in \citet{zoro} assumes the distribution of the entries of $\nabla f(\BM x)$ while our Assumption~\ref{ASS:sparse} does not assume the distribution; also, it implies our Assumption~\ref{ASS:sparse} with $\rho=1-\frac1{(2\kappa-1)s^{2\kappa-1}}$. Hence, our Assumption~\ref{ASS:sparse} is a relaxation of existing assumptions.

%Furthermore, our analysis does not make any assumptions on the Hessian while some previous works assume weak Hessian sparsity \cite{zo-lasso,zoro}. 

In addition to approximate gradient sparsity, we make the following standard assumptions on the function $f$. %See Lemmas~\ref{LEM:lip-smooth} \& \ref{LEM:lip-cont} for basic implications of Assumptions~\ref{ASS:lip-cont} \& \ref{ASS:lip-smooth}.

\begin{ASS}[Lower boundedness]
The function $f$ is lower-bounded:
\AM{f_*:=\inf_{\BM x}f(\BM x)>-\infty.}
\end{ASS}

\begin{ASS}[Lipschitz continuity]\label{ASS:lip-cont}
The function $f$ is $L_0$-Lipschitz continuous:
\AM{|f(\BM x+\BM u)-f(\BM x)|\le L_0\|\BM u\|_2,\quad\forall(\BM x,\BM u).}
\end{ASS}

\begin{ASS}[Lipschitz smoothness]\label{ASS:lip-smooth}
The function $f$ is differentiable and $L_1$-Lipschitz smooth:
\AM{\|\nabla f(\BM x+\BM u)-\nabla f(\BM x)\|_2\le L_1\|\BM u\|_2,\quad\forall(\BM x,\BM u).}
\end{ASS}

%\begin{figure}\centering\includegraphics[width=0.7\linewidth]{figs/fig-landscape.pdf}\vspace{-1em}\caption{Landscape and gradient sparsity \hh{either we add more arrows, or directly mark 'dense gradient' or 'sparse gradient' near the corresponding arrow}\RZ{sounds good. I will add more arrows}.}\label{fig:landscape}\end{figure}

%Finally, we remark that for real-world applications \hh{why do we need to emphasize 'for real-world applications'?}\RZ{Although our assumption is already relaxed, it is hard to find a function that can be proved to have approximately sparse gradients everywhere (except for trivial cases where functions have exactly sparse gradients). How can we express this precisely?}, as long as the assumptions hold along the optimization trajectory, our analysis applies. This \hh{what does 'this' refer to?}mainly depends on the landscape of $\OP{graph}f$ (illustrated in Fig.~\ref{fig:landscape}). To make our analysis as general as possible, we do not make \hh{any or specific?}\RZ{any} any restrictive assumptions on the landscape of \hh{graph?}\RZ{the graph of a function $f$ is the set $\{(\BM x,f(\BM x)):\BM x\in\OP{dom}f\}\subset\BB R^{d+1}$. It is a standard concept, but different papers may use different notations for it}$\OP{graph}f$. In Section~\ref{sec:exp}, we will further empirically validate and verify the efficacy of our method even when the assumptions might not hold for the entire space $\BB R^d$. 
\begin{algorithm}[t]
\caption{\OursFull{} (\Ours)}\label{alg:grace}
\begin{algorithmic}[1]
\REQUIRE{point $\BM x$; sparsity $s$; finite difference $\epsilon$; number $m$ of repeats; group size $n$; division schedule $\{D_r\}_{r\ge1}$}
\ENSURE{the gradient estimate $\BM g\in\BB R^d$}
\STATE candidate set $J\gets\varnothing$
\FOR{$l=1$ \textbf{to} $m$}
    \STATE random permutation $\omega\sim\mathop{\mathsf{Unif}}(\CAL P_{[d]\to[d]})$
    \FOR{$k=1$ \textbf{to} $\lceil d/n\rceil$}
        \STATE candidate group $S\gets\big\{i\in[d]:\big\lceil\frac{\omega(i)}{n}\big\rceil=k\big\}$
        \STATE iteration number $r\gets0$
        \REPEAT
            \STATE iteration number $r\gets r+1$
            \STATE random permutation $\varpi\sim\mathop{\mathsf{Unif}}(\CAL P_{S\to[|S|]})$
            \STATE block size $B\gets\big\lceil\frac{|S|}{D_r}\big\rceil$
            \STATE perturbations $\BM u\gets\BM0_d$, $\BM v\gets\BM0_d$
            \FOR{$i\in S$}
                \STATE random sign $\sigma_i\sim\mathop{\mathsf{Unif}}(\{\pm1\})$
                \STATE block label $h_i\gets\big\lceil\frac{\varpi(i)}{B}\big\rceil$
                \STATE perturbations $u_i\gets\epsilon\cdot\sigma_i$, $v_i\gets\epsilon\cdot\sigma_i\cdot h_i$
            \ENDFOR
            \STATE target $q\gets{\OP{round}}\big(\!\frac{f(\BM x+\BM v)-f(\BM x)}{f(\BM x+\BM u)-f(\BM x)}\!\big)$ via 2 queries
            \STATE candidate group $S\gets\{i\in S:h_i=q\}$
        \UNTIL{$|S|\le2$}
        \STATE candidate set $J\gets J\cup S$
    \ENDFOR
\ENDFOR
\STATE gradient estimate $\BM g\gets\BM 0_d$
\FOR{$j\in J$}
    \STATE finite difference $g_j\gets\frac{f(\BM x+\epsilon\BM e_j)-f(\BM x)}{\epsilon}$ via 1 query
\ENDFOR
\STATE\textbf{return} gradient estimate $\BM g$
\end{algorithmic}
\end{algorithm}

\section{\Ours: \OursFull}
In this section, we first propose a query-efficient method, \OursFull{} (\Ours), for estimating $\rho$-approximately $s$-sparse gradients using only $O\big(s\log\log\frac ds\big)$ adaptive queries.

Our %The proposed
\Ours{} generalizes the Indyk--Price--Woodruff (IPW) algorithm \cite{indyk2011power} in compressed sensing from linear measurements to nonlinear functions. %The idea of \Ours{} is illustrated in Figure~\ref{}\TODO.
First, \Ours{} randomly partitions the $d$ dimensions into $O(s)$ groups of size $O\big(\frac ds\big)$ so that (with high probability) each group has at most one large-gradient dimension. Then for each group, to locate the large-gradient dimension, \Ours{} constructs adaptive queries to iteratively shrink the candidate set of dimensions and finds the large-gradient dimension after $O\big({\log\log\frac ds}\big)$ iterations (with high probability). The procedure of \Ours{} is presented in Algorithm~\ref{alg:grace}. %A simplified yet equivalent implemetation is provided in Algorithm~\ref{alg:grace-simp}.

In the rest of this section, Section~\ref{ssec:grad-1} introduces how to design adaptive queries to locate the large-gradient dimension in a group, and Section~\ref{ssec:grad-s} describes how to divide the groups to achieve accurate gradient estimation with high probability. Proofs are deferred to Appendix~\ref{app:proof}.

\SSEC{Base case: Approximately $1$-sparse gradient}\label{ssec:grad-1}

Suppose that we have a candidate group $S\subseteq[d]$ in which there is only one dimension $j\in S$ with a large gradient $|\nabla_jf(\BM x)|$. We will introduce how to find $j$ with a small number of adaptive queries.

First, consider a motivating case: the \emph{signal-to-noise ratio} (SNR) $\frac{|\nabla_jf(\BM x)|}{\|\nabla_{S\SM j}f(\BM x)\|_2}$ is sufficiently large. Then, one can use an idea in \citet{ba2010lower} to encode dimension information into queries. Given a small $\epsilon>0$, define perturbations $\BM u',\BM v'\in\BB R^d$ by
\AL{u_i':=\epsilon\cdot 1_{[i\in S]},\;\;v_i':=\epsilon\cdot i\cdot1_{[i\in S]},\quad i\in[d].}
With a sufficiently large SNR $\frac{|\nabla_jf(\BM x)|}{\|\nabla_{S\SM j}f(\BM x)\|_2}$,
\ALN{eq:grad-1-base}{
\begin{split}
&\frac{f(\BM x+\BM v')-f(\BM x)}{f(\BM x+\BM u')-f(\BM x)}\approx\frac{\sum_{i\in[d]}v_i'\cdot\nabla_if(\BM x)}{\sum_{i\in[d]}u_i'\cdot\nabla_if(\BM x)}\\
={}&\frac{\sum_{i\in S}\epsilon\cdot i\cdot\nabla_if(\BM x)}{\sum_{i\in S}\epsilon\cdot\nabla_if(\BM x)}\approx\frac{\epsilon\cdot j\cdot\nabla_jf(\BM x)}{\epsilon\cdot\nabla_jf(\BM x)}=j.
\end{split}
}
In this case, we can find $j$ using only $O(1)$ queries by rounding $\frac{f(\BM x+\BM v')-f(\BM x)}{f(\BM x+\BM u')-f(\BM x)}$ to the nearest integer.

In general, however, it can happen that the SNR is not sufficiently large. To address this issue, the idea is iteratively increasing the SNR by identifying small-gradient dimensions and removing them from the candidate group $S$. As an improvement over the IPW algorithm, we introduce a technique that we call \emph{dependent random partition} for increasing the SNR: given a parameter $D$, we randomly divide $S$ into blocks of a fixed size $B:=\big\lceil\frac{|S|}{D}\big\rceil$ and label the blocks $1,\dots,\big\lceil\frac{|S|}{B}\big\rceil$. For each dimension $i\in S$, let $h_i$ denote the label of the block that $i$ belongs to. Each $i\in S$ is also assigned a random sign $\sigma_i\sim\mathop{\mathsf{Unif}}\{\pm1\}$. Then, define perturbations $\BM u,\BM v\in\BB R^d$ by
\AL{u_i:=\epsilon\cdot\sigma_i\cdot 1_{[i\in S]},\;v_i:=\epsilon\cdot\sigma_i\cdot h_i\cdot1_{[i\in S]},\;i\in[d].}
Intuitively, since all dimensions in the same block have the same label, then the ``signal'' of the label $h_j$ should be strengthened. Furthermore, although the labels $\{h_i\}$ are not mutually independent, their dependence is weakened by the random signs $\{\sigma_i\}$. Thus, the ``noises'' $h_i\ne h_j$ would not be strengthened. Hence, under suitable conditions,
%similarly with Eq.~\eqref{eq:grad-1-base}, we shall have $\frac{f(\BM x+\BM v)-f(\BM x)}{f(\BM x+\BM u)-f(\BM x)}\approx h_j$ under suitable conditions.
\AL{
&\frac{f(\BM x+\BM v)-f(\BM x)}{f(\BM x+\BM u)-f(\BM x)}\approx\frac{\sum_{i\in[d]}v_i\cdot\nabla_if(\BM x)}{\sum_{i\in[d]}u_i\cdot\nabla_if(\BM x)}\\
={}&\frac{\sum_{i\in S}\epsilon\sigma_ih_i\cdot\nabla_if(\BM x)}{\sum_{i\in S}\epsilon\sigma_i\cdot\nabla_if(\BM x)}\approx\frac{\sum_{\begin{subarray}{l}i\in S\\h_i=h_j\end{subarray}}\!\!\!\epsilon\sigma_i h_i\cdot\nabla_if(\BM x)}{\sum_{\begin{subarray}{l}i\in S\\h_i=h_j\end{subarray}}\!\!\!\epsilon\sigma_i\cdot\nabla_if(\BM x)}=h_j.\nonumber
}
Once we obtain $h_j$, we can shrink the candidate group to $S':=\{i\in S:h_i=h_j\}$, which has an increased SNR. This is formally stated in Lemma~\ref{LEM:shrink}. %\hh{lemma 3.1 seems to jump out of nowhere. add some transition text, e.g., 'Formally, we have ...'}\RZ{added above} %$\frac{|\nabla_jf(\BM x)|}{\|\nabla_{S'\SM j}f(\BM x)\|_2}$. %This is formally stated in Lemma~\ref{LEM:shrink}.

%\hh{in lemma 3.1, why 'absolute constant' (as opposed to just constant)? is it because it is the same C regardless of what kind of f functions we deal with?}\RZ{Yes, it is a standard term in math. An absolute constant means a number, like $1,2,e,\pi$. For example, the Lipschitz constant of a function is not an absolute constant.}

\begin{lemma}\label{LEM:shrink}
There is an absolute constant $C_1>0$ such that given $\BM x\in\BB R^d$, $\epsilon>0$, $S\subseteq[d]$, $0<\delta_1,\delta_2<1$, and integer $2\le D\le d$, if there exists $j\in S$ with $|\nabla_j f(\BM x)|>$
\ALN{eq:shrink-lem-0}{\!\big(C_1D+\tfrac1D\big)\!\sqrt{2\ln\tfrac3{\delta_1}}\|\nabla_{S\setminus\{j\}} f(\BM x)\|_2+\lambda_{1,|S|}\cdot\epsilon,}
then using $O(1)$ queries, with probability $\ge1-(\delta_1+\delta_2)$, we can find a subset $S'\subseteq S$ with $j\in S'$ and
\ALN{eq:shrink-lem-2}{\label{eq:shrink-lem-1}|S'\SM j|&\le\frac{|S\SM j|}{D},\\\|\nabla_{S'\SM j} f(\BM x)\|_2&\le\frac{\|\nabla_{S\setminus\{j\}} f(\BM x)\|_2}{\sqrt{D\delta_2}}.}
Here, $\lambda_{1,n}:=L_1\big(d^2+d+\frac12\big)n$.
\end{lemma}

Eq.~\eqref{eq:shrink-lem-0} quantifies the condition between the SNR and the division parameter $D$; Eq.~\eqref{eq:shrink-lem-1} shows that the size of the candidate group shrinks by $D$ times; and Eq.~\eqref{eq:shrink-lem-2} shows that the SNR $\frac{|\nabla_jf(\BM x)|}{\|\nabla_{S'\SM j}f(\BM x)\|_2}$ increases by a factor of $\sqrt{D\delta_2}$. Lemma~\ref{LEM:shrink} will be used next as the key subroutine in Lemma~\ref{LEM:one}.

Since the SNR has increased, we can repeat Lemma~\ref{LEM:shrink} with a larger division parameter $D$ for the next iteration. Let $D_r$ denote the division parameter for the $r$-th iteration. With the help of our \emph{dependent random partition} technique, the size of the candidate group shrinks rapidly. With the candidate set $S$ shrinking and the division parameter increasing, we shall have $|S\SM j|<D_r$ at some iteration. Then by Eq.~\eqref{eq:shrink-lem-1}, %we have
\AL{|S'\SM j|\le\frac{|S\SM j|}{D_r}<1,}
which implies that $S'$ contains $j$ only, i.e., $j$ is found. 

It remains to bound the number of iterations. This depends on the growth rate of %the division parameter 
$D_r$. We show in Lemma \ref{LEM:one} that $D_r$ can grow rapidly so that $O(\log\log|S|)$ iterations suffice. %, which is formally stated in Lemma~\ref{LEM:one}.

\begin{lemma}\label{LEM:one}
There exist absolute constants $C_2,C_3>0$, $A>1$, and a division schedule $\{D_r\}_{r\ge1}$ such that (i) $D_r\ge C_3A^{(3/2)^{r-1}}$, and (ii) given $\BM x\in\BB R^d$, $\epsilon>0$, and $S\subseteq[d]$, if there exists $j\in S$ such that
\ALN{eq:one-cond}{|\nabla_jf(\BM x)|>C_2\|\nabla_{S\SM j}f(\BM x)\|_2+\lambda_{1,|S|}\cdot\epsilon,}
then $O(\log_{3/2}\log_A|S|)$ iterations of Lemma~\ref{LEM:shrink} with parameters $\{D_r\}_{r\ge1}$ can find $j$ with probability at least $1/2$.
\end{lemma}

We provide the general version of Lemma~\ref{LEM:one} in Lemma~\ref{LEM:one-general}, which gives the exact relation between the failure probability and the absolute constants. Besides that, we recommend choosing division parameters via the recurrence $D_{r+1}:=\lfloor D_r^{3/2}\rfloor$ with an appropriate $D_1$ in practice. Note that the subset $S$ in Lemma~\ref{LEM:one} is different from the subset $I$ in Assumption~\ref{ASS:sparse}. We will show how to find such subsets $S$ in Section~\ref{ssec:grad-s}.

We remark that %owing to our \emph{dependent random partition} technique\TODO,
our constant $C_2\approx135$ is nearly $4300$ times smaller than the corresponding constant $C_2'\approx579263$ of the IPW algorithm. This is owing to our \emph{dependent random partition} technique and our corresponding novel analysis. %As $C_2'$ is impractically large,
In contrast to the purely theoretical IPW algorithm, our \Ours{} achieves strong empirical performance, which is demonstrated by our experiments in Section~\ref{sec:exp}. %This helps our \Ours{} perform empirically better than the IPW algorithm.

\SSEC{General case: Approximately $s$-sparse gradient}\label{ssec:grad-s}

Building upon the base case, next we describe how to partition the $d$ dimensions into groups so that most groups satisfy the condition Eq.~\eqref{eq:one-cond} in Lemma~\ref{LEM:one}.

Here we employ again our aforementioned technique \emph{dependent random partition}: given a parameter $n$, we randomly partition the $d$ dimensions into groups of a fixed size $n$. It remains to determine the group size $n$. On the one hand, the condition Eq.~\eqref{eq:one-cond} requires a group to have an SNR greater than an absolute constant $C_2$. This means that $S\SM j$ should not contain too many dimensions, so the group size $n$ should not be too large. On the other hand, the group size $n$ should not be too small. Otherwise, the number of groups would be too large, resulting in a large number of queries. For example, if $S\SM j=\varnothing$, the SNR would be $\infty$, but the overall query complexity would be $\Omega(d)$. We will show in Lemma~\ref{LEM:cand} that $n=\Theta\big(\frac ds\big)$ suffices, and we use repetition to ensure success with high probability. 

Combining the dimensions $j$ found in each group gives a candidate set $J\subseteq[d]$. We show in Lemma~\ref{LEM:cand} that $J$ is likely to contain most large-gradient dimensions. 

\begin{lemma}\label{LEM:cand}
Given $\BM x\in\BB R^d$, $\epsilon>0$, $0<\alpha<\rho$, and $0<\delta<1$, % and $\beta>0$ with $\alpha+(1-\rho)\beta<\rho$,
there exist hyperparameters for Algorithm~\ref{alg:grace} such that with probability at least $1-\delta$, it can use $O\big(s\log\log\frac ds\big)$ adaptive queries to find a set $J\subseteq[d]$ of size $O(s)$ such that 
\AL{\|\nabla_Jf(\BM x)\|_2^2\ge\alpha\|\nabla f(\BM x)\|_2^2-\lambda_{2,d}\epsilon-\lambda_{1,d}^2\epsilon^2,}
where $\lambda_{2,d}:=2L_0\lambda_{1,d}$. The $O$ notation hides constants that depend only on $\rho$, $\alpha$, and $\delta$.
\end{lemma}

Finally, for each candidate dimension $j\in J$, we estimate the gradient $\nabla_jf(\BM x)$ via finite difference:
\AL{g_j:=\frac{f(\BM x+\epsilon\BM e_j)-f(\BM x)}{\epsilon}.}
With the good candidate set $J$, we show in Theorem~\ref{THM:grad} that the direction of the gradient estimate $\BM g$ aligns well with that of the true gradient $\nabla f(\BM x)$.

%Fix hyperparameters $0<\alpha<\rho$ and $0<\delta<1$, and let $C_\rho:=C_{\rho,\alpha,\delta}$ for notational simplicity. 

\begin{theorem}\label{THM:grad}
Given $\BM x\in\BB R^d$, $\epsilon>0$, and $0<\alpha<\rho$, there exist hyperparameters for Algorithm~\ref{alg:grace} such that it can use $O\big(s\log\log\frac ds\big)$ adaptive queries to find a gradient estimate $\BM g\in\BB R^d$ such that with probability $1$,
\GA{\|\BM g\|_2\le\|\nabla f(\BM x)\|_2+O(L_1\sqrt s\epsilon),\label{eq:grad-norm}\\
\Exp[\langle\nabla f(\BM x),\BM g\rangle\mid\BM x]\ge\alpha\|\nabla f(\BM x)\|_2^2-\lambda_{3,d,s}\epsilon-\lambda_{4,d}\epsilon^2,\nonumber}
where $\lambda_{3,d,s}=O(\lambda_{2,d}+L_0L_1\sqrt s)$, $\lambda_{4,d}=O(\lambda_{1,d}^2)$.
The $O$ notation hides constants that depend only on $\rho$ and $\alpha$.
\end{theorem}

%\hh{Eq14 is out of width}\RZ{Yeah... I'm considering how to shorten it}
Theorem~\ref{THM:grad} shows that the inner product $\langle\nabla f(\BM x),\BM g\rangle$ is relatively large, and Eq.~\eqref{eq:grad-norm} shows that it is not due to an unbounded norm $\|\BM g\|_2$. Together, we can conclude that the gradient estimate $\BM g$ has high cosine similarity with the true gradient $\nabla f(\BM x)$. This property will be useful in improving the rate of convergence in nonconvex ZOO.
\begin{algorithm}[t]
\caption{Zeroth-order gradient descent with \Ours}\label{alg:zo-grace}
\begin{algorithmic}[1]
\REQUIRE{initial point $\BM x_1$; step size $\eta$; finite difference schedule $\{\epsilon_t\}_{t\ge1}$; hyperparameters for \Ours{}}
\ENSURE{optimized point}
\STATE step number $t\gets1$
\REPEAT
    \STATE gradient estimate $\BM g_t$ via \Ours{} with $(\BM x_t,\epsilon_t)$
    \STATE next point $\BM x_{t+1}\gets\BM x_t-\eta\BM g_t$
    \STATE step number $t\gets t+1$
\UNTIL{stopping criterion is met}
\STATE\textbf{return} $\argmin_{\BM x\in\{\BM x_1,\dots,\BM x_t\}}f(\BM x)$
\end{algorithmic}
\end{algorithm}

\section{Zeroth-Order Optimization with \Ours}
As a zeroth-order gradient estimator, \Ours{} can be applied to ZOO by integrating the estimated gradient into existing first-order methods. In this work, we consider zeroth-order gradient descent with \Ours{}.

Let $\BM x_1\in\BB R^d$ denote the initial point, let $\eta>0$ denote the step size, and let $\{\epsilon_t\}_{t\ge1}$ denote the finite difference schedule. At each iteration $t\ge1$, the algorithm finds a gradient estimate $\BM g_t\in\BB R^d$ using \Ours{} with $(\BM x_t,\epsilon_t)$ and performs a gradient descent step:
\AL{\BM x_{t+1}\gets\BM x_t-\eta_t\BM g_t.}
The overall procedure is presented in Algorithm~\ref{alg:zo-grace}. %\hh{In both Algorithms 1 and 2, it seem that neither of them uses function f (as the input)?}\RZ{Alg 1 uses f in lines 17 \& 25. Alg 2 uses f implicitly via Alg 1. I did not include f as an input because f is in our global assumptions}

Next, we analyze the rate of convergence of Algorithm~\ref{alg:zo-grace}. With the help of the accurate gradient estimation by \Ours{}, Algorithm~\ref{alg:zo-grace} achieves an $O\big(\frac1T\big)$ rate of convergence for finding a first-order stationary point in nonconvex ZOO. A comparison of nonconvex bounds is summarized in Table~\ref{tab:bounds}.

\begin{theorem}\label{THM:gd-nc}
Given any initial point $\BM x_1\in\BB R^d$ and any $\varDelta>0$, there exist a step size $\eta$, a finite difference schedule $\{\epsilon_t\}_{t\ge1}$ for Algorithm~\ref{alg:zo-grace}, and hyperparameters for \Ours{} such that for every $T\ge1$,
\AL{\Exp\Big[\min_{t=1,\dots,T}\|\nabla f(\BM x_t)\|_2^2\Big]\le\frac{\frac{2L_1}{\rho^2}(f(\BM x_1)-f_*)+\varDelta}T.}
\end{theorem}

The proof of Theorem~\ref{THM:gd-nc} is owing to Theorem~\ref{THM:grad}, which enables us to show a constant upper bound of the cumulative regret $\Exp[\sum_{t=1}^T\|\nabla f(\BM x_t)\|_2^2]$. Furthermore, we also provide a high-probability bound of convergence. % in Theorem~\ref{THM:gd-prob}.

\begin{theorem}\label{THM:gd-prob}
Given any initial point $\BM x_1\in\BB R^d$, any step size $0<\eta<\frac\rho{L_1}$, any $0<\beta<1$, and any $\varDelta>0$, there exist a finite difference schedule $\{\epsilon_t\}_{t\ge1}$ for Algorithm~\ref{alg:zo-grace} and hyperparameters for \Ours{} such that with probability at least $1-\beta$, for all $T\ge1$ simultaneously,
\AL{\!\min_{t=1,\dots,T}\!\|\nabla f(\BM x_t)\|_2^2\!\le\!\frac{\frac{1+\tfrac{2(1-L_1\eta)}{L_1\eta\beta}}{\eta-\tfrac{L_1\eta^2}2}(f(\BM x_1)\!-\!f_*)\!+\!\varDelta}T.}
\end{theorem}

In practice, we recommend using a constant $\epsilon$ for all $t$ in order to avoid underflow in floating point arithmetics. Our experiments demonstrate that a constant $\epsilon$ still works well. 
\begin{table*}[t]
\caption{Comparison among ZOO methods (mean$\,\pm\,$s.e.).}
\label{tab:exp-main}
\begin{center}\small%\resizebox{0.6\linewidth}{!}{
\begin{tabular}{cl|ccc}
\toprule
\multicolumn{1}{c}{\textbf{Type}}&\multicolumn{1}{c|}{\textbf{Method}}&\textsc{Distance}&\textsc{Magnitude}&\textsc{Attack}\\
\midrule
\multirow{6}*{\makecell[c]{Full\\Gradient}} & RS & 0.66326\,{$\pm$\,0.00780} & 0.91847\,{$\pm$\,0.00140} & 0.41310\,{$\pm$\,0.00048} \\
 & TPGE & 0.75618\,{$\pm$\,0.00680} & 0.96111\,{$\pm$\,0.00110} & 0.42757\,{$\pm$\,0.00310} \\
 & RSPG & 0.47299\,{$\pm$\,0.00994} & 0.69877\,{$\pm$\,0.00228} & 0.46512\,{$\pm$\,0.00115} \\
 & ZO-signSGD & 0.93413\,{$\pm$\,0.00823} & 0.98787\,{$\pm$\,0.00024} & 0.81652\,{$\pm$\,0.00046} \\
 & ZO-AdaMM & 0.80454\,{$\pm$\,0.01442} & 0.97235\,{$\pm$\,0.00076} & 0.59624\,{$\pm$\,0.00747} \\
 & GLD & 0.85677\,{$\pm$\,0.00436} & 0.98267\,{$\pm$\,0.00074} & 0.85497\,{$\pm$\,0.00172} \\
\midrule
\multirow{7}*{\makecell[c]{Sparse\\Gradient}} & LASSO & 0.47432\,{$\pm$\,0.00873} & 0.70524\,{$\pm$\,0.00343} & 0.33776\,{$\pm$\,0.00027} \\
 & SparseSZO & 0.27062\,{$\pm$\,0.00994} & 0.09523\,{$\pm$\,0.00277} & 0.45858\,{$\pm$\,0.00151} \\
 & TruncZSGD & 0.18022\,{$\pm$\,0.01223} & 0.14323\,{$\pm$\,0.01869} & 0.99149\,{$\pm$\,0.00214} \\
 & ZORO & 0.51254\,{$\pm$\,0.06313} & 0.02534\,{$\pm$\,0.00188} & 0.99998\,{$\pm$\,0.00001} \\
 & ZO-BCD & 0.00708\,{$\pm$\,0.00256} & 0.02759\,{$\pm$\,0.01988} & 0.99994\,{$\pm$\,0.00003} \\
 & SZOHT & 0.49686\,{$\pm$\,0.03160} & 0.12000\,{$\pm$\,0.09466} & 0.33883\,{$\pm$\,0.00554} \\
 & \textbf{GraCe} (ours) & \textbf{0.00508}\,{$\pm$\,0.00242} & \textbf{0.00449}\,{$\pm$\,0.00005} & \textbf{0.32381}\,{$\pm$\,0.00097} \\
\bottomrule
\end{tabular}
%}
\end{center}
\end{table*}

\section{Experiments}\label{sec:exp}

To demonstrate the empirical competence of our \Ours{}, we compare it with 12 strong baselines on three challenging functions. In the rest of the section, we introduce our benchmark functions in Section~\ref{ssec:exp-bm}, describe baselines and implementation details in Section~\ref{ssec:exp-impl}, and discuss the results in Section~\ref{ssec:exp-res}. The results are presented in Table~\ref{tab:exp-main} and Figures~\ref{fig:exp-distance} \& \ref{fig:exp-magnitude}. Our code is publicly available at \url{https://github.com/q-rz/ICML24-GraCe}.

\begin{figure*}[t]
\centering
\hfill
\subfigure[Comparison with full-gradient methods.]{\includegraphics[width=0.4\textwidth]{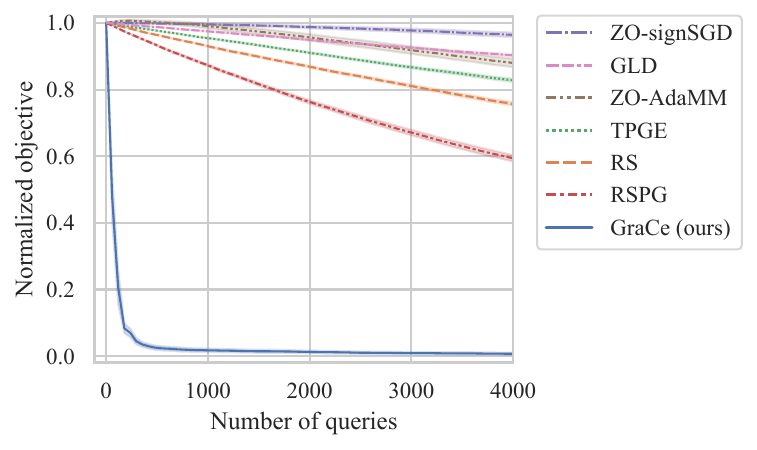}}
\hfill
\subfigure[Comparison with sparse-gradient methods.]{\includegraphics[width=0.4\textwidth]{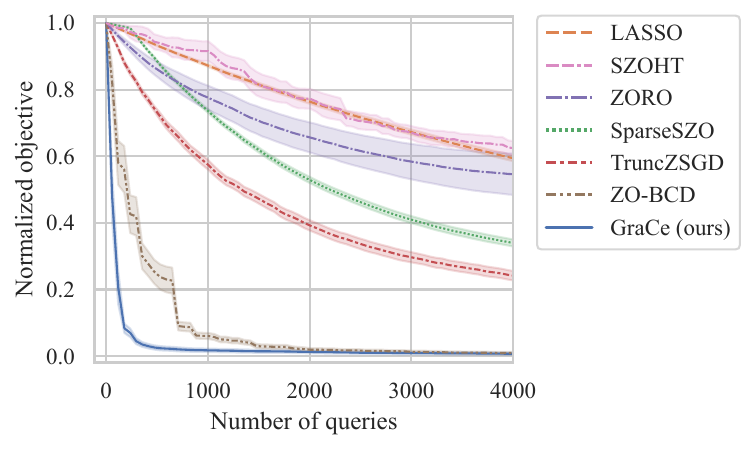}}
\hfill
\caption{Convergence plots for \textsc{Distance} (mean$\,\pm\,$s.e.).}
\label{fig:exp-distance}
\end{figure*}

\begin{figure*}[t]
\centering
\hfill
\subfigure[Comparison with full-gradient methods.]{\includegraphics[width=0.4\textwidth]{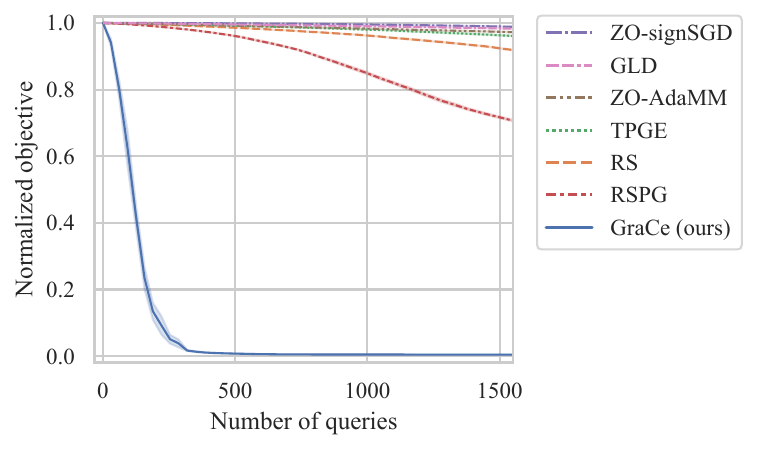}}
\hfill
\subfigure[Comparison with sparse-gradient methods.]{\includegraphics[width=0.4\textwidth]{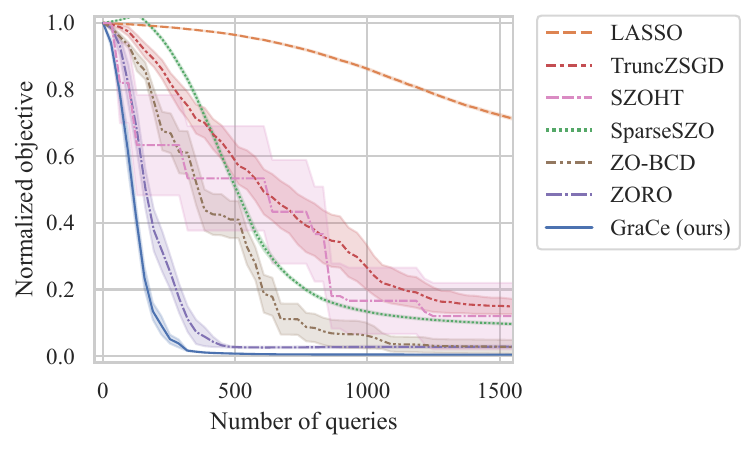}}
\hfill
\caption{Convergence plots for \textsc{Magnitude} (mean$\,\pm\,$s.e.).}
\label{fig:exp-magnitude}
\end{figure*}

\begin{figure*}[t]
\centering
\hfill
\subfigure[Comparison with full-gradient methods.]{\includegraphics[width=0.4\textwidth]{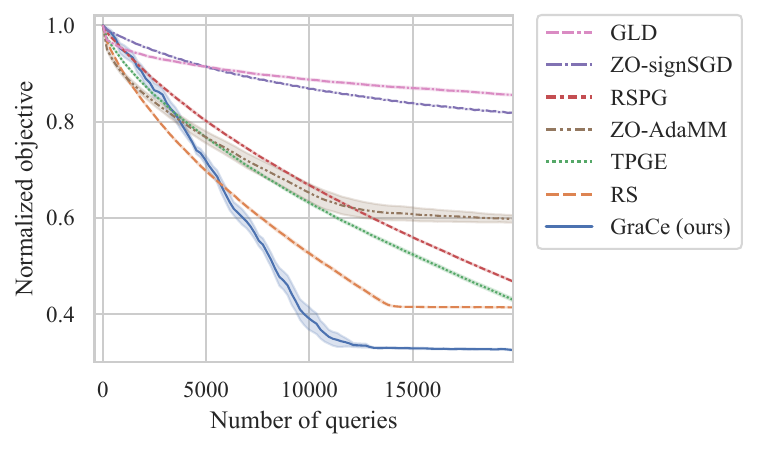}}
\hfill
\subfigure[Comparison with sparse-gradient methods.]{\includegraphics[width=0.4\textwidth]{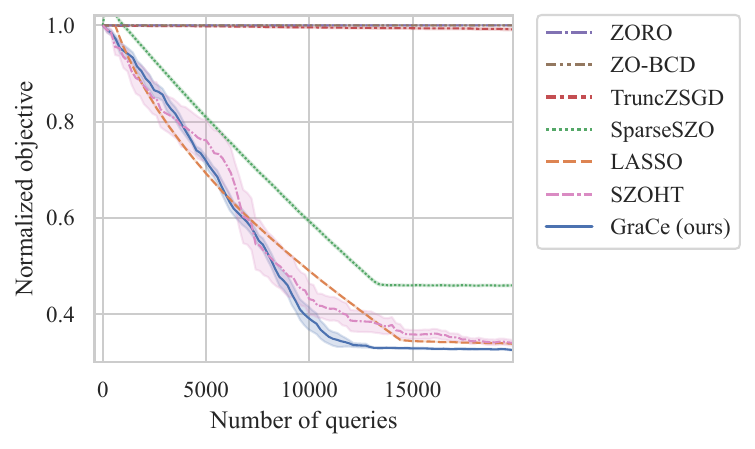}}
\hfill
\caption{Convergence plots for \textsc{Attack} (mean$\,\pm\,$s.e.).}
\label{fig:exp-attack}
\end{figure*}

\subsection{Benchmark Functions}\label{ssec:exp-bm}

To demonstrate the empirical competence of our \Ours{} in high-dimensional ZOO, we consider two challenging synthetic functions in $d=10,000$ dimensions and a real-world task in $d=13,225$ dimensions. For each synthetic benchmark function, we randomly generate 10 instantiations. 
\begin{itemize}
\item\textsc{Distance}: $f(\BM x):=(\BM x-\BM x_*)\Tp\BM W(\BM x-\BM x_*)$, where $\BM v\in\BB R^d$ is an $s$-sparse vector, and $\BM W\in\BB R^{d\times d}$ is a diagonal matrix. We randomly sample a subset $S\subseteq[d]$ of size $s$ as the nonzero dimensions of $\BM x_*$, and we generate the nonzero entries in $\BM x_*$ and $\BM W$ from $\mathop{\mathsf{Unif}}(0,1)$. We use $s=10$ and initial point $\BM x_1=\BM0_d$.

\item\textsc{Magnitude}: $f(\BM x):=\lambda\cdot\sum_{i=s+1}^d\tanh(x_{(i)}^2)-\sum_{i=1}^s\tanh(x_{(i)}^2)+s$, where $\lambda$ is a constant, and $x_{(i)}$ denotes the $i$-th largest magnitude among the coordinates of $\BM x$. For the initial point $\BM x_1$, we randomly sample a subset $S\subseteq[d]$ of size $s$ as the nonzero dimensions of $\BM x_i$ and let $(\BM x_1)_S$ be a constant $w$ times random signs. We use $s=5$, $\lambda=0.1$, and $w=0.2$.

\item\textsc{Attack}: There are various attacks on graphs \cite{dai2018adversarial,fu2023privacy}, and here we consider attacking the connectivity between two vertices on a real-world undirected graph \cite{girvan2002network} with $n=115$ vertices. %Let $\BM A$ denote the adjacency matrix of the graph, and let $u,v$ be two vertices on the graph.
We assume that the attacker wants to minimally change the adjacency matrix $\BM A\in[0,1]^{n\times n}$ by a perturbation $\BM X\in\BB R^{n\times n}$ into $\TLD{\BM A}:=\max\{\BM A\odot(\BM1_{n\times n}-|\BM X|)+(\BM1_{n\times n}-\BM A)\odot|\BM X|,0\}$ to minimize the connectivity between two vertices $u,v$, where $|\cdot|$ denotes entry-wise absolute value operation, and $\odot$ denotes entry-wise multiplication. Let $\TLD{\BM D}:=\OP{diag}(\TLD{\BM A}\BM1_{n\times1})$ denote the degree matrix w.r.t.\ $\TLD{\BM A}$, and let $\TLD{\BM A}_\text{sym}:=\TLD{\BM D}^{-1/2}\TLD{\BM A}\TLD{\BM D}^{-1/2}$ denote the symmetric normalization of $\TLD{\BM A}$. Then, we define the objective function by $f(\BM X):=\sum_{w=1}^W(\TLD{\BM A}_\text{sym}^w)_{u,v}+\lambda\|\BM X\|_\text F^2$, where $W$ denotes the number of hops in connectivity estimation, and $\|\cdot\|_\text F$ denotes the Frobenius norm. We use $u=0$, $v=1$, $W=4$, and $\lambda=\frac{100}{n^2}$ here. The true sparsity $s$ is unknown, and we use $s=30$ here. The initial point is $\BM X_1=\BM0_{n\times n}$.
\end{itemize}

We remark that these functions have approximately sparse gradients along the gradient flow starting from the initial point (although they might have non-sparse gradients elsewhere). Thus, as long as the the gradient estimates are sufficiently accurate, the gradients should be approximately sparse along the optimization trajectory.
%They help to show that our \Ours{} can give accurate gradient estimates \hh{shall we say sth like 'they help to show that even when Assumption xxx is not entirely satisfied, our \Ours{} can still give accurate gradient estimate'?}\RZ{I want to say that our method will not go out of the region where the assumption holds}. %demonstrate the capability of our \Ours{} when Assumption~\ref{ASS:sparse} is slightly violated.

\subsection{Baselines \& Implementation Details}\label{ssec:exp-impl}
We extensively compare our \Ours{} with existing ZOO methods including full-gradient and sparse-gradient methods. We do not compare proposed \Ours{} with global optimization methods such as evolutionary algorithms %and metaheuristics
because they rely on strong prior knowledge on the function structure. 

For full-gradient methods, we use the following baselines:
\begin{itemize}
\item\textbf{RS} (random search): a classic gradient estimator using a random perturbation, which is one of the oldest estimators in ZOO \cite{spall1998overview} and is later referred to as Gaussian smoothing \cite{nesterov2015random}. 
\item\textbf{TPGE} \cite{tpge}: a two-point gradient estimator that uses two additive random perturbations to smooth the function. We use their general-case version. %proposed in Section~2.2 in \citet{tpge}. 
\item\textbf{RSPG} \cite{rspg}: can be viewed as RS with multiple random perturbations in unconstrained optimization. The multiple perturbations help to reduce the variance of the estimator.
\item\textbf{ZO-signSGD} \cite{zo-signsgd}: using only the signs of the zeroth-order gradient estimate instead of specific gradient values. 
\item\textbf{ZO-AdaMM} \cite{zo-adamm}: applying the adaptive momentum method to the zeroth-order gradient estimate. As suggested by the authors, we use momentum parameters $\beta_1=0.9$ and $\beta_2=0.5$.
\item\textbf{GLD} \cite{gld}: moving to the best point among $K$ random perturbations of different scales. We use their binary search version with $K=4$, so the perturbation scales for GLD are $\{\eta,\eta/2,\eta/4,\eta/8\}$, where $\eta$ is the step size. %\hh{$\eta$ is the step size? mention it.}\RZ{added}
\end{itemize}
For sparse-gradient methods, we use the following state-of-the-art methods as baselines:
\begin{itemize}
\item\textbf{LASSO} \cite{zo-lasso}: generating random signs as perturbations and using LASSO \cite{tibshirani1996regression} to estimate the sparse gradient. As suggested by the authors, we use their mirror descent version.
\item\textbf{SparseSZO} \cite{sparseszo}: applying a mask to the gradient estimate to ensure sparsity and updating the mask periodically according to the magnitudes of coordinates. Following the authors, we update the mask every 5 steps.
\item\textbf{TruncZSGD} \cite{trunczsgd}: truncating the gradient estimate according to the magnitude of its coordinates. 
\item\textbf{ZORO} \cite{zoro}: generating random signs as perturbations and using CoSaMP \cite{needell2009cosamp} to estimate the sparse gradient. Following the authors, we run CoSaMP for at most 10 iterations with tolerence 0.5.
\item\textbf{ZO-BCD} \cite{zo-bcd}: dividing the dimensions into blocks and applying CoSaMP to each block. We use 5 blocks for ZO-BCD because this gives the best performance. Following the authors, we run CoSaMP for at most 10 iterations with tolerence 0.5.
\item\textbf{SZOHT} \cite{szoht}: perturbing only a random subset of dimensions and applying hard thresholding to the point according to the magnitudes of the coordinates. 
\end{itemize}

The hyperparameters of all methods are summarized in Table~\ref{tab:app-exp-hyp} in Appendix~\ref{app:exp-hyp}. To ensure a fair comparison, we let all methods have the same budget number of queries as that of \Ours{}. Since RS, TPGE, ZO-AdaMM, and GLD use $O(1)$ queries per step, we adjust their number $T$ of steps so that their total number of queries matches that of our \Ours{}; for other methods, we use the same number of queries per step as that of our \Ours{}. For each method, we choose the best step size $\eta$ among $\{0.5,0.2,0.1,0.05,0.02,0.01,\dots\}$. For \Ours{}, we use $m=1$, $n=\big\lfloor\frac{0.7d}{s}\big\rfloor$, and $D_1=20$ for \textsc{Distance} and \textsc{Magnitude} and $D_1=10$ for \textsc{Attack}. 

\subsection{Results \& Discussion}\label{ssec:exp-res}
We use the \emph{normalized objective} $\frac{f(\BM x_t)}{f(\BM x_1)}$ as the evaluation metric (the lower, the better). The best normalized objectives found by each method are presented in Table~\ref{tab:exp-main}, and their convergence plots are shown in Figures~\ref{fig:exp-distance}, \ref{fig:exp-magnitude}, \& \ref{fig:exp-attack}. We report means and standard errors (s.e.) over 10 runs.

From Table~\ref{tab:exp-main} and Figures~\ref{fig:exp-distance}, \ref{fig:exp-magnitude}, \& \ref{fig:exp-attack}, we can observe that our \Ours{} significantly outperforms baseline methods. For \textsc{Distance}, our \Ours{} finds a near-optimal solution within 1000 queries while none of the baselines converge even with over 5000 queries. For \textsc{Magnitude}, our \Ours{} finds a near-optimal solution within 500 queries while most baselines need at least 1000 queries. Furthermore, we can observe that our \Ours{} achieves consistent strong performance for both \textsc{Distance} and \textsc{Magnitude}. In contrast, the performance of sparse-gradient baselines varies drastically between \textsc{Distance} and \textsc{Magnitude}. For example, TruncZSGD performs well for \textsc{Distance} but not satisfactorily for \textsc{Magnitude}; ZORO performs well for \textsc{Magnitude} but badly for \textsc{Distance}. Our \Ours{} also achieves the best performance on the real-world dataset \textsc{Attack}. %In terms of the best normalized objectives found by each method, even the best baseline is 36 times worse than \Ours{} for \textsc{Distance} and is 6 times worse than \Ours{} for \textsc{Magnitude}. 

\subsection{Additional Experiments}
Due to the space limit, we provide additional experiments in Appendix~\ref{app:exp} (i) to show that our \Ours{} still achieves strong performance even when $s$ is inexact or when the gradient is non-sparse and (ii) to validate that the actual number of queries does scale as the query complexity $O\big(s\log\log\frac ds\big)$.
\section{Related Work}
%\subsection{Low-dimensional zeroth-order optimization}\TODO

\subsection{High-dimensional zeroth-order optimization}

Zeroth-order optimization aims to apply first-order optimization methods except with gradients estimated from queries. As queries are typically expensive in practice, the most important metric for comparing ZOO methods is their query complexity (i.e., the total number of queries till convergence). In the era of big data \cite{wei2024robust,chen2024wapiti,liu2024logic,liu2024class,liu2024aim,liu2023topological,qiu2024tucket,qiu2024ask,qiu2024efficient,qiu2023reconstructing,qiu2022dimes,xu2024discrete,zeng2024graph,lin2024backtime,yoo2025generalizable,yoo2024ensuring,chan2024group,wu2024fair,he2024sensitivity,wang2023networked}, high-dimensional ZOO has become increasingly important. In high-dimensional ZOO, the dimensionality $d$ can be very large. Thus, the main goal of high-dimensional ZOO is to reduce the dependence on $d$ in the query complexity under structural assumptions such as gradient sparsity and solution sparsity. Existing works in high-dimensional ZOO can be categorized into two lines. One line \cite{sparseszo,trunczsgd,szoht} applies a a mask on the gradient or the solution to enforce sparsity. Their query complexity depends on the quality of the masks. The other line  \cite{zo-lasso,zoro} employs sparse learning algorithms such as LASSO \cite{tibshirani1996regression} and CoSaMP \cite{needell2009cosamp}. Their query complexity depends on the accuracy of the sparse learning algorithms.
%\TODO

\subsection{High-dimensional first-order optimization}
A parallel line of research is high-dimensional first-order optimization, where the gradients of the objective function $f$ can be exactly computed or unbiasedly estimated. In the stark contrast to ZOO, general first-order optimization methods typically have dimension-independent rates of convergence. Thus, unlike ZOO, high-dimensional first-order optimization mainly focuses on handling high-dimensional constraints and achieving further acceleration for special problem structures. Mirror descent \cite{nemirovskij1983problem} is an efficient method to handle non-standard geometry. It has been successfully applied to high-dimensional optimization with simplicial \cite{beck2003mirror} and sparsity \cite{shalev2009stochastic} constraints, and also to problems with convex--concave \cite{nemirovski2009robust} and compositional \cite{lan2012optimal} structures. %\hh{this sentence has three 'and'. consider to re-organize it.}\RZ{revised}
Other methods such as coordinate descent \cite{shalev2010trading} for sparse optimization and the homotopy method \cite{xiao2013proximal} for $\ell_1$-regularized least squares have also been developed to further accelerate convergence for special problem structures over general methods. 

\subsection{Compressed sensing}
Our \Ours{} is a generalization and an improvement of the Indyk--Price--Woodruff (IPW) algorithm \cite{indyk2011power} in compressed sensing, bridging an interesting connection between zeroth-order optimization and this parallel field. Compressed sensing is a classic field that has been widely studied in various domains, including signal processing, medical imaging, and data compression \cite{price2013sparse}. The aim of compressed sensing is to recover a sparse signal from a minimal number of linear measurements. %, where the measurements are represented by a so-called \emph{sensing matrix}.
Early methods focuses on non-adaptive measurements. For instance, magnetic resonance imaging (MRI) machines uses 2-dimensional Fourier transforms of the image \cite{lustig2008compressed}; single-pixel cameras employs wavelet transforms \cite{duarte2008single}. More recently, adaptive methods have been proposed \cite{haupt2009compressive}, and the IPW algorithm is the state of the art among adaptive methods. Nonetheless, the IPW algorithm is purely theoretical %\hh{'purely theoretical' -- will this statement too strong? it basically means ipw is useless in practice}\RZ{The IPW paper provided no experiments, and it does have extremely large constants. As far as I know, many advanced algorithms in TCS have no practical use due to their large constants. Thus, I think it might be OK to say IPW is purely theoretical?}\hh{ok}
due to its impractically large constant. We have improved the IPW algorithm via our \emph{dependent random partition} technique and our corresponding novel analysis to make it practical.     
\section{Concluding Remarks}
In this paper, we have studied the problem of zeroth-order optimizing (ZOO) in high dimensions for functions with approximately sparse gradients. We have introduced a relaxed assumption on approximate gradient sparsity, which is weaker than previous assumptions. We have proposed a query-efficient gradient estimator called \Ours{}, whose query complexity has only double-logarithmic dependence on the dimensionality. With the help of \Ours{}, we have achieved an $O\big(\frac1T\big)$ rate of convergence for nonconvex ZOO. Experiments have demonstrated the strong empirical performance of our proposed method. We view our work as an early yet inspiring step towards high-dimensional ZOO.

%Due to the space limit, please refer to Appendix~\ref{app:lim} for a discussion on limitations and future work.

The following are limitations of this work that we wish to be addressed in future work. %\hh{i like these three open questions and the way we present them.}\RZ{Thanks!}
\begin{itemize}
\item\textbf{Stochastic ZOO.} \Ours{} encodes information into queries, so a limitation of \Ours{} is that it assumes noise-free function evaluations. Unfortunately, this does not always hold for real-world applications. Thus, an interesting open question is: can we encode information into queries under noisy function evaluations? A possible idea is by employing error correcting codes \cite{berrou1993near} to encode the dimension information under noisy queries. % and employing sample-efficient mean estimators such as the Hodges--Lehmann estimator \cite{hodges1963estimates} to reduce the noise. 
\item\textbf{Lower bound for ZOO.} It is still unclear whether our $O\big(s\log\log\frac ds\big)$ query complexity is optimal for sparse-gradient ZOO. To date, there is only limited work on the lower bound for ZOO. For instance, \citet{alabdulkareem2021information} show that $\Omega(d/\varepsilon^2)$ queries are required for noisy ZOO to achieve $\varepsilon$ error in the worst case, but it is still unclear how many queries are required for sparse-gradient ZOO. Thus, an interesting open problem is: can we find a matching lower bound for the query complexity? A possible idea is by considering the example in \citet{price2013lower}, which has been used to show a $O(\log\log d)$ lower bound for $s=O(1)$ in compressed sensing.
\item\textbf{ZOO with memory.} Another limitation of \Ours{} is that it uses only the information collected in each step to find the candidate set $J$ but ignores %the information collected in
previous steps. Meanwhile, we observe that the optimal candidate sets of different steps are typically the same or at least highly correlated. Thus, an interesting open problem is: can we leverage the information collected in previous steps to help find the candidate set $J$ using fewer queries than $O\big(s\log\log\frac ds\big)$? A possible idea is by keeping the candidate set $J$ from the previous step and %make minor revisions to
updating it using a small number of queries. 
\end{itemize}

%\TODO[remark on SparseSZO]

\section*{Acknowledgements}
This work was supported by NSF (%
2134079 %MoDL
and 2324770%RDH
),
NIFA (2020-67021-32799), % NSF-NIFA AIFarm@UIUC
DHS (17STQAC00001-07-00), %DHS COE privacy
AFOSR (FA9550-24-1-0002), %MURI network alignment
and the C3.ai Digital Transformation Institute. %C3.AI
The content of the information in this document does not necessarily reflect the position or the policy of the Government, and no official endorsement should be inferred.  The U.S. Government is authorized to reproduce and distribute reprints for Government purposes notwithstanding any copyright notation here on.

\section*{Impact Statement}

This paper presents work whose goal is to advance the field of zeroth-order optimization. There are many potential societal consequences of zeroth-order optimization, including both positive and negative consequences. On the one hand, zeroth-order optimization can be applied to sequential experimental design. Since the query complexity of the zeroth-order optimizer corresponds to the number of experiments, improvement in the query complexity reduces the cost for the experimenter. On the other hand, zeroth-order optimization can also be applied to black-box adversarial attacks %\RZ{Is it appropriate to mention black-box attacks in our impact statement?}
against machine learning models. Since the query complexity of the zeroth-order optimizer corresponds to the number of trials needed by the attacker, improvement in the query complexity reduces the cost for the attacker. We remark that these societal impacts apply to all zeroth-order optimization methods and are not specific to our work. 

\bibliography{output}

\begin{thebibliography}{58}
\providecommand{\natexlab}[1]{#1}
\providecommand{\url}[1]{\texttt{#1}}
\expandafter\ifx\csname urlstyle\endcsname\relax
  \providecommand{\doi}[1]{doi: #1}\else
  \providecommand{\doi}{doi: \begingroup \urlstyle{rm}\Url}\fi

\bibitem[Alabdulkareem \& Honorio(2021)Alabdulkareem and Honorio]{alabdulkareem2021information}
Alabdulkareem, A. and Honorio, J.
\newblock Information-theoretic lower bounds for zero-order stochastic gradient estimation.
\newblock In \emph{Proceedings of the 2021 IEEE International Symposium on Information Theory}, pp.\  2316--2321. IEEE, 2021.

\bibitem[Ba et~al.(2010)Ba, Indyk, Price, and Woodruff]{ba2010lower}
Ba, K.~D., Indyk, P., Price, E., and Woodruff, D.~P.
\newblock Lower bounds for sparse recovery.
\newblock In \emph{Proceedings of the Twenty-First Annual {ACM-SIAM} Symposium on Discrete Algorithms}, pp.\  1190--1197. {SIAM}, 2010.

\bibitem[Balasubramanian \& Ghadimi(2022)Balasubramanian and Ghadimi]{trunczsgd}
Balasubramanian, K. and Ghadimi, S.
\newblock Zeroth-order nonconvex stochastic optimization: Handling constraints, high dimensionality, and saddle points.
\newblock \emph{Foundations of Computational Mathematics}, 22\penalty0 (1):\penalty0 35--76, 2022.

\bibitem[Beck \& Teboulle(2003)Beck and Teboulle]{beck2003mirror}
Beck, A. and Teboulle, M.
\newblock Mirror descent and nonlinear projected subgradient methods for convex optimization.
\newblock \emph{Operations Research Letters}, 31\penalty0 (3):\penalty0 167--175, 2003.

\bibitem[Berrou et~al.(1993)Berrou, Glavieux, and Thitimajshima]{berrou1993near}
Berrou, C., Glavieux, A., and Thitimajshima, P.
\newblock Near {Shannon} limit error-correcting coding and decoding: {Turbo}-codes. 1.
\newblock In \emph{Proceedings of ICC'93--IEEE International Conference on Communications}, volume~2, pp.\  1064--1070. IEEE, 1993.

\bibitem[Cai et~al.(2021)Cai, Lou, {McKenzie}, and Yin]{zo-bcd}
Cai, H., Lou, Y., {McKenzie}, D., and Yin, W.
\newblock A zeroth-order block coordinate descent algorithm for huge-scale black-box optimization.
\newblock In \emph{Proceedings of the 38th International Conference on Machine Learning}, pp.\  1193--1203. PMLR, 2021.

\bibitem[Cai et~al.(2022)Cai, McKenzie, Yin, and Zhang]{zoro}
Cai, H., McKenzie, D., Yin, W., and Zhang, Z.
\newblock Zeroth-order regularized optimization {(ZORO)}: Approximately sparse gradients and adaptive sampling.
\newblock \emph{{SIAM} Journal on Optimization}, 32\penalty0 (2):\penalty0 687--714, 2022.

\bibitem[Chan et~al.(2024)Chan, Liu, Qiu, Zhang, Maciejewski, and Tong]{chan2024group}
Chan, E., Liu, Z., Qiu, R., Zhang, Y., Maciejewski, R., and Tong, H.
\newblock Group fairness via group consensus.
\newblock In \emph{The 2024 ACM Conference on Fairness, Accountability, and Transparency}, pp.\  1788--1808, 2024.

\bibitem[Chen et~al.(2024)Chen, Qiu, Yuan, Liu, Wei, Yoo, Zeng, Yang, and Tong]{chen2024wapiti}
Chen, L., Qiu, R., Yuan, S., Liu, Z., Wei, T., Yoo, H., Zeng, Z., Yang, D., and Tong, H.
\newblock {WAPITI}: A watermark for finetuned open-source {LLMs}.
\newblock \emph{arXiv}, 2410.06467, 2024.

\bibitem[Chen et~al.(2019)Chen, Liu, Xu, Li, Lin, Hong, and Cox]{zo-adamm}
Chen, X., Liu, S., Xu, K., Li, X., Lin, X., Hong, M., and Cox, D.~D.
\newblock {ZO-AdaMM}: Zeroth-order adaptive momentum method for black-box optimization.
\newblock In \emph{Advances in Neural Information Processing Systems}, volume~32, pp.\  7202--7213, 2019.

\bibitem[Dai et~al.(2018)Dai, Li, Tian, Huang, Wang, Zhu, and Song]{dai2018adversarial}
Dai, H., Li, H., Tian, T., Huang, X., Wang, L., Zhu, J., and Song, L.
\newblock Adversarial attack on graph structured data.
\newblock In \emph{Proceedings of the 35th International Conference on Machine Learning}, pp.\  1115--1124, 2018.

\bibitem[{de Vazelhes} et~al.(2022){de Vazelhes}, Zhang, Wu, Yuan, and Gu]{szoht}
{de Vazelhes}, W., Zhang, H., Wu, H., Yuan, X., and Gu, B.
\newblock Zeroth-order hard-thresholding: Gradient error vs expansivity.
\newblock In \emph{Advances in Neural Information Processing Systems}, volume~35, 2022.

\bibitem[Duarte et~al.(2008)Duarte, Davenport, Takhar, Laska, Sun, Kelly, and Baraniuk]{duarte2008single}
Duarte, M.~F., Davenport, M.~A., Takhar, D., Laska, J.~N., Sun, T., Kelly, K.~F., and Baraniuk, R.~G.
\newblock Single-pixel imaging via compressive sampling.
\newblock \emph{IEEE Signal Processing Magazine}, 25\penalty0 (2):\penalty0 83--91, 2008.

\bibitem[Duchi et~al.(2015)Duchi, Jordan, Wainwright, and Wibisono]{tpge}
Duchi, J.~C., Jordan, M.~I., Wainwright, M.~J., and Wibisono, A.
\newblock Optimal rates for zero-order convex optimization: The power of two function evaluations.
\newblock \emph{{IEEE} Transactions on Information Theory}, 61\penalty0 (5):\penalty0 2788--2806, 2015.

\bibitem[Fu et~al.(2023)Fu, Bao, Maciejewski, Tong, and He]{fu2023privacy}
Fu, D., Bao, W., Maciejewski, R., Tong, H., and He, J.
\newblock Privacy-preserving graph machine learning from data to computation: A survey.
\newblock \emph{ACM SIGKDD Explorations Newsletter}, 25\penalty0 (1):\penalty0 54--72, 2023.

\bibitem[Ghadimi \& Lan(2012)Ghadimi and Lan]{zo-rs}
Ghadimi, S. and Lan, G.
\newblock Stochastic first- and zeroth-order methods for nonconvex stochastic programming.
\newblock \emph{SIAM Journal on Optimization}, 23\penalty0 (4):\penalty0 2341--2368, 2012.

\bibitem[Ghadimi et~al.(2016)Ghadimi, Lan, and Zhang]{rspg}
Ghadimi, S., Lan, G., and Zhang, H.
\newblock Mini-batch stochastic approximation methods for nonconvex stochastic composite optimization.
\newblock \emph{Mathematical Programming}, 155\penalty0 (1-2):\penalty0 267--305, 2016.

\bibitem[Girvan \& Me(2002)Girvan and Me]{girvan2002network}
Girvan, M. and Me, N.
\newblock Network of {American} football games between division ia colleges during regular season fall 2000.
\newblock \emph{Proceedings of the National Academy of Sciences of the United States of America}, 99:\penalty0 7821--7826, 2002.

\bibitem[Golovin et~al.(2020)Golovin, Karro, Kochanski, Lee, Song, and Zhang]{gld}
Golovin, D., Karro, J., Kochanski, G., Lee, C., Song, X., and Zhang, Q.
\newblock Gradientless descent: High-dimensional zeroth-order optimization.
\newblock In \emph{International Conference on Learning Representations}, 2020.

\bibitem[Haupt et~al.(2009)Haupt, Baraniuk, Castro, and Nowak]{haupt2009compressive}
Haupt, J.~D., Baraniuk, R.~G., Castro, R.~M., and Nowak, R.~D.
\newblock Compressive distilled sensing: Sparse recovery using adaptivity in compressive measurements.
\newblock In \emph{2009 Conference Record of the Forty-Third Asilomar Conference on Signals, Systems and Computers}, pp.\  1551--1555. IEEE, 2009.

\bibitem[He et~al.(2024)He, Kang, Qiu, Wang, Sepulveda, and Tong]{he2024sensitivity}
He, X., Kang, J., Qiu, R., Wang, F., Sepulveda, J., and Tong, H.
\newblock On the sensitivity of individual fairness: Measures and robust algorithms.
\newblock In \emph{Proceedings of the 33rd ACM International Conference on Information and Knowledge Management}, pp.\  829--838, 2024.

\bibitem[Indyk et~al.(2011)Indyk, Price, and Woodruff]{indyk2011power}
Indyk, P., Price, E., and Woodruff, D.~P.
\newblock On the power of adaptivity in sparse recovery.
\newblock In \emph{2011 {IEEE} 52nd Annual Symposium on Foundations of Computer Science}, pp.\  285--294. {IEEE} Computer Society, 2011.

\bibitem[Kiefer \& Wolfowitz(1952)Kiefer and Wolfowitz]{kiefer1952stochastic}
Kiefer, J. and Wolfowitz, J.
\newblock Stochastic estimation of the maximum of a regression function.
\newblock \emph{The Annals of Mathematical Statistics}, 23\penalty0 (3):\penalty0 462--466, 1952.

\bibitem[Lan(2012)]{lan2012optimal}
Lan, G.
\newblock An optimal method for stochastic composite optimization.
\newblock \emph{Mathematical Programming}, 133\penalty0 (1-2):\penalty0 365--397, 2012.

\bibitem[Lin et~al.(2024)Lin, Liu, Fu, Qiu, and Tong]{lin2024backtime}
Lin, X., Liu, Z., Fu, D., Qiu, R., and Tong, H.
\newblock {BackTime}: Backdoor attacks on multivariate time series forecasting.
\newblock In \emph{Advances in Neural Information Processing Systems}, volume~37, 2024.

\bibitem[Liu et~al.(2024{\natexlab{a}})Liu, Wang, Qiu, Ban, Chan, Song, He, and Tong]{liu2024logic}
Liu, L., Wang, Z., Qiu, R., Ban, Y., Chan, E., Song, Y., He, J., and Tong, H.
\newblock Logic query of thoughts: Guiding large language models to answer complex logic queries with knowledge graphs.
\newblock \emph{arXiv}, 2404.04264, 2024{\natexlab{a}}.

\bibitem[Liu et~al.(2019)Liu, Chen, Chen, and Hong]{zo-signsgd}
Liu, S., Chen, P.-Y., Chen, X., and Hong, M.
\newblock {signSGD} via zeroth-order oracle.
\newblock In \emph{International Conference on Learning Representations}, 2019.

\bibitem[Liu et~al.(2023)Liu, Zeng, Qiu, Yoo, Zhou, Xu, Zhu, Weldemariam, He, and Tong]{liu2023topological}
Liu, Z., Zeng, Z., Qiu, R., Yoo, H., Zhou, D., Xu, Z., Zhu, Y., Weldemariam, K., He, J., and Tong, H.
\newblock Topological augmentation for class-imbalanced node classification.
\newblock \emph{arXiv}, 2308.14181, 2023.

\bibitem[Liu et~al.(2024{\natexlab{b}})Liu, Qiu, Zeng, Yoo, Zhou, Xu, Zhu, Weldemariam, He, and Tong]{liu2024class}
Liu, Z., Qiu, R., Zeng, Z., Yoo, H., Zhou, D., Xu, Z., Zhu, Y., Weldemariam, K., He, J., and Tong, H.
\newblock Class-imbalanced graph learning without class rebalancing.
\newblock In \emph{Proceedings of the 41st International Conference on Machine Learning}, 2024{\natexlab{b}}.

\bibitem[Liu et~al.(2024{\natexlab{c}})Liu, Qiu, Zeng, Zhu, Hamann, and Tong]{liu2024aim}
Liu, Z., Qiu, R., Zeng, Z., Zhu, Y., Hamann, H., and Tong, H.
\newblock {AIM}: Attributing, interpreting, mitigating data unfairness.
\newblock In \emph{Proceedings of the 30th ACM SIGKDD Conference on Knowledge Discovery and Data Mining}, pp.\  2014--2025, 2024{\natexlab{c}}.

\bibitem[Lustig et~al.(2008)Lustig, Donoho, Santos, and Pauly]{lustig2008compressed}
Lustig, M., Donoho, D.~L., Santos, J.~M., and Pauly, J.~M.
\newblock Compressed sensing {MRI}.
\newblock \emph{IEEE Signal Processing Magazine}, 25\penalty0 (2):\penalty0 72--82, 2008.

\bibitem[Maty{\'a\v s}(1965)]{matyas1965random}
Maty{\'a\v s}, J.
\newblock Random optimization.
\newblock \emph{Avtomatika i Telemekhanika}, pp.\  246--253, 1965.

\bibitem[Needell \& Tropp(2009)Needell and Tropp]{needell2009cosamp}
Needell, D. and Tropp, J.~A.
\newblock {CoSaMP}: Iterative signal recovery from incomplete and inaccurate samples.
\newblock \emph{Applied and computational harmonic analysis}, 26\penalty0 (3):\penalty0 301--321, 2009.

\bibitem[Nemirovski et~al.(2009)Nemirovski, Juditsky, Lan, and Shapiro]{nemirovski2009robust}
Nemirovski, A., Juditsky, A., Lan, G., and Shapiro, A.
\newblock Robust stochastic approximation approach to stochastic programming.
\newblock \emph{SIAM Journal on optimization}, 19\penalty0 (4):\penalty0 1574--1609, 2009.

\bibitem[Nemirovski \& Yudin(1983)Nemirovski and Yudin]{nemirovskij1983problem}
Nemirovski, A.~S. and Yudin, D.~B.
\newblock \emph{Problem Complexity and Method Efficiency in Optimization}.
\newblock Wiley-Interscience, 1983.

\bibitem[Nesterov(2018)]{nesterov2018lectures}
Nesterov, Y.
\newblock \emph{Lectures on Convex Optimization}, volume 137.
\newblock Springer, 2018.

\bibitem[Nesterov \& Spokoiny(2015)Nesterov and Spokoiny]{nesterov2015random}
Nesterov, Y. and Spokoiny, V.
\newblock Random gradient-free minimization of convex functions.
\newblock \emph{Foundations of Computational Mathematics}, 17, 2015.

\bibitem[Ohta et~al.(2020)Ohta, Berger, Sokolov, and Riezler]{sparseszo}
Ohta, M., Berger, N., Sokolov, A., and Riezler, S.
\newblock Sparse perturbations for improved convergence in stochastic zeroth-order optimization.
\newblock In \emph{International Conference on Machine Learning, Optimization, and Data Science ({LOD} 2020)}, volume 12566, pp.\  39--64. Springer, 2020.

\bibitem[Price \& Woodruff(2013)Price and Woodruff]{price2013lower}
Price, E. and Woodruff, D.~P.
\newblock Lower bounds for adaptive sparse recovery.
\newblock In \emph{Proceedings of the Twenty-Fourth Annual {ACM-SIAM} Symposium on Discrete Algorithms}, pp.\  652--663. {SIAM}, 2013.

\bibitem[Price(2013)]{price2013sparse}
Price, E.~C.
\newblock \emph{Sparse recovery and {Fourier} sampling}.
\newblock PhD thesis, Massachusetts Institute of Technology, 2013.

\bibitem[Qiu et~al.(2022)Qiu, Sun, and Yang]{qiu2022dimes}
Qiu, R., Sun, Z., and Yang, Y.
\newblock {DIMES}: A differentiable meta solver for combinatorial optimization problems.
\newblock In \emph{Advances in Neural Information Processing Systems}, volume~35, pp.\  25531--25546, 2022.

\bibitem[Qiu et~al.(2023)Qiu, Wang, Ying, Poor, Zhang, and Tong]{qiu2023reconstructing}
Qiu, R., Wang, D., Ying, L., Poor, H.~V., Zhang, Y., and Tong, H.
\newblock Reconstructing graph diffusion history from a single snapshot.
\newblock In \emph{Proceedings of the 29th ACM SIGKDD Conference on Knowledge Discovery and Data Mining}, pp.\  1978--1988, 2023.

\bibitem[Qiu et~al.(2024{\natexlab{a}})Qiu, Jang, Lin, Liu, and Tong]{qiu2024tucket}
Qiu, R., Jang, J.-G., Lin, X., Liu, L., and Tong, H.
\newblock {TUCKET}: A tensor time series data structure for efficient and accurate factor analysis over time ranges.
\newblock \emph{Proceedings of the VLDB Endowment}, 17\penalty0 (13), 2024{\natexlab{a}}.

\bibitem[Qiu et~al.(2024{\natexlab{b}})Qiu, Xu, Bao, and Tong]{qiu2024ask}
Qiu, R., Xu, Z., Bao, W., and Tong, H.
\newblock Ask, and it shall be given: On the {Turing} completeness of prompting.
\newblock \emph{arXiv}, 2411.01992, 2024{\natexlab{b}}.

\bibitem[Qiu et~al.(2024{\natexlab{c}})Qiu, Zeng, Tong, Ezick, and Lott]{qiu2024efficient}
Qiu, R., Zeng, W.~W., Tong, H., Ezick, J., and Lott, C.
\newblock How efficient is {LLM}-generated code? a rigorous \& high-standard benchmark.
\newblock \emph{arXiv}, 2406.06647, 2024{\natexlab{c}}.

\bibitem[Shalev-Shwartz \& Tewari(2009)Shalev-Shwartz and Tewari]{shalev2009stochastic}
Shalev-Shwartz, S. and Tewari, A.
\newblock Stochastic methods for l 1 regularized loss minimization.
\newblock In \emph{Proceedings of the 26th Annual International Conference on Machine Learning}, pp.\  929--936, 2009.

\bibitem[Shalev-Shwartz et~al.(2010)Shalev-Shwartz, Srebro, and Zhang]{shalev2010trading}
Shalev-Shwartz, S., Srebro, N., and Zhang, T.
\newblock Trading accuracy for sparsity in optimization problems with sparsity constraints.
\newblock \emph{SIAM Journal on Optimization}, 20\penalty0 (6):\penalty0 2807--2832, 2010.

\bibitem[Spall(1998)]{spall1998overview}
Spall, J.~C.
\newblock An overview of the simultaneous perturbation method for efficient optimization.
\newblock \emph{Johns Hopkins APL Technical Digest}, 19\penalty0 (4):\penalty0 482--492, 1998.

\bibitem[Tibshirani(1996)]{tibshirani1996regression}
Tibshirani, R.
\newblock Regression shrinkage and selection via the lasso.
\newblock \emph{Journal of the Royal Statistical Society Series B: Statistical Methodology}, 58\penalty0 (1):\penalty0 267--288, 1996.

\bibitem[Wang et~al.(2023)Wang, Yan, Qiu, Zhu, Guan, Margenot, and Tong]{wang2023networked}
Wang, D., Yan, Y., Qiu, R., Zhu, Y., Guan, K., Margenot, A., and Tong, H.
\newblock Networked time series imputation via position-aware graph enhanced variational autoencoders.
\newblock In \emph{Proceedings of the 29th ACM SIGKDD Conference on Knowledge Discovery and Data Mining}, pp.\  2256--2268, 2023.

\bibitem[Wang et~al.(2018)Wang, Du, Balakrishnan, and Singh]{zo-lasso}
Wang, Y., Du, S.~S., Balakrishnan, S., and Singh, A.
\newblock Stochastic zeroth-order optimization in high dimensions.
\newblock In \emph{Proceedings of the 21st International Conference on Artificial Intelligence and Statistics (AISTATS) 2018}, volume~84, pp.\  1356--1365. {PMLR}, 2018.

\bibitem[Wei et~al.(2024)Wei, Qiu, Chen, Qi, Lin, Xu, Nag, Li, Lu, Wang, Luo, Liu, Wang, He, He, and Tang]{wei2024robust}
Wei, T., Qiu, R., Chen, Y., Qi, Y., Lin, J., Xu, W., Nag, S., Li, R., Lu, H., Wang, Z., Luo, C., Liu, H., Wang, S., He, J., He, Q., and Tang, X.
\newblock Robust watermarking for diffusion models: A unified multi-dimensional recipe, 2024.
\newblock URL \url{https://openreview.net/pdf?id=O13fIFEB81}.

\bibitem[Wu et~al.(2024)Wu, Zheng, Yu, Qiu, Birge, and He]{wu2024fair}
Wu, Z., Zheng, L., Yu, Y., Qiu, R., Birge, J., and He, J.
\newblock Fair anomaly detection for imbalanced groups.
\newblock \emph{arXiv}, 2409.10951, 2024.

\bibitem[Xiao \& Zhang(2013)Xiao and Zhang]{xiao2013proximal}
Xiao, L. and Zhang, T.
\newblock A proximal-gradient homotopy method for the sparse least-squares problem.
\newblock \emph{SIAM Journal on Optimization}, 23\penalty0 (2):\penalty0 1062--1091, 2013.

\bibitem[Xu et~al.(2024)Xu, Qiu, Chen, Chen, Fan, Pan, Zeng, Das, and Tong]{xu2024discrete}
Xu, Z., Qiu, R., Chen, Y., Chen, H., Fan, X., Pan, M., Zeng, Z., Das, M., and Tong, H.
\newblock Discrete-state continuous-time diffusion for graph generation.
\newblock In \emph{Advances in Neural Information Processing Systems}, volume~37, 2024.

\bibitem[Yoo et~al.(2024)Yoo, Zeng, Kang, Qiu, Zhou, Liu, Wang, Xu, Chan, and Tong]{yoo2024ensuring}
Yoo, H., Zeng, Z., Kang, J., Qiu, R., Zhou, D., Liu, Z., Wang, F., Xu, C., Chan, E., and Tong, H.
\newblock Ensuring user-side fairness in dynamic recommender systems.
\newblock In \emph{Proceedings of the ACM on Web Conference 2024}, pp.\  3667--3678, 2024.

\bibitem[Yoo et~al.(2025)Yoo, Qiu, Xu, Wang, and Tong]{yoo2025generalizable}
Yoo, H., Qiu, R., Xu, C., Wang, F., and Tong, H.
\newblock Generalizable recommender system during temporal popularity distribution shifts.
\newblock In \emph{Proceedings of the 31st ACM SIGKDD Conference on Knowledge Discovery and Data Mining}, 2025.

\bibitem[Zeng et~al.(2024)Zeng, Qiu, Xu, Liu, Yan, Wei, Ying, He, and Tong]{zeng2024graph}
Zeng, Z., Qiu, R., Xu, Z., Liu, Z., Yan, Y., Wei, T., Ying, L., He, J., and Tong, H.
\newblock Graph mixup on approximate {Gromov--Wasserstein} geodesics.
\newblock In \emph{Proceedings of the 41st International Conference on Machine Learning}, 2024.

\end{thebibliography}
\bibliographystyle{icml2023}

\newpage
\appendix
\onecolumn

\section*{Contents}
\ALIST{
    \AITEM{app:proof}
    \ALIST{
        \AITEM{app:pf-lip}
        \AITEM{app:pf-shrink}
        \AITEM{app:pf-one}
        \AITEM{app:pf-all}
        \AITEM{app:pf-grad}
        \AITEM{app:pf-opt}
        \AITEM{app:pf-gd-prob}
    }
    \AITEM{app:exp}
    \ALIST{
        \AITEM{app:exp-hyp}
        \AITEM{app:exp-unknown-s}
        \AITEM{app:exp-nonsparse}
        \AITEM{app:exp-num-queries}
    }
}
\ASEC{Proofs}{app:proof}

\ASSEC{Preliminaries}{app:pf-lip}

We introduce two classic results on Lipschitz properties, whose proofs can be found, for example, in \citet{nesterov2018lectures}. We restate the theorems and their proofs with our notation as follows.

\begin{lemma}[Lemma~1.2.3, \citealp{nesterov2018lectures}]\label{LEM:lip-smooth}
For an $L_1$-Lipschitz smooth function $f$,
\AL{|f(\BM x+\BM u)-f(\BM x)-\langle\nabla f(\BM x),\BM u\rangle|\le\frac{L_1}2\|\BM u\|_2^2,&\qquad\forall\BM x,\BM u.}
\end{lemma}

\begin{proof}
By the fundamental theorem of calculus, the chain rule, and the the Cauchy--Schwarz inequality,
\AL{
|f(\BM x+\BM u)-f(\BM x)-\langle\nabla f(\BM x),\BM u\rangle|={}&\bigg|\int_0^1\frac{\partial}{\partial\xi}f(\BM x+\xi\BM u)\DD\xi-\,\langle\nabla f(\BM x),\BM u\rangle\bigg|\\
={}&\bigg|\int_0^1\langle\nabla f(\BM x+\xi\BM u),\BM u\rangle\DD\xi-\int_0^1\langle\nabla f(\BM x),\BM u\rangle\DD\xi\bigg|\\
={}&\bigg|\int_0^1\langle\nabla f(\BM x+\xi\BM u)-\nabla f(\BM x),\BM u\rangle\DD\xi\bigg|\\
\le{}&\int_0^1|\langle\nabla f(\BM x+\xi\BM u)-\nabla f(\BM x),\BM u\rangle|\DD\xi\\
\le{}&\int_0^1\|\nabla f(\BM x+\xi\BM u)-\nabla f(\BM x)\|_2\|\BM u\|_2\DD\xi\\
\le{}&\int_0^1L_1\|\xi\BM u\|_2\|\BM u\|_2\DD\xi\\
={}&L_1\|\BM u\|_2^2\int_0^1\xi\DD\xi=L_1\|\BM u\|_2^2\cdot\frac12
.\qedhere}
\end{proof}

%Next, note that Lipschitz smoothness implies that $f$ has continuous gradients.

\begin{lemma}[Lemma~1.2.2, \citealp{nesterov2018lectures}]\label{LEM:lip-cont}
For an $L_0$-Lipschitz continuous function $f$ with continuous gradients,
\AM{\|\nabla f(\BM x)\|_2\le L_0,\qquad\forall\BM x.}
\end{lemma}

\begin{proof}
By the chain rule and the $L_0$-Lipschitz continuity of $f$,
\AL{
\|\nabla f(\BM x)\|_2^2&=\langle\nabla f(\BM x),\nabla f(\BM x)\rangle\\
&=\frac{\partial}{\partial\xi}\Big|_{\xi=0}f(\BM x+\xi\nabla f(\BM x))\\
&=\lim_{\xi\searrow0}\frac{f(\BM x+\xi\nabla f(\BM x))-f(\BM x)}{\xi}\\
&\le\lim_{\xi\searrow0}\frac{L_0\|\xi\nabla f(\BM x)\|_2}{\xi}\\
&=\lim_{\xi\searrow0}\frac{L_0\xi\|\nabla f(\BM x)\|_2}{\xi}\\
&=L_0\|\nabla f(\BM x)\|_2
.}
It follows that $\|\nabla f(\BM x)\|_2\le L_0$, which still holds even when $\|\nabla f(\BM x)\|_2=0$.
%Use contradiction. Suppose that there is $\BM x$ with $G:=\|\nabla f(\BM x)\|_2>L_0$. By continuity of $\nabla f$, there exists $\delta>0$ such that $\|\nabla f(\BM x+\BM u)-\nabla f(\BM x)\|_2\le\frac{G-L_0}2$ for all $\BM u$ with $\|\BM u\|_2\le\delta G$. Then, by the fundamental theorem of calculus,\AL{f(\BM x+\delta\nabla f(\BM x))-f(\BM x)&=\int_0^\delta\frac{\partial}{\partial\xi}f(\BM x+\xi\nabla f(\BM x))\DD\xi\\&=\int_0^\delta\langle\nabla f(\BM x+\xi\nabla f(\BM x)),\nabla f(\BM x)\rangle\DD\xi\\&=\int_0^\delta\frac{\|\nabla f(\BM x+\xi\nabla f(\BM x))\|_2^2+\|\nabla f(\BM x)\|_2^2-\|\nabla f(\BM x+\xi\nabla f(\BM x))-\nabla f(\BM x)\|_2^2}2\DD\xi\\&\ge\int_0^\delta\frac{\big(G-\frac{G-L_0}2\big)^2+G^2-\big(\frac{G-L_0}2\big)^2}2\DD\xi\\&=\frac{G+L_0}2\delta G=\frac{G+L_0}2\|\delta\nabla f(\BM x)\|_2>L_0\|\delta\nabla f(\BM x)\|_2.}This contradicts the $L_0$-Lipschitz continuity of $f$. Therefore, $\|\nabla f(\BM x)\|_2\le L_0$ for all $\BM x$.
\end{proof}

%\begin{LEM}[A classic Chernoff bound]\label{LEM:bern-chernoff}Let $X_1,\dots,X_m\in[0.1]$ be independent random variables with $\Exp[\sum_{i=1}^mX_i]=\mu$. For any $\theta\ge0$,\AM{\Prb\bigg\{\sum_{i=1}^mX_i>(1+\theta)\mu\bigg\}\le\RM e^{-\frac{\min\{\theta,\theta^2\}}3\mu}.}\end{LEM}

\ASSEC{Proof of Lemma~\ref{LEM:shrink}}{app:pf-shrink}

Our Lemma~\ref{LEM:shrink} is an improvement over Lemma~3.2 in \citet{indyk2011power}. Our main improvement here is by using dependent random partition to bound the worst-case block size while \citet{indyk2011power} used independent subsampling. A main difference here is that we use Azuma's inequality to handle weak dependence between blocks. 

Before proving our Lemma~\ref{LEM:shrink}, we show a technical lemma via basic calculus.

\begin{lemma}\label{LEM:C1}
There is an absolute constant $0<C_1<2.29$ such that for any $0<\delta<1$ and any $D\ge2$,
\AM{2D\sqrt{2\ln\frac{4D+1}{2D\delta}}+\sqrt{2\ln\frac{8D+2}{\delta}}\le\Big(C_1D+\frac1D\Big)\sqrt{2\ln\frac3\delta}.}
\end{lemma}
\begin{proof}
%It is equivalent to show that\AM{2\sqrt{\log_{3/\delta}\frac{4D+1}{2D\delta}}+\frac1D\sqrt{\log_{3/\delta}\frac{8D+2}{\delta}}-\frac1{D^2}\le C_1.}
It suffices to find
\AL{C_1:=\sup_{\begin{subarray}{c}0<\delta<1\\D\ge2\end{subarray}}\bigg(2\sqrt{\log_{3/\delta}\frac{4D+1}{2D\delta}}+\frac1D\sqrt{\log_{3/\delta}\frac{8D+2}{\delta}}-\frac1{D^2}\bigg).}
Note that
\AL{&\frac{\partial}{\partial D}\bigg(D-\frac{D^2}{(4D+1)\sqrt{\ln\frac3\delta\ln\frac{4D+1}{2D\delta}}}+\frac{2D^2}{(4D+1)\sqrt{\ln\frac3\delta\ln\frac{8D+2}{\delta}}}-D\sqrt{\log_{3/\delta}\frac{8D+2}\delta}\bigg)\\
={}&{-\frac1{2(4D+1)^2}}\bigg(\frac{4D(2D+1)}{\sqrt{\ln\frac3\delta\ln\frac{4D+1}{2D\delta}}}+\frac{D\sqrt{\log_{3/\delta}\frac{4D+1}{2D\delta}}}{\big({\ln\frac{4D+1}{2D\delta}}\big)^2}\\&+\frac{2\sqrt{\log_{3/\delta}\frac{8D+2}{\delta}}\big(\big(2D-\frac12\ln\frac{8D+2}{\delta}\big)^2+\big(16D^2+8D+\frac34\big)\big({\ln\frac{8D+2}{\delta}}\big)^2\big)}{\big({\ln\frac{8D+2}{\delta}}\big)^2}\bigg)\\
\le{}&0
,}
and that
\AL{
&\frac{\partial}{\partial\delta}\bigg(2-\frac4{9\sqrt{\ln\frac3\delta\ln\frac{9}{4\delta}}}+\frac8{9\sqrt{\ln\frac3\delta\ln\frac{18}{\delta}}}-2\sqrt{\log_{3/\delta}\frac{18}\delta}\bigg)\\
={}&{-\frac1{9\delta\big({\ln\frac3\delta}\big)^{3/2}}}\bigg(\frac{2\big(\ln\frac{27}4+2\ln\frac1\delta\big)}{\big({\ln\frac{9}{4\delta}}\big)^{3/2}}+\frac{(9\ln6\ln18-4\ln54)+(9\ln6-8)\ln\frac1\delta}{\big({\ln\frac{18}\delta}\big)^{3/2}}\bigg)\\
\le{}&0
.}
Thus,
\AL{&\frac{\partial}{\partial D}\bigg(2\sqrt{\log_{3/\delta}\frac{4D+1}{2D\delta}}+\frac1D\sqrt{\log_{3/\delta}\frac{8D+2}{\delta}}-\frac1{D^2}\bigg)\\
={}&\frac1{D^3}\bigg(D-\frac{D^2}{(4D+1)\sqrt{\ln\frac3\delta\ln\frac{4D+1}{2D\delta}}}+\frac{2D^2}{(4D+1)\sqrt{\ln\frac3\delta\ln\frac{8D+2}{\delta}}}-D\sqrt{\log_{3/\delta}\frac{8D+2}\delta}\bigg)\\
\le{}&\frac1{D^3}\bigg(\bigg(D-\frac{D^2}{(4D+1)\sqrt{\ln\frac3\delta\ln\frac{4D+1}{2D\delta}}}+\frac{2D^2}{(4D+1)\sqrt{\ln\frac3\delta\ln\frac{8D+2}{\delta}}}-D\sqrt{\log_{3/\delta}\frac{8D+2}\delta}\bigg)\bigg|_{D=2}\bigg)\\
={}&\frac1{D^3}\bigg(2-\frac4{9\sqrt{\ln\frac3\delta\ln\frac{9}{4\delta}}}+\frac8{9\sqrt{\ln\frac3\delta\ln\frac{18}{\delta}}}-2\sqrt{\log_{3/\delta}\frac{18}\delta}\bigg)\\
\le{}&\frac1{D^3}\lim_{\delta\searrow0}\bigg(2-\frac4{9\sqrt{\ln\frac3\delta\ln\frac{9}{4\delta}}}+\frac8{9\sqrt{\ln\frac3\delta\ln\frac{18}{\delta}}}-2\sqrt{\log_{3/\delta}\frac{18}\delta}\bigg)\\
={}&\frac1{D^3}(2-0+0-2)=0
.}
It follows that
\AL{&2\sqrt{\log_{3/\delta}\frac{4D+1}{2D\delta}}+\frac1D\sqrt{\log_{3/\delta}\frac{8D+2}{\delta}}-\frac1{D^2}\\
\le{}&\bigg(2\sqrt{\log_{3/\delta}\frac{4D+1}{2D\delta}}+\frac1D\sqrt{\log_{3/\delta}\frac{8D+2}{\delta}}-\frac1{D^2}\bigg)\bigg|_{D=2}\\
={}&2\sqrt{\log_{3/\delta}\frac{9}{4\delta}}+\frac12\sqrt{\log_{3/\delta}\frac{18}{\delta}}-\frac14
.}
Futhermore, note that
\AL{&\frac{\partial}{\partial\ln\frac1\delta}\bigg(2\sqrt{\log_{3/\delta}\frac{9}{4\delta}}+\frac12\sqrt{\log_{3/\delta}\frac{18}{\delta}}-\frac14\bigg)\\
={}&\frac{\partial}{\partial\ln\frac1\delta}\bigg(2\sqrt{\frac{\ln\frac94+\ln\frac1\delta}{\ln3+\ln\frac1\delta}}+\frac12\sqrt{\frac{\ln18+\ln\frac1\delta}{\ln3+\ln\frac1\delta}}-\frac14\bigg)\\
={}&\frac1{\big({\ln3+\ln\frac1\delta}\big)^{3/2}}\bigg(\frac{\ln\frac43}{\sqrt{\ln\frac94+\ln\frac1\delta}}-\frac{\ln6}{4\sqrt{\ln18+\ln\frac1\delta}}\bigg)
.}
The derivative equals 0 when $\ln\frac1\delta=\frac{16\ln2\left({\ln\frac43}\right)^2+\ln4(\ln6)^2}{(\ln6)^2-16\left({\ln\frac43}\right)^2}-\ln9\approx0.648887$. Therefore,
\AM{&C_1=\bigg(2\sqrt{\log_{3/\delta}\frac{9}{4\delta}}+\frac12\sqrt{\log_{3/\delta}\frac{18}{\delta}}-\frac14\bigg)\bigg|_{\ln\frac1\delta=\frac{16\ln2\left({\ln\frac43}\right)^2+\ln4(\ln6)^2}{(\ln6)^2-16\left({\ln\frac43}\right)^2}-\ln9}\approx2.28955.\qedhere}
\end{proof}

Now we are ready to prove Lemma~\ref{LEM:shrink}. 

\begin{proof}[Proof of Lemma~\ref{LEM:shrink}]
To simplify notation, for fixed $\BM x\in\BB R^d$ and $\epsilon>0$, define functions $g,e:\BB R^d\to\BB R$ by
\AL{g(\BM w):=\frac{f(\BM x+\epsilon\BM w)-f(\BM x)}{\epsilon},\quad e(\BM w):=g(\BM w)-\langle\nabla f(\BM x),\BM w\rangle,\qquad\BM w\in\BB R^d.}

Consider the following procedure, which is equivalent to the procedure in Algorithm~\ref{alg:grace}.
Sample $\sigma_i\sim\OP{\mathsf{Unif}}(\{\pm1\})$ for each $i\in S$ independently, and sample a random permutation $\omega:S\to[|S|]$ of $S$. Let $B:=\lceil|S|/D\rceil$, and let $h_i:=\lceil\omega(i)/B\rceil$ for $i\in S$. Note that for any $i\in S$,
\AL{1\le h_i\le\Big\lceil\frac{|S|}B\Big\rceil=\Big\lceil\frac{|S|}{\lceil|S|/D\rceil}\Big\rceil\le\Big\lceil\frac{|S|}{|S|/D}\Big\rceil=\lceil D\rceil=D.}
Let $S_p:=\{i\in S:h_i=p\}$ for each $p\in[D]$. Define $\BM u,\BM v\in\BB R^d$ by
\AL{u_i:=\sigma_i\cdot1_{\{i\in S\}},\quad v_i:=\sigma_i\cdot h_i\cdot 1_{\{i\in S\}},\qquad i=1,\dots,d.}
Make queries $f(\BM x+\epsilon\BM u)$ and $f(\BM x+\epsilon\BM v)$. We claim that $S_{\OP{round}\left(\frac{f(\BM x+\epsilon\BM v)-f(\BM x)}{f(\BM x+\epsilon\BM u)-f(\BM x)}\right)}$ is the desired $S'$.

We prove the claim as follows. 
Let $q:=h_j$. 
First, since
\AL{&|S_q|=\#\Big\{i\in S:\Big\lceil\frac{\omega(i)}B\Big\rceil=q\Big\}=\#\Big\{i\in S:q-1<\frac{\omega(i)}B\le q\Big\}\\
={}&\#\{i\in S:(q-1)B<\pi(i)\le qB\}\le qB-(q-1)B=B=\Big\lceil\frac{|S|}D\Big\rceil.}
then
\AL{|S_q\SM j|&=|S_q|-1\le\Big\lceil\frac{|S|}D\Big\rceil-1=\Big\lfloor\frac{|S|+D-1}{D}\Big\rfloor-1\\
&\le\frac{|S|+D-1}{D}-1=\frac{|S|-1}{D}=\frac{|S\SM j|}{D}.}
Second, since
\AL{
\Exp[\|\nabla_{S_q\SM j}f(\BM x)\|_2^2]={}&\sum_{i\in S}(\nabla_if(\BM x))^2\Prb\{i\in S_q\SM j\}=\sum_{i\in S}(\nabla_if(\BM x))^2\frac{|S_q|-1}{|S|}\\
={}&\frac{\|\nabla_{S\SM j}f(\BM x)\|_2^2(|S_q|-1)}{|S|}\le\frac{\|\nabla_{S\SM j}f(\BM x)\|_2^2(\lceil|S|/D\rceil-1)}{|S|}\\
\le{}&\frac{\|\nabla_{S\SM j}f(\BM x)\|_2^2(|S|/D)}{|S|}=\frac{\|\nabla_{S\SM j}f(\BM x)\|_2^2}{D},
}
then by Markov's inequality, w.p.\ $\ge1-\delta_2$,
\AL{\|\nabla_{S_q\SM j}f(\BM x)\|_2\le\sqrt{\frac{\Exp[\|\nabla_{S_q\SM j}f(\BM x)\|_2^2]}{\delta_2}}\le\sqrt{\frac{\|\nabla_{S\SM j}f(\BM x)\|_2^2}{D\delta_2}}=\frac{\|\nabla_{S\SM j}f(\BM x)\|_2}{\sqrt{D\delta_2}}.}
Next, since $|\nabla_if(\BM x)\cdot u_i|\le|\nabla_i|$ for all $i$, and
\AL{\sum_{i\in S\SM j}(|\nabla_if(\BM x)|-(-|\nabla_if(\BM x)|))^2=4\|\nabla_{S\SM j}f(\BM x)\|_2^2,}
then by Hoeffding's inequality, w.p.\ $\ge1-\frac{\delta_1}{4D+1}$,
\ALN{eq:shrink-1}{|\langle\nabla_{S\SM j}f(\BM x),\BM u_{S\SM j}\rangle|\le\sqrt{\frac{4\|\nabla_{S\SM j}f(\BM x)\|_2^2}2\ln\frac2{\frac{\delta_1}{4D+1}}}=\|\nabla_{S\SM j}f(\BM x)\|_2\sqrt{2\ln\frac{8D+2}{\delta_1}}
.}
Besides that, thanks to the independent signs $\{\sigma_i\}_{i\in S\SM j}$, we can view $\{\nabla_if(\BM x)\cdot(v_i-qu_i)\}_{i\in S\setminus\{j\}}=\{\nabla_if(\BM x)\cdot\sigma_i\cdot(h_i-q)\}_{i\in S\setminus\{j\}}$ as a martingale difference sequence. Since $|\nabla_if(\BM x)\cdot(v_i-qu_i)|\le(D-1)|\nabla_i|\le D|\nabla_i|$ for all $i$, and
\AL{\sum_{i\in S\SM j}(D|\nabla_if(\BM x)|-(-D|\nabla_if(\BM x)|))^2=4D^2\|\nabla_{S\SM j}f(\BM x)\|_2^2,}
then by Azuma's inequality, w.p.\ $\ge1-\frac{4D\delta_1}{4D+1}$, % $h_i$'s are not independent
\ALN{eq:shrink-2}{|\langle\nabla_{S\SM j}f(\BM x),\BM v_{S\SM j}-q\BM u_{S\SM j}\rangle|\le\sqrt{\frac{4D^2\|\nabla_{S\SM j}f(\BM x)\|_2^2}2\ln\frac2{\frac{4D\delta_1}{4D+1}}}=D\|\nabla_{S\SM j}f(\BM x)\|_2\sqrt{2\ln\frac{4D+1}{2D\delta_1}}
.}
Third, by Lemma~\ref{LEM:lip-smooth},
\ALN{eq:shrink-4}{|e(\BM u)|&\le\frac{L_1\epsilon}2\|\BM u\|_2^2=\frac{L_1\epsilon}2\sum_{i\in S}\sigma_i^2=\frac{L_1\epsilon}2\sum_{i\in S}1^2=\frac{L_1|S|}2\epsilon,\label{eq:shrink-3}\\%\le\frac{L_1d}2\epsilon
|e(\BM v)|&\le\frac{L_1\epsilon}2\|\BM v\|_2^2=\frac{L_1\epsilon}2\sum_{i\in S}(\sigma_ih_i)^2\le\frac{L_1\epsilon}2\sum_{i\in S}D^2=\frac{L_1|S|D^2}2\epsilon\le\frac{L_1|S|d^2}2\epsilon%\frac{L_1d^3}2\epsilon
.}
Fourth, with the absolute constant $C_1$ in Lemma~\ref{LEM:C1}, the assumption on $j$ implies
\AL{|\nabla_jf(\BM x)|&>\big(C_1D+\tfrac1D\big)\sqrt{2\ln\tfrac3{\delta_1}}\|\nabla_{S\setminus\{j\}} f(\BM x)\|_2+L_1|S|\Big(d^2+d+\frac12\Big)\epsilon\\
&\ge 2D\|\nabla_{S\SM j}f(\BM x)\|_2\sqrt{2\ln\frac{4D+1}{2D\delta_1}}+\|\nabla_{S\SM j}f(\BM x)\|_2\sqrt{2\ln\frac{8D+2}{\delta_1}}+L_1|S|\Big(d^2+d+\frac12\Big)\epsilon.}
Thus,
\ALN{eq:shrink-5}{\frac{D\|\nabla_{S\SM j}f(\BM x)\|_2\sqrt{2\ln\frac{4D+1}{2D\delta_1}}+\frac{L_1|S|d^2}2\epsilon+d\frac{L_1|S|}2\epsilon}{|\nabla_jf(\BM x)|-\|\nabla_{S\SM j}f(\BM x)\|_2\sqrt{2\ln\frac{8D+2}{\delta_1}}-\frac{L_1|S|}2\epsilon}<\frac12.}
Finally, by Eqs.~\eqref{eq:shrink-1}, \eqref{eq:shrink-2}, \eqref{eq:shrink-3}, \eqref{eq:shrink-4}, \& \eqref{eq:shrink-5}, with probability $\ge1-\big(\delta_2+\frac{\delta_1}{4D+1}+\frac{4D\delta_1}{4D+1}\big)=1-(\delta_1+\delta_2)$,
\AL{
\Big|\frac{f(\BM x+\epsilon\BM v)-f(\BM x)}{f(\BM x+\epsilon\BM u)-f(\BM x)}-q\Big|={}&\Big|\frac{g(\BM v)}{g(\BM u)}-q\Big|\\
={}&\Big|\frac{\sigma_jq\nabla_jf(\BM x)+\langle\nabla_{S\SM j}f(\BM x),\BM v_{S\SM j}\rangle+e(\BM v)}{\sigma_j\nabla_jf(\BM x)+\langle\nabla_{S\SM j}f(\BM x),\BM u_{S\SM j}\rangle+e(\BM u)}-q\Big|\\
={}&\Big|\frac{\langle\nabla_{S\SM j}f(\BM x),\BM v_{S\SM j}-q\BM u_{S\SM j}\rangle+e(\BM v)-qe(\BM u)}{\sigma_j\nabla_jf(\BM x)+\langle\nabla_{S\SM j}f(\BM x),\BM u_{S\SM j}\rangle+e(\BM u)}\Big|\\
\le{}&\frac{|\langle\nabla_{S\SM j}f(\BM x),\BM v_{S\SM j}-q\BM u_{S\SM j}\rangle|+|e(\BM v)|+q|e(\BM u)|}{|\nabla_jf(\BM x)|-|\langle\nabla_{S\SM j}f(\BM x),\BM u_{S\SM j}\rangle|-|e(\BM u)|}\\
\le{}&\frac{D\|\nabla_{S\SM j}f(\BM x)\|_2\sqrt{2\ln\frac{4D+1}{2D\delta_1}}+\frac{L_1|S|d^2}2\epsilon+d\frac{L_1|S|}2\epsilon}{|\nabla_jf(\BM x)|-\|\nabla_{S\SM j}f(\BM x)\|_2\sqrt{2\ln\frac{8D+2}{\delta_1}}-\frac{L_1|S|}2\epsilon}\\
<{}&\frac12
.}
This implies $\OP{round}\!\big(\frac{f(\BM x+\epsilon\BM v)-f(\BM x)}{f(\BM x+\epsilon\BM u)-f(\BM x)}\big)=q=h_j$, completing the proof.
\end{proof}

\ASSEC{Proof of Lemma~\ref{LEM:one}}{app:pf-one}

The following Lemma~\ref{LEM:one-general} is the general version of our Lemma~\ref{LEM:one}, which is an improvement of Lemma~3.3 in \citet{indyk2011power}. The main technical difficulties here are (i) how to construct a division schedule $\{D_r\}_{r\ge1}$ with which we can iteratively apply Lemma~\ref{LEM:shrink} and (ii) how to show that the constructed division schedule grows rapidly. 

\begin{lemma}\label{LEM:one-general}
Suppose that hyperparameters $0<\delta,\phi,\theta<1$, and integer $D\ge2$ satisfy
\GA{
\phi\bigg(D\frac{(1-\theta)(1-\phi)\delta\ln(3/(\theta(1-\phi)\delta))}{\ln(3/(\theta(1-\phi)\phi\delta))}\bigg)^{3/2}-D\frac{(1-\theta)(1-\phi)\delta\ln(3/(\theta(1-\phi)\delta))}{\ln(3/(\theta(1-\phi)\phi\delta))}\ge(1-\theta)(1-\phi)\phi\delta,\label{eq:LEM:one:ASS1}\\
D\frac{(1-\theta)(1-\phi)\delta\ln(3/(\theta(1-\phi)\delta))}{\ln(3/(\theta(1-\phi)\phi\delta))}\ge\frac32,\label{eq:LEM:one:ASS2}}
and
\AL{A:=\bigg(D-\frac1{\sqrt{D\frac{(1-\theta)(1-\phi)\delta\ln(3/(\theta(1-\phi)\delta))}{\ln(3/(\theta(1-\phi)\phi\delta))}}-1}\bigg)\frac{(1-\theta)(1-\phi)\phi^2\delta\ln\frac{3}{\theta(1-\phi)\delta}}{\ln\frac{3}{\theta(1-\phi)\phi\delta}}>1.\label{eq:LEM:one:ASS3}}
There exists a division schedule $\{D_r\}_{r\ge1}$ with $D_1=D$ and $D_r\ge\frac{A^{(3/2)^{r-1}}}{(1-\theta)(1-\phi)\phi^2\delta}$ such that given $\BM x\in\BB R^d$, $\epsilon>0$, $S\subseteq[d]$, if there exists $j\in S$ with
\ALN{eq:LEM:one:ASS-j}{|\nabla_jf(\BM x)|>\Big(C_1D+\frac1{D}\Big)\sqrt{2\ln\frac{3}{\theta(1-\phi)\delta}}\|\nabla_{S\SM j}f(\BM x)\|_2+\lambda_{1,|S|}\cdot\epsilon,}
then $O(\log_{3/2}\log_A|S|)$ iterations of Lemma~\ref{LEM:shrink} with parameters $\{D_r\}_{r\ge1}$ can find $j$ with probability at least $1-\delta$.
\end{lemma}

Before proving Lemma~\ref{LEM:one-general}, we show a technical lemma.

\begin{lemma}
\label{LEM:dda-mono}
For any $0<\phi<1$ and $0<\delta<1$, if there is $x_0\ge0$ with $\phi x_0^{3/2}-x_0\ge\phi\delta$, then $\phi x^{3/2}-x\ge\phi\delta$ $\forall x\ge x_0$.
\end{lemma}

\begin{proof}
Let $\psi(x):=\phi x^{3/2}-x$, $x\ge0$. Since
\AL{\psi'(x)=\frac32\phi x^{1/2}-1,}
then $\psi(x)$ is decreasing over $\big[0,\frac{4}{9\phi^2}\big]$ and increasing over $\big(\frac{4}{9\phi^2},+\infty\big)$. Thus, for every $x\in\big[0,\frac{4}{9\phi^2}\big]$,
\AL{\psi(x)\le\psi(0)=0<\phi\delta.}
This implies $x_0\in\big(\frac{4}{9\phi^2},+\infty\big)$, over which $\psi$ is increasing. It follows that for every $x\ge x_0$,
\AL{&\psi(x)\ge\psi(x_0)\ge\phi\delta.\qedhere}
\end{proof}

%\begin{LEM}\label{LEM:d-mono}For any $0<\phi<1$,\AM{\frac1\phi\bigg(1-\bigg(\frac{\ln\frac1\phi}{\ln\frac{3}{\phi(1-\phi)}}\bigg)^2\bigg)\ge1.}\end{LEM}\begin{proof}\TODO\end{proof}

Now we are ready to prove Lemma~\ref{LEM:one-general}.

\begin{proof}[Proof of Lemma~\ref{LEM:one-general}]
For $r\ge1$, let
\AL{\delta_{r,1}:=\theta(1-\phi)\phi^{r-1}\delta,\qquad\delta_{r,2}:=(1-\theta)(1-\phi)\phi^{r-1}\delta.}
Let $D_1:=D$, and let $D_{r+1}:=\big\lfloor D_r^{3/2}\sqrt{\frac{\delta_{r,2}\ln(3/\delta_{r,1})}{\ln(3/\delta_{r+1,1})}}\big\rfloor$ for $r\ge1$. Note that
%\AM{\frac{\ln(3/\delta_{k,1})}{\ln(3/\delta_{k+1,1})}=1-\frac{\ln(1/\phi)}{\ln(3/\delta_{k+1,1})}}
\AL{\frac{\ln\frac{3}{\delta_{r,1}}}{\ln\frac{3}{\delta_{r+1,1}}}=1-\frac{\ln\frac1\phi}{\ln\frac{3}{\delta_{r+1,1}}}}
is non-decreasing as $r$ increases, so we have
\ALN{eq:one-mono}{\frac{\ln\frac{3}{\delta_{r+1,1}}}{\ln\frac{3}{\delta_{r+2,1}}}\ge\frac{\ln\frac{3}{\delta_{r,1}}}{\ln\frac{3}{\delta_{r+1,1}}}.}
%and note that assumption Eq.~\eqref{eq:LEM:one:ASS1} means\AL{\phi\bigg(D_1\frac{\delta_{1,2}\ln\frac3{\delta_{1,1}}}{\ln\frac3{\delta_{2,1}}}\bigg)^{3/2}-D_1\frac{\delta_{1,2}\ln\frac{3}{\delta_{1,1}}}{\ln\frac{3}{\delta_{2,1}}}\ge\TODO.}
We will show by strong induction that for every $r\ge1$,
\ALN{eq:one-ind-1}{
\phi\bigg(D_r\frac{\delta_{r,2}\ln\frac3{\delta_{r,1}}}{\ln\frac3{\delta_{r+1,1}}}\bigg)^{3/2}-D_r\frac{\delta_{r,2}\ln\frac{3}{\delta_{r,1}}}{\ln\frac{3}{\delta_{r+1,1}}}\ge\frac{\phi\delta_{r,2}\ln\frac{3}{\delta_{r,1}}}{\ln\frac{3}{\delta_{r+1,1}}},}
and
\ALN{eq:one-ind-2}{
D_{r+1}\frac{\delta_{r+1,2}\ln(3/\delta_{r+1,1})}{\ln(3/\delta_{r+2,1})}\ge D_r\frac{\delta_{r,2}\ln(3/\delta_{r,1})}{\ln(3/\delta_{r+1,1})}.
}
First, consider the base case $r=1$. Using the assumption~Eq.~\eqref{eq:LEM:one:ASS1} and the fact that $\delta_{2,1}\le\delta_{1,1}$,
\ALN{eq:one-base-1}{
\phi\bigg(D_1\frac{\delta_{1,2}\ln\frac3{\delta_{1,1}}}{\ln\frac3{\delta_{2,1}}}\bigg)^{3/2}-D_1\frac{\delta_{1,2}\ln\frac{3}{\delta_{1,1}}}{\ln\frac{3}{\delta_{2,1}}}\ge\phi\delta_{1,2}\ge\frac{\phi\delta_{1,2}\ln\frac{3}{\delta_{1,1}}}{\ln\frac{3}{\delta_{2,1}}}.}
Furthermore, by Eqs.~\eqref{eq:one-mono} \& \eqref{eq:one-base-1},
%Thus, by the inequalities $\delta_{2,1}\le\delta_{1,1}$ and $\lfloor z\rfloor\ge z-1$ ($z\in\BB R$),
\AL{D_2\frac{\delta_{2,2}\ln\frac{3}{\delta_{2,1}}}{\ln\frac{3}{\delta_{3,1}}}-D_1\frac{\delta_{1,2}\ln\frac{3}{\delta_{1,1}}}{\ln\frac{3}{\delta_{2,1}}}
={}&\bigg\lfloor D_1^{3/2}\sqrt{\frac{\delta_{1,2}\ln\frac3{\delta_{1,1}}}{\ln\frac3{\delta_{2,1}}}}\bigg\rfloor\frac{\phi\delta_{1,2}\ln\frac{3}{\delta_{2,1}}}{\ln\frac{3}{\delta_{3,1}}}-D_1\frac{\delta_{1,2}\ln\frac{3}{\delta_{1,1}}}{\ln\frac{3}{\delta_{2,1}}}\\
\ge{}&\bigg\lfloor D_1^{3/2}\sqrt{\frac{\delta_{1,2}\ln\frac3{\delta_{1,1}}}{\ln\frac3{\delta_{2,1}}}}\bigg\rfloor\frac{\phi\delta_{1,2}\ln\frac{3}{\delta_{1,1}}}{\ln\frac{3}{\delta_{2,1}}}-D_1\frac{\delta_{1,2}\ln\frac{3}{\delta_{1,1}}}{\ln\frac{3}{\delta_{2,1}}}\\
\ge{}&\bigg(D_1^{3/2}\sqrt{\frac{\delta_{1,2}\ln\frac3{\delta_{1,1}}}{\ln\frac3{\delta_{2,1}}}}-1\bigg)\frac{\phi\delta_{1,2}\ln\frac{3}{\delta_{1,1}}}{\ln\frac{3}{\delta_{2,1}}}-D_1\frac{\delta_{1,2}\ln\frac{3}{\delta_{1,1}}}{\ln\frac{3}{\delta_{2,1}}}\\
={}&\phi\bigg(D_1\frac{\delta_{1,2}\ln\frac3{\delta_{1,1}}}{\ln\frac3{\delta_{2,1}}}\bigg)^{3/2}-D_1\frac{\delta_{1,2}\ln\frac{3}{\delta_{1,1}}}{\ln\frac{3}{\delta_{2,1}}}-\frac{\phi\delta_{1,2}\ln\frac{3}{\delta_{1,1}}}{\ln\frac{3}{\delta_{2,1}}}\\
%\ge{}&\frac{\phi\delta_{1,2}\ln\frac{3}{\delta_{1,1}}}{\ln\frac{3}{\delta_{2,1}}}-\frac{\phi\delta_{1,2}\ln\frac{3}{\delta_{1,1}}}{\ln\frac{3}{\delta_{2,1}}}\TODO
\ge{}&0
.}
Next, suppose that the inductive hypotheses Eqs.~\eqref{eq:one-ind-1} \& \eqref{eq:one-ind-2} hold for $r$, and consider the induction step from $r$ to $r+1$. By strong induction w.r.t.\ Eq.~\eqref{eq:one-ind-2},
\ALN{eq:one-ind-strong}{D_{r+1}\frac{\delta_{r+1,2}\ln(3/\delta_{r+1,1})}{\ln(3/\delta_{r+2,1})}\ge D_{1}\frac{\delta_{1,2}\ln(3/\delta_{1,1})}{\ln(3/\delta_{2,1})}.}
By Lemma~\ref{LEM:dda-mono} with $x_0=D_1\frac{\delta_{1,2}\ln\frac{3}{\delta_{1,1}}}{\ln\frac{3}{\delta_{2,1}}}$ and Eqs.~\eqref{eq:one-base-1} \& \eqref{eq:one-ind-strong} as well as the fact that $\delta_{r+2,1}\le\delta_{r+1,1}$,
\ALN{eq:one-ind-step1}{
\phi\bigg(D_{r+1}\frac{\delta_{r+1,2}\ln\frac3{\delta_{r+1,1}}}{\ln\frac3{\delta_{r+2,1}}}\bigg)^{3/2}-D_{r+1}\frac{\delta_{r+1,2}\ln\frac{3}{\delta_{r+1,1}}}{\ln\frac{3}{\delta_{r+2,1}}}\ge\phi\delta_{1,2}\ge\phi\delta_{r+1,2}\ge\frac{\phi\delta_{r+1,2}\ln\frac3{\delta_{r+1,1}}}{\ln\frac3{\delta_{r+2,1}}}
.}
Furthermore, by Eqs.~\eqref{eq:one-mono} \& \eqref{eq:one-ind-step1},
\AL{
&D_{r+2}\frac{\delta_{r+2,2}\ln\frac{3}{\delta_{r+2,1}}}{\ln\frac{3}{\delta_{r+3,1}}}-D_{r+1}\frac{\delta_{r+1,2}\ln\frac{3}{\delta_{r+1,1}}}{\ln\frac{3}{\delta_{r+2,1}}}\\
={}&\bigg\lfloor D_{r+1}^{3/2}\sqrt{\frac{\delta_{r+1,2}\ln\frac3{\delta_{r+1,1}}}{\ln\frac3{\delta_{r+2,1}}}}\bigg\rfloor\frac{\phi\delta_{r+1,2}\ln\frac{3}{\delta_{r+2,1}}}{\ln\frac{3}{\delta_{r+3,1}}}-D_1\frac{\delta_{r+1,2}\ln\frac{3}{\delta_{r+1,1}}}{\ln\frac{3}{\delta_{r+2,1}}}\\
\ge{}&\bigg\lfloor D_{r+1}^{3/2}\sqrt{\frac{\delta_{r+1,2}\ln\frac3{\delta_{r+1,1}}}{\ln\frac3{\delta_{r+2,1}}}}\bigg\rfloor\frac{\phi\delta_{r+1,2}\ln\frac{3}{\delta_{r+1,1}}}{\ln\frac{3}{\delta_{r+2,1}}}-D_{r+1}\frac{\delta_{r+1,2}\ln\frac{3}{\delta_{r+1,1}}}{\ln\frac{3}{\delta_{r+2,1}}}\\
\ge{}&\bigg(D_{r+1}^{3/2}\sqrt{\frac{\delta_{r+1,2}\ln\frac3{\delta_{r+1,1}}}{\ln\frac3{\delta_{r+2,1}}}}-1\bigg)\frac{\phi\delta_{r+1,2}\ln\frac{3}{\delta_{r+1,1}}}{\ln\frac{3}{\delta_{r+2,1}}}-D_{r+1}\frac{\delta_{r+1,2}\ln\frac{3}{\delta_{r+1,1}}}{\ln\frac{3}{\delta_{r+2,1}}}\\
={}&\phi\bigg(D_{r+1}\frac{\delta_{r+1,2}\ln\frac3{\delta_{r+1,1}}}{\ln\frac3{\delta_{r+2,1}}}\bigg)^{3/2}-D_{r+1}\frac{\delta_{r+1,2}\ln\frac{3}{\delta_{r+1,1}}}{\ln\frac{3}{\delta_{r+2,1}}}-\frac{\phi\delta_{r+1,2}\ln\frac{3}{\delta_{r+1,1}}}{\ln\frac{3}{\delta_{r+2,1}}}\\
\ge{}&0
.}
Hence, the inductive hypotheses hold for all $r\ge1$. In particular, Eq.~\eqref{eq:one-ind-2} shows that $D_r\frac{\delta_{r,2}\ln(3/\delta_{r,1})}{\ln(3/\delta_{r+1,1})}$ is non-decreasing as $r$ increases. Since $D_r\frac{\delta_{r,2}\ln(3/\delta_{r,1})}{\ln(3/\delta_{r+1,1})}$ is non-decreasing with $r$, but both $\delta_{r,2}$ and $\frac{\ln(3/\delta_{r,1})}{\ln(3/\delta_{r+1,1})}$ are non-increasing with $r$, then $D_r$ is non-decreasing with $r$. Together with assumption Eq.~\eqref{eq:LEM:one:ASS2}, for every $r\ge1$,
\ALN{eq:one-div-aux}{&\frac{1}{D_{r}\big(\sqrt{D_{r}\frac{\delta_{r,2}\ln(3/\delta_{r,1})}{\ln(3/\delta_{r+1,1})}}-1\big)}
\le{}\frac{1}{D_{r}\big(\sqrt{D_1\frac{\delta_{1,2}\ln(3/\delta_{1,1})}{\ln(3/\delta_{2,1})}}-1\big)}
\le{}\frac1{D_r\big(\frac32-1\big)}=\frac2{D_r}\le\frac{2}{D_1}\le\frac22=1.
}
Furthermore, using the fact that $\delta_{r+1,2}=\phi\delta_{r,2}$ and Eq.~\eqref{eq:one-mono},
\ALN{eq:one-div-1}{&\bigg(D_{r+1}-\frac1{\sqrt{D_{r+1}\frac{\delta_{r+1,2}\ln(3/\delta_{r+1,1})}{\ln(3/\delta_{r+2,1})}}-1}\bigg)\frac{\phi^2\delta_{r+1,2}\ln\frac{3}{\delta_{r+1,1}}}{\ln\frac{3}{\delta_{r+2,1}}}\\
={}&\bigg(D_{r+1}-\frac1{\sqrt{D_{r+1}\frac{\delta_{r+1,2}\ln(3/\delta_{r+1,1})}{\ln(3/\delta_{r+2,1})}}-1}\bigg)\frac{\phi^3\delta_{r,2}\ln\frac{3}{\delta_{r+1,1}}}{\ln\frac{3}{\delta_{r+2,1}}}\\
\ge{}&\bigg(D_{r+1}-\frac1{\sqrt{D_{r}\frac{\delta_{r,2}\ln(3/\delta_{r,1})}{\ln(3/\delta_{r+1,1})}}-1}\bigg)\frac{\phi^3\delta_{r,2}\ln\frac{3}{\delta_{r+1,1}}}{\ln\frac{3}{\delta_{r+2,1}}}\\
\ge{}&\bigg(D_{r+1}-\frac1{\sqrt{D_{r}\frac{\delta_{r,2}\ln(3/\delta_{r,1})}{\ln(3/\delta_{r+1,1})}}-1}\bigg)\frac{\phi^3\delta_{r,2}\ln\frac{3}{\delta_{r,1}}}{\ln\frac{3}{\delta_{r+1,1}}}.}
Plugging into the definition of $D_{r+1}$ gives
\ALN{eq:one-div-2}{
&\bigg(D_{r+1}-\frac1{\sqrt{D_{r}\frac{\delta_{r,2}\ln(3/\delta_{r,1})}{\ln(3/\delta_{r+1,1})}}-1}\bigg)\frac{\phi^3\delta_{r,2}\ln\frac{3}{\delta_{r,1}}}{\ln\frac{3}{\delta_{r+1,1}}}\\
={}&\bigg(\bigg\lfloor D_r^{3/2}\sqrt{\frac{\delta_{r,2}\ln\frac{3}{\delta_{r,1}}}{\ln\frac{3}{\delta_{r+1,1}}}}\bigg\rfloor-\frac1{\sqrt{D_{r}\frac{\delta_{r,2}\ln(3/\delta_{r,1})}{\ln(3/\delta_{r+1,1})}}-1}\bigg)\frac{\phi^3\delta_{r,2}\ln\frac{3}{\delta_{r,1}}}{\ln\frac{3}{\delta_{r+1,1}}}\\
\ge{}&\bigg(D_r^{3/2}\sqrt{\frac{\delta_{r,2}\ln\frac{3}{\delta_{r,1}}}{\ln\frac{3}{\delta_{r+1,1}}}}-1-\frac1{\sqrt{D_{r}\frac{\delta_{r,2}\ln(3/\delta_{r,1})}{\ln(3/\delta_{r+1,1})}}-1}\bigg)\frac{\phi^3\delta_{r,2}\ln\frac{3}{\delta_{r,1}}}{\ln\frac{3}{\delta_{r+1,1}}}\\
={}&\bigg(\bigg(D_r\frac{\delta_{r,2}\ln\frac{3}{\delta_{r,1}}}{\ln\frac{3}{\delta_{r+1,1}}}\bigg)^{3/2}-\frac{\sqrt{D_{r}}\big(\frac{\delta_{r,2}\ln(3/\delta_{r,1})}{\ln(3/\delta_{r+1,1})}\big)^{3/2}}{\sqrt{D_{r}\frac{\delta_{r,2}\ln(3/\delta_{r,1})}{\ln(3/\delta_{r+1,1})}}-1}\bigg)\phi^3\\
={}&\bigg(D_r\frac{\delta_{r,2}\ln\frac{3}{\delta_{r,1}}}{\ln\frac{3}{\delta_{r+1,1}}}\bigg)^{3/2}\bigg(1-\frac{1}{D_{r}\big(\sqrt{D_{r}\frac{\delta_{r,2}\ln(3/\delta_{r,1})}{\ln(3/\delta_{r+1,1})}}-1\big)}\bigg)\phi^3
.}
Since $0\le1-\frac{1}{D_{r}\big(\sqrt{D_{r}\frac{\delta_{r,2}\ln(3/\delta_{r,1})}{\ln(3/\delta_{r+1,1})}}-1\big)}\le1$ by Eq.~\eqref{eq:one-div-aux}, then
\ALN{eq:one-div-3}{
&\bigg(D_r\frac{\delta_{r,2}\ln\frac{3}{\delta_{r,1}}}{\ln\frac{3}{\delta_{r+1,1}}}\bigg)^{3/2}\bigg(1-\frac{1}{D_{r}\big(\sqrt{D_{r}\frac{\delta_{r,2}\ln(3/\delta_{r,1})}{\ln(3/\delta_{r+1,1})}}-1\big)}\bigg)\phi^3\\
\ge{}&\bigg(D_r\frac{\delta_{r,2}\ln\frac{3}{\delta_{r,1}}}{\ln\frac{3}{\delta_{r+1,1}}}\bigg)^{3/2}\bigg(1-\frac{1}{D_{r}\big(\sqrt{D_{r}\frac{\delta_{r,2}\ln(3/\delta_{r,1})}{\ln(3/\delta_{r+1,1})}}-1\big)}\bigg)^{3/2}\phi^3\\
={}&\bigg(\bigg(D_{r}-\frac1{\sqrt{D_{r}\frac{\delta_{r,2}\ln(3/\delta_{r,1})}{\ln(3/\delta_{r+1,1})}}-1}\bigg)\frac{\phi^2\delta_{r,2}\ln\frac{3}{\delta_{r,1}}}{\ln\frac{3}{\delta_{r+1,1}}}\bigg)^{3/2}
.}
It follows from Eqs.~\eqref{eq:one-div-1}, \eqref{eq:one-div-2}, \& \eqref{eq:one-div-3} that
\ALN{eq:one-div}{\bigg(D_{r+1}-\frac1{\sqrt{D_{r+1}\frac{\delta_{r+1,2}\ln(3/\delta_{r+1,1})}{\ln(3/\delta_{r+2,1})}}-1}\bigg)\frac{\phi^2\delta_{r+1,2}\ln\frac{3}{\delta_{r+1,1}}}{\ln\frac{3}{\delta_{r+2,1}}}\ge\bigg(\bigg(D_{r}-\frac1{\sqrt{D_{r}\frac{\delta_{r,2}\ln(3/\delta_{r,1})}{\ln(3/\delta_{r+1,1})}}-1}\bigg)\frac{\phi^2\delta_{r,2}\ln\frac{3}{\delta_{r,1}}}{\ln\frac{3}{\delta_{r+1,1}}}\bigg)^{3/2}.}
Then by an induction argument using Eq.~\eqref{eq:one-div}, we can show that for every $r\ge1$,
\AL{&\bigg(D_{r}-\frac1{\sqrt{D_{r}\frac{\delta_{r,2}\ln(3/\delta_{r,1})}{\ln(3/\delta_{r+1,1})}}-1}\bigg)\frac{\phi^2\delta_{r,2}\ln\frac{3}{\delta_{r,1}}}{\ln\frac{3}{\delta_{r+1,1}}}\\
\ge{}&\bigg(\bigg(D_{1}-\frac1{\sqrt{D_{1}\frac{\delta_{1,2}\ln(3/\delta_{1,1})}{\ln(3/\delta_{2,1})}}-1}\bigg)\frac{\phi^2\delta_{1,2}\ln\frac{3}{\delta_{1,1}}}{\ln\frac{3}{\delta_{2,1}}}\bigg)^{(3/2)^{r-1}}\\
={}&A^{(3/2)^{r-1}}.}
Hence,
\AL{D_r&>D_{r}-\frac1{\sqrt{D_{r}\frac{\delta_{r,2}\ln(3/\delta_{r,1})}{\ln(3/\delta_{r+1,1})}}-1}\ge\frac{\ln\frac{3}{\delta_{r+1,1}}}{\phi^2\delta_{r,2}\ln\frac{3}{\delta_{r,1}}}A^{(3/2)^{r-1}}\ge\frac1{\phi^2\delta_{1,2}}A^{(3/2)^{k-1}}=\frac{A^{(3/2)^{r-1}}}{(1-\theta)(1-\phi)\phi^2\delta}
.}
Finally, recall the assumption Eq.~\eqref{eq:LEM:one:ASS-j}:
\AL{|\nabla_jf(\BM x)|&>\Big(C_1D+\frac1{D}\Big)\sqrt{2\ln\frac{3}{\theta(1-\phi)\delta}}\|\nabla_{S\SM j}f(\BM x)\|_2+\lambda_{1,|S|}\cdot\epsilon\\
&=\Big(C_1D_1+\frac1{D_1}\Big)\sqrt{2\ln\frac{3}{\delta_{1,1}}}\|\nabla_{S\SM j}f(\BM x)\|_2+\lambda_{1,|S|}\cdot\epsilon
.}
Then, applying Lemma~\ref{LEM:shrink} with $(S,D_1,\delta_{1,1},\delta_{1,2})$ gives a set $S_1\subseteq S$ with
\ALN{eq:one-grad-1}{|S_1\SM j|\le\frac{|S\SM j|}{D_1},\qquad\|\nabla_{S_1\SM j}f(\BM x)\|_2\le\frac{\|\nabla_{S\setminus\{j\}}f(\BM x)\|_2}{\sqrt{D_1\delta_{1,2}}}.}
By Eqs.~\eqref{eq:LEM:one:ASS-j} \& \eqref{eq:one-grad-1},
\AL{|\nabla_jf(\BM x)|&>\Big(C_1D_1+\frac1{D_1}\Big)\sqrt{2\ln\frac{3}{\delta_{1,1}}}\|\nabla_{S\SM j}f(\BM x)\|_2+\lambda_{1,|S|}\cdot\epsilon\\
&\ge\Big(C_1D_1+\frac1{D_1}\Big)\sqrt{2\ln\frac{3}{\delta_{1,1}}}\sqrt{D_1\delta_{1,2}}\|\nabla_{S_1\SM j}f(\BM x)\|_2+\lambda_{1,|S|}\cdot\epsilon\\
&=\sqrt2\Big(C_1+\frac1{D_1^2}\Big)D_1^{3/2}\sqrt{\delta_{1,2}\ln\frac{3}{\delta_{1,1}}}\|\nabla_{S_1\SM j}f(\BM x)\|_2+\lambda_{1,|S|}\cdot\epsilon.}
Since we have shown that $\{D_r\}_{r\ge1}$ is non-decreasing, it follows that
\AL{|\nabla_jf(\BM x)|&>\sqrt2\Big(C_1+\frac1{D_1^2}\Big)D_1^{3/2}\sqrt{\delta_{1,2}\ln\frac{3}{\delta_{1,1}}}\|\nabla_{S_1\SM j}f(\BM x)\|_2+\lambda_{1,|S|}\cdot\epsilon\\
&\ge\sqrt2\Big(C_1+\frac1{D_2^2}\Big)D_1^{3/2}\sqrt{\delta_{1,2}\ln\frac{3}{\delta_{1,1}}}\|\nabla_{S_1\SM j}f(\BM x)\|_2+\lambda_{1,|S|}\cdot\epsilon\\
&=\sqrt2\Big(C_1+\frac1{D_2^2}\Big) D_1^{3/2}\sqrt{\frac{\delta_{1,2}\ln\frac3{\delta_{1,1}}}{\ln\frac3{\delta_{2,1}}}}\sqrt{\ln\frac{3}{\delta_{2,1}}}\|\nabla_{S_1\SM j}f(\BM x)\|_2+\lambda_{1,|S|}\cdot\epsilon\\
&\ge\sqrt2\Big(C_1+\frac1{D_2^2}\Big)\bigg\lfloor D_1^{3/2}\sqrt{\frac{\delta_{1,2}\ln\frac3{\delta_{1,1}}}{\ln\frac3{\delta_{2,1}}}}\bigg\rfloor\sqrt{\ln\frac{3}{\delta_{2,1}}}\|\nabla_{S_1\SM j}f(\BM x)\|_2+\lambda_{1,|S|}\cdot\epsilon\\
&=\sqrt2\Big(C_1+\frac1{D_2^2}\Big)D_2\sqrt{\ln\frac{3}{\delta_{2,1}}}\|\nabla_{S_1\SM j}f(\BM x)\|_2+\lambda_{1,|S|}\cdot\epsilon\\
&=\Big(C_1D_2+\frac1{D_2}\Big)\sqrt{2\ln\frac{3}{\delta_{2,1}}}\|\nabla_{S_1\SM j}f(\BM x)\|_2+\lambda_{1,|S|}\cdot\epsilon
.}
Thus, we can apply Lemma~\ref{LEM:shrink} again to $(S_1,D_2)$. In fact, the same argument holds for all $r\ge1$. We can repeat Lemma~\ref{LEM:shrink} with division schedule $\{D_r\}_{r\ge1}$ for
\AL{R:=\lceil\log_{3/2}\log_A(\phi^2\delta_{1,2}|S\SM j|)\rceil+2=O(\log_{3/2}\log_A|S|)}
times to get $S_R\subseteq S$, and we have
\AL{|S_R\SM j|\le\frac{|S\SM j|}{D_1\cdots D_R}\le\frac{|S\SM j|}{D_R}<1.}
This implies $S_R=\{j\}$, so we have found $j$. The total probability of failure is at most
\AL{\sum_{r=1}^R(\delta_{r,1}+\delta_{r,2})&=\sum_{r=1}^R(\theta(1-\phi)\phi^{r-1}\delta+(1-\theta)(1-\phi)\phi^{r-1}\delta)\\
&=\sum_{r=1}^R(1-\phi)\phi^{r-1}\delta=(1-\phi^R)\delta<\delta.\qedhere}
\end{proof}

\begin{proof}[Proof of Lemma~\ref{LEM:one}]
We can use Lemma~\ref{LEM:one-general} with $D=18$, $\delta=1/2$, $\phi=0.64$, and $\theta=0.08$, and it gives
\AL{&C_2:=\Big(C_1D+\frac1{D}\Big)\sqrt{2\ln\frac{3}{\theta(1-\phi)\delta}}\approx134.88.\qedhere}
\end{proof}

\begin{remark}
Our constant $C_2$ here is nearly $4300$ times smaller than the corresponding constant $C_2'\approx579263$ of the IPW algorithm. % This helps our \Ours{} perform empirically better than the IPW algorithm.
\end{remark}

\ASSEC{Proof of Lemma~\ref{LEM:cand}}{app:pf-all}

Our Lemma~\ref{LEM:cand} is an improvement over Lemma~3.6 in \citet{indyk2011power}. Our main improvement here is by using dependent random partition to bound the worst-case size of candidate groups while \citet{indyk2011power} used independent subsampling. A main difference here is our Lemma~\ref{LEM:perm-cond} for analyzing dependent random partition.

Before proving Lemma~\ref{LEM:cand}, we show two technical lemmas. Let $(n)_m:=\prod_{k=0}^{m-1}(n-k)$ denote the falling factorial. 

\begin{lemma}\label{LEM:perm-cond}
Let $\omega$ be a random permutation of $[d]$. Given $1\le n\le d$, $1\le k\le\big\lceil\frac dn\big\rceil$, define a random set
\AL{S:=\Big\{j\in[d]:\Big\lceil\frac{\omega(j)}{n}\Big\rceil=k\Big\}.}
Note that $|S|=\sum_{q\in[d]}1_{[\lceil q/n\rceil=k]}$ is not random. Then, given any $H\subset[d]$ and any $j\in H$,
\AL{\Prb[S\cap H=\{j\}]=\frac{|S|}{d}\cdot\frac{(d-|S|)_{|H|-1}}{(d-1)_{|H|-1}}.}
Besides that, given any $H\subset[d]$, any $J\subseteq H$ with $|J|\le|S|$, and any $i\in[d]\setminus H$
\AL{\Prb[i\in S\mid S\cap H=J]=\frac{|S|-|J|}{d-|H|}.}
\end{lemma}
\begin{proof}
W.l.o.g., suppose that $H=[|H|]$, $j=1\in H$, $i=|H|+1\in[d]\setminus H$. Recall the brute-force algorithm for generating a random permutation (shown Algorithm~\ref{alg:randperm}). To calculate the desired probabilities, we suppose that the random permutation $\omega$ is indeed generated via Algorithm~\ref{alg:randperm} from now on. 

\begin{algorithm}[h]
\caption{Generating a random permutation}\label{alg:randperm}
\begin{algorithmic}[1]
\REQUIRE{the number $d$ of elements}
\ENSURE{a random permutation of $[d]$}
\FOR{$r\gets1,2,\dots,d$}
    \STATE$\omega(r)\gets\textsf{Unif}([d]\SM{\omega(1),\omega(2),\dots,\omega(r-1)})$
\ENDFOR
\STATE\textbf{return} $\omega$
\end{algorithmic}
\end{algorithm}

Note that $S=\{r\in[d]:\omega(r)\in Q\}$, where $Q:=\{q\in[d]:\lceil q/n\rceil=k\}$. We will use this set $Q$ to rewrite events.

First, let's calculate $\Prb[S\cap H=\{j\}]$. For $j=1$, the event that $j\in S$ is equivalent to the event that $\omega(j)\in Q$; for all other $r\in H\SM j$, the event that $r\notin S$ is equivalent to the event that $\omega(j)\notin Q$. Hence, by the chain rule,
\AL{\Prb[S\cap H=\{j\}]&=\Prb[\omega(1)\in Q,\,\omega(2),\dots,\omega(|H|)\notin Q]\\
&=\Prb[\omega(1)\in Q]\cdot\prod_{r=2}^{|H|}\Prb[\omega(r)\notin Q\mid\omega(1)\in Q,\,\omega(2),\dots,\omega(r-1)\notin Q]\\
&=\frac{|Q|}{d}\cdot\prod_{r=2}^{|H|}\frac{d-|Q|-(r-2)}{d-(r-1)}=\frac{|Q|}{d}\cdot\frac{(d-|Q|)_{|H|-1}}{(d-1)_{|H|-1}}=\frac{|S|}{d}\cdot\frac{(d-|S|)_{|H|-1}}{(d-1)_{|H|-1}}
.}

Next, let's calculate $\Prb[i\in S\mid S\cap H=J]$. Since $\omega(i)=\omega(|H|+1)$ is determined immediately after $\omega(H)$, then given any $\omega(H)$, %with $\omega(H)\cap Q=\omega(J)$, %by the law of total probability,
\AL{
\Prb[\omega(i)\in Q\mid\omega(H)]=\frac{|Q\setminus\omega(H)|}{|[d]\setminus\omega(H)|}=\frac{|Q\setminus\omega(H)|}{|[d]|-|\omega(H)|}=\frac{|Q\setminus\omega(H)|}{d-|H|}%\frac{|Q\setminus\omega(J)|}{|[d]\setminus\omega(H)|}=\frac{|Q|-|\omega(J)|}{|[d]|-|\omega(H)|}=\frac{|S|-|J|}{d-|H|}
.}
Therefore, by the law of total probability and the fact that $|Q|=|S|$,
\AL{
\Prb[i\in S\mid S\cap H=J]&=\Prb[\omega(i)\in Q\mid\omega(H)\cap Q=\omega(J)]\\
&=\Exp_{\omega(H)}\big[\Prb[\omega(i)\in Q\mid\omega(H)]\mathbin{\big|}\omega(H)\cap Q=\omega(J)\big]\\
&=\Exp_{\omega(H)}\Big[\frac{|Q\setminus\omega(H)|}{d-|H|}\mathbin{\Big|}\omega(H)\cap Q=\omega(J)\Big]\\
&=\Exp_{\omega(H)}\Big[\frac{|Q|-|\omega(J)|}{d-|H|}\mathbin{\Big|}\omega(H)\cap Q=\omega(J)\Big]\\
&=\Exp_{\omega(H)}\Big[\frac{|S|-|J|}{d-|H|}\mathbin{\Big|}\omega(H)\cap Q=\omega(J)\Big]\\
&=\frac{|S|-|J|}{d-|H|}
.\qedhere}
\end{proof}

\begin{lemma}\label{LEM:egamma}
For any $1\le s\le d$ and any $0<\gamma<1$,
\AM{\frac{\big(d-\big\lfloor\frac{\gamma d}{s}\big\rfloor\big)_{s-1}}{(d-1)_{s-1}}\ge\RM e^{-\gamma}.}
\end{lemma}

\begin{proof}
\emph{Case 1:} If $s=1$, then
\AL{\frac{\big(d-\big\lfloor\frac{\gamma d}{s}\big\rfloor\big)_{s-1}}{(d-1)_{s-1}}=\frac{\big(d-\big\lfloor\frac{\gamma d}{s}\big\rfloor\big)_{0}}{(d-1)_{0}}=\frac11=1\ge\RM e^{-\gamma}.}

\emph{Case 2:} If $s>1$ and $s\ge\gamma d$, then
\AL{\frac{\big(d-\big\lfloor\frac{\gamma d}{s}\big\rfloor\big)_{s-1}}{(d-1)_{s-1}}=\prod_{k=0}^{s-2}\frac{d-\big\lfloor\frac{\gamma d}{s}\big\rfloor-k}{d-1-k}=\prod_{k=0}^{s-2}\bigg(1+\frac{1-\big\lfloor\frac{\gamma d}{s}\big\rfloor}{d-1-k}\bigg)\ge1\ge\RM e^{-\gamma}.}

\emph{Case 3:} If $1<s<\gamma d$, since $\frac s{\gamma}>s>s-1$, then
\AL{
\frac{\big(d-\big\lfloor\frac{\gamma d}{s}\big\rfloor\big)_{s-1}}{(d-1)_{s-1}}
&=\prod_{k=0}^{s-2}\frac{d-\big\lfloor\frac{\gamma d}{s}\big\rfloor-k}{d-1-k}\ge\prod_{k=0}^{s-2}\frac{d-\frac{\gamma d}{s}-k}{d-1-k}\\
&=\prod_{k=0}^{s-2}\bigg(1-\frac{\gamma}s\bigg(1-\frac{\frac{s}{\gamma}-(1+k)}{d-1-k}\bigg)\bigg)\\
&\ge\prod_{k=0}^{s-2}\bigg(1-\frac{\gamma}s\bigg(1-\frac{\frac{s}{\gamma}-(s-1)}{d-1-k}\bigg)\bigg)\\
&\ge\prod_{k=0}^{s-2}\Big(1-\frac{\gamma}{s}\Big)=\Big(1-\frac{\gamma}{s}\Big)^{s-1}\\
&\ge\lim_{s'\to+\infty}\Big(1-\frac{\gamma}{s'}\Big)^{s'-1}=\RM e^{-\gamma}
.}
The last inequality uses the fact that $\big(1-\frac{\gamma}{s'}\big)^{s'-1}$ is decreasing w.r.t\ $s'$ for $0<\gamma<1$.
\end{proof}

Now we are ready to prove Lemma~\ref{LEM:cand}.

\begin{proof}[Proof of Lemma~\ref{LEM:cand}]
Let $0<\nu<1$ be an absolute constant, and let
\AL{\beta:=\min\bigg\{1,\frac{\rho-\alpha}{1-\rho}\bigg(1-\frac1{-W_{-1}\big({-\frac{(\rho-\alpha)\delta}{\RM e\rho}}\big)}\bigg)\bigg\},\qquad\gamma:=\frac{(1-\nu)\beta}{C_2^2}<1,}
where $W_{-1}:\big[{-\frac1{\RM e}},0\big)\to(-\infty,-1]$ is the $(-1)$ branch of the Lambert product logarithm function.

If $s>\gamma d=\Omega(d)$, then we can simply let $J:=[d]$. Otherwise, let
\AL{n:=\Big\lfloor\frac{\gamma d}s\Big\rfloor,\qquad K:=\Big\lceil\frac{d}{n}\Big\rceil,\qquad m:=\Big\lceil\log_{\frac1{1-\exp(-(1-\nu)/C_2^2)\nu/2}}\frac{\rho}{(\rho-\alpha-(1-\rho)\beta)\delta}\Big\rceil.}
Consider the following procedure, which is equivalent to the procedure in Algorithm~\ref{alg:grace}. Randomly generate $m$ permutations $\omega_1,\dots,\omega_m$ of $[d]$. Apply the algorithm in Lemma~\ref{LEM:one} independently to each $S_{l,k}:=\{i\in[d]:\lceil\omega_l(i)/n\rceil=k\}$ get a $j_{l,k}\in S_{l,k}$ for each $l\in[m]$ and $k\in[K]$. We will show that $J:=\{j_{l,k}\}_{l\in[m],k\in[K]}$ is the desired set.

Let
\AL{
H&:=\argmax_{\begin{subarray}{c}I\subseteq[d]\\|I|=s\end{subarray}}\|\nabla_If(\BM x)\|_2^2,\\
H^*&:=\bigg\{i\in H:|\nabla_if(\BM x)|>\sqrt{\frac{\beta}{s}}\|\nabla_{[d]\setminus H}f(\BM x)\|_2+\lambda_{1,n}\cdot\epsilon\bigg\}.
}
We remark that $H^*\subseteq H$ and that the number of large-gradient dimensions $H^*$ can be less than $s$. 

Fix an $l\in[m]$. For any $k\in[K]$ and any $j\in H$, since $|S_{l,k}|\le n=\big\lfloor\frac{\gamma d}{s}\big\rfloor\le\frac{\gamma d}{s}$, then by Lemma~\ref{LEM:perm-cond} and Lemma~\ref{LEM:egamma},
\ALN{eq:all-1}{
\Prb[S_{l,k}\cap H=\{j\}]&=\frac{|S_{l,k}|}{d}\cdot\frac{(d-|S_{l,k}|)_{s-1}}{(d-1)_{s-1}}\ge\frac{|S_{l,k}|}{d}\cdot\frac{\big(d-\big\lfloor\frac{\gamma d}{s}\big\rfloor\big)_{s-1}}{(d-1)_{s-1}}\\&\ge\frac{|S_{l,k}|}{d}\cdot\RM e^{-\gamma}=\frac{|S_{l,k}|}{d}\cdot\RM e^{-(1-\nu)\beta/C_2^2}\ge\frac{|S_{l,k}|}{d}\cdot\RM e^{-(1-\nu)/C_2^2}
.}
Besides that, by Lemma~\ref{LEM:perm-cond},
\AL{
\Prb[i\in S_{l,k}\mid S_{l,k}\cap H=\{j\}]=\frac{|S_{l,k}|-|\{j\}|}{d-|H|}=\frac{|S_{l,k}|-1}{d-s}
.}
Thus,
\ALN{eq:all-exp}{&\Exp[\|\nabla_{S_{l,k}\setminus H}f(\BM x)\|_2^2\mid\{S_{l,k}\cap H=\{j\}\}]\\
={}&\sum_{i\in[d]\setminus H}(\nabla_if(\BM x))^2\cdot\Prb[i\in S_{l,k}\setminus H\mid S_{l,k}\cap H=\{j\}]\\
={}&\sum_{i\in H}(\nabla_if(\BM x))^2\cdot0+\sum_{i\in[d]\setminus H}(\nabla_if(\BM x))^2\cdot\Prb[i\in S_{l,k}\mid S_{l,k}\cap H=\{j\}]\\
={}&0+\sum_{i\in[d]\setminus H}(\nabla_if(\BM x))^2\cdot\frac{|S_{l,k}|-1}{d-s}={}\frac{|S_{l,k}|-1}{d-s}\|\nabla_{[d]\setminus H}f(\BM x)\|_2^2\\
\le{}&\frac{\frac{\gamma d}{s}-1}{d-s}\|\nabla_{[d]\setminus H}f(\BM x)\|_2^2=\frac{\gamma d-s}{s(d-s)}\|\nabla_{[d]\setminus H}f(\BM x)\|_2^2
.}
By Markov's inequality and Eq.~\eqref{eq:all-exp},
\ALN{eq:all-2}{&\Prb\Big[\|\nabla_{S_{l,k}\setminus H}f(\BM x)\|_2^2\le\frac{\beta}{C_2^2s}\|\nabla_{[d]\setminus H}f(\BM x)\|_2^2\mathbin{\Big|}S_{l,k}\cap H=\{j\}\Big]\\
\ge{}&1-\frac{\Exp[\|\nabla_{S_{l,k}\setminus H}f(\BM x)\|_2^2\mid\{S_{l,k}\cap H=\{j\}\}]}{\frac{\beta}{C_2^2s}\|\nabla_{[d]\setminus H}f(\BM x)\|_2^2}\ge1-\frac{\frac{\gamma d-s}{s(d-s)}\|\nabla_{[d]\setminus H}f(\BM x)\|_2^2}{\frac{\beta}{C_2^2s}\|\nabla_{[d]\setminus H}f(\BM x)\|_2^2}\\
={}&1-\frac{C_2^2(\gamma d-s)}{\beta(d-s)}=1-\frac{(1-\nu)(\gamma d-s)}{\gamma(d-s)}\ge1-(1-\nu)=\nu
.}
Thus, for any $j\in H^*$ and any $k\in[K]$, by the definition of $H^*$ and Eqs.~\eqref{eq:all-1} \& \eqref{eq:all-2},
\ALN{eq:all-3}{
\Prb\{j=j_{l,k}\}={}&\Prb\{j\in S_{l,k},\text{ and Lemma~\ref{LEM:one} finds }j\}\\
\ge{}&\Prb\{j\in S_{l,k},\,|\nabla_jf(\BM x)|>C_2\|\nabla_{S_{l,k}\SM j}f(\BM x)\|_2+\lambda_{1,|S_{l,k}|}\cdot\epsilon,\text{ and Lemma~\ref{LEM:one} finds }j\}\\
\ge{}&\Prb\{j\in S_{l,k},\,|\nabla_jf(\BM x)|>C_2\|\nabla_{S_{l,k}\SM j}f(\BM x)\|_2+\lambda_{1,|S_{l,k}|}\cdot\epsilon\}\cdot\frac12\\
\ge{}&\Prb\{S_{l,k}\cap H=\{j\},\,|\nabla_jf(\BM x)|>C_2\|\nabla_{S_{l,k}\setminus H}f(\BM x)\|_2+\lambda_{1,|S_{l,k}|}\cdot\epsilon\}\cdot\frac12\\
\ge{}&\Prb\{S_{l,k}\cap H=\{j\},\,|\nabla_jf(\BM x)|>C_2\|\nabla_{S_{l,k}\setminus H}f(\BM x)\|_2+\lambda_{1,n}\cdot\epsilon\}\cdot\frac12\\
\ge{}&\Prb\Big\{S_{l,k}\cap H=\{j\},\,\|\nabla_{S_{l,k}\setminus H}f(\BM x)\|_2\le\frac{\beta}{C_2^2s}\|\nabla_{[d]\setminus H}f(\BM x)\|_2\Big\}\cdot\frac12\\
={}&\Prb[S_{l,k}\cap H=\{j\}]\cdot\Prb\Big[\|\nabla_{S_{l,k}\setminus H}f(\BM x)\|_2^2\le\frac{\beta}{C_2^2s}\|\nabla_{[d]\setminus H}f(\BM x)\|_2^2\mathbin{\Big|}S_{l,k}\cap H=\{j\}\Big]\cdot\frac12\\
\ge{}&\Big(\frac{|S_{l,k}|}{d}\cdot\RM e^{-(1-\nu)/C_2^2}\Big)\cdot\nu\cdot\frac12
,}
Since $\sum_{k\in[K]}|S_{l,k}|=d$ by definition, then by Eq.~\eqref{eq:all-3},
\ALN{eq:all-4}{&\Prb\{j\in\{j_{l,k}\}_{k\in[K]}\}=\sum_{k\in[K]}\Prb\{j=j_{l,k}\}
\ge{}\sum_{k\in[K]}\Big(\frac{|S_{l,k}|}{d}\cdot\RM e^{-(1-\nu)/C_2^2}\Big)\cdot\nu\cdot\frac12=\RM e^{-(1-\nu)/C_2^2}\cdot\frac\nu2
.}
By Eq.~\eqref{eq:all-4} and the definition of $m$,
\AL{\Exp[\|\nabla_{H^*\setminus J}f(\BM x)\|_2^2]&=\sum_{j\in H^*}(\nabla_jf(\BM x))^2\cdot\Prb\{j\notin J\}\\
&=\sum_{j\in H^*}(\nabla_jf(\BM x))^2\prod_{l\in[m]}(1-\Prb\{j\in\{j_{l,k}\}_{k\in[K]}\})\\
&\le\sum_{j\in H^*}(\nabla_jf(\BM x))^2\Big(1-\RM e^{-(1-\nu)/C_2^2}\cdot\frac\nu2\Big)^m\\
&=\Big(1-\RM e^{-(1-\nu)/C_2^2}\cdot\frac\nu2\Big)^m\|\nabla_{H^*}f(\BM x)\|_2^2\\
&\le\frac{(\rho-\alpha-(1-\rho)\beta)\delta}{\rho}\|\nabla_{H^*}f(\BM x)\|_2^2
.}
Hence, by Markov's inequality, with probability at least $1-\delta$,
\ALN{eq:all-5}{\|\nabla_{H^*\setminus J}f(\BM x)\|_2^2\le{}&\frac{\Exp[\|\nabla_{H^*\setminus J}f(\BM x)\|_2^2]}{\delta}\\
\le{}&\frac{\rho-\alpha-(1-\rho)\beta}{\rho}\|\nabla_{H^*}f(\BM x)\|_2^2\\
\le{}&\frac{\rho-\alpha-(1-\rho)\beta}{\rho}\|\nabla_{H}f(\BM x)\|_2^2
.}
In addition to that, by the definition of $H^*$, Lemma~\ref{LEM:lip-cont}, and convexity and monotonicity of $\lambda_{1,n}^2$ w.r.t.\ $n$,
\ALN{eq:all-6}{\|\nabla_{H\setminus H^*}f(\BM x)\|_2^2&\le s\|\nabla_{H\setminus H^*}f(\BM x)\|_\infty^2\\
&\le s\Big(\sqrt{\frac{\beta}{s}}\|\nabla_{[d]\setminus H}f(\BM x)\|_2+\lambda_{1,n}\cdot\epsilon\Big)^2\\
&=\beta\|\nabla_{[d]\setminus H}f(\BM x)\|_2^2+2\sqrt{s\beta}\|\nabla_{[d]\setminus H}f(\BM x)\|_2\lambda_{1,n}\epsilon+s\lambda_{1,n}^2\epsilon^2\\
&\le\beta\|\nabla_{[d]\setminus H}f(\BM x)\|_2^2+2\sqrt{s1}L_0\lambda_{1,n}\epsilon+s\lambda_{1,n}^2\epsilon^2\\
&\le\beta\|\nabla_{[d]\setminus H}f(\BM x)\|_2^2+2L_0\lambda_{1,sn}\epsilon+\lambda_{1,sn}^2\epsilon^2\\
&\le\beta\|\nabla_{[d]\setminus H}f(\BM x)\|_2^2+2L_0\lambda_{1,\gamma d}\epsilon+\lambda_{1,\gamma d}^2\epsilon^2\\
&\le\beta\|\nabla_{[d]\setminus H}f(\BM x)\|_2^2+2L_0\lambda_{1,d}\epsilon+\lambda_{1,d}^2\epsilon^2\\
&=\beta(\|\nabla f(\BM x)\|_2^2-\|\nabla_{H}f(\BM x)\|_2^2)+\lambda_{2,d}\epsilon+\lambda_{1,d}^2\epsilon^2
.}
It follows from Eqs.~\eqref{eq:all-6} \& \eqref{eq:all-5} and Assumption~\ref{ASS:sparse} that
\ALN{eq:all-J}{\|\nabla_Jf(\BM x)\|_2^2\ge{}&\|\nabla_Hf(\BM x)\|_2^2-\|\nabla_{H\setminus H^*}f(\BM x)\|_2^2-\|\nabla_{H^*\setminus J}f(\BM x)\|_2^2\\
\ge{}&\|\nabla_Hf(\BM x)\|_2^2-(\beta(\|\nabla f(\BM x)\|_2^2-\|\nabla_{H}f(\BM x)\|_2^2)+\lambda_{2,d}\epsilon+\lambda_{1,d}^2\epsilon^2)-\frac{\rho\!-\!\alpha\!-\!(1\!-\!\rho)\beta}{\rho}\|\nabla_{H}f(\BM x)\|_2^2\\
={}&\frac{\alpha+\beta}{\rho}\|\nabla_{H}f(\BM x)\|_2^2-\beta\|\nabla f(\BM x)\|_2^2-\lambda_{2,d}\epsilon-\lambda_{1,d}^2\epsilon^2\\
\ge{}&\frac{\alpha+\beta}{\rho}\rho\|\nabla f(\BM x)\|_2^2-\beta\|\nabla f(\BM x)\|_2^2-\lambda_{2,d}\epsilon-\lambda_{1,d}^2\epsilon^2\\
={}&\alpha\|\nabla f(\BM x)\|_2^2-\lambda_{2,d}\epsilon-\lambda_{1,d}^2\epsilon^2
.}
Since $K=O(\gamma s)$, then the size of $J$ is at most $mK\le C_{\rho,\alpha,\delta}s$ where $C_{\rho,\alpha,\delta}:=O(m\gamma)$, and the total number of queries is at most
\AL{&O\bigg(\sum_{l\in[m]}\sum_{k\in[K]}\log\log|S_{l,k}|\bigg)\le O\Big(mK\log\log\Big\lfloor\frac{\gamma d}s\Big\rfloor\Big)=O\Big(C_{\rho,\alpha,\delta}s\log\log\frac{d}s\Big).\qedhere}
\end{proof}

\ASSEC{Proof of Theorem~\ref{THM:grad}}{app:pf-grad}

While previous ZOO works try to upper-bound $\|\BM g-\nabla f(\BM x)\|_2$, here we only lower-bound $\langle\nabla f(\BM x),\BM g\rangle$ because ensuring this weaker condition suffices for the $O\big(\frac1T\big)$ convergence and needs fewer queries than ensuring $\|\BM g-\nabla f(\BM x)\|_2$ to be small. This insight also helps us to simplify the proofs of Theorems~\ref{THM:grad} \& \ref{THM:gd-nc}.

\begin{proof}[Proof of Theorem~\ref{THM:grad}]
Since $\alpha<\rho$, we can pick $\alpha'$ such that $\alpha<\alpha'<\rho$. Let $\delta:=1-\frac{\alpha}{\alpha'}$. 

Define $C_{\rho,\alpha}:=C_{\rho,\alpha',\delta}$, where $C_{\rho,\alpha',\delta}$ is the constant in the proof of Lemma~\ref{LEM:cand}.
Using Lemma~\ref{LEM:cand} w.r.t.\ $(\BM x,\alpha',\delta)$, we can find a set $J\subseteq[d]$ of size at most $C_{\rho,\alpha',\delta}s=C_{\rho,\alpha}s$ such that with probability at least $1-\delta$, Eq.~\eqref{eq:all-J} holds.

For each dimension $i\in[d]\setminus J$, let $g_i:=0$. For each dimension $j\in J$, recall that
\AL{g_j:=\frac{f(\BM x+\epsilon\BM e_j)-f(\BM x)}{\epsilon}.}
By Lemma~\ref{LEM:lip-smooth}, for every $j\in J$,
\AL{|g_j-\nabla_jf(\BM x)|=\frac{|f(\BM x+\epsilon\BM e_j)-f(\BM x)-\langle\nabla f(\BM x),\epsilon\BM e_j\rangle|}{\epsilon}\le\frac{L_1\epsilon^2/2}{\epsilon}=\frac{L_1\epsilon}{2}.}
This implies
\AL{\|\BM g_J-\nabla_{J}f(\BM x)\|_\infty&\le\frac{L_1\epsilon}2.
}
It follows that
\ALN{eq:grad-diff}{\|\BM g_J-\nabla_Jf(\BM x)\|_2\le\sqrt{|J|}\cdot\|\nabla_Jf(\BM x)-\BM g\|_\infty\le\sqrt{C_{\rho,\alpha}s}\cdot\frac{L_1\epsilon}2=\frac{\sqrt{C_{\rho,\alpha}}L_1}2\sqrt s\epsilon.}
Hence,
\ALN{eq:grad-est-upper}{\|\BM g\|_2=\|\BM g_J\|_2\le\|\nabla_Jf(\BM x)\|_2+\|\BM g_J-\nabla_Jf(\BM x)\|_2\le\|\nabla f(\BM x)\|_2+\frac{\sqrt{C_{\rho,\alpha}}L_1}2\sqrt s\epsilon.}
Besides that, since $\BM g_{[d]\setminus J}=\BM 0$, then by the Cauchy--Schwarz inequality and Eq.~\eqref{eq:grad-diff},
\AL{\langle\nabla f(\BM x),\BM g\rangle&=\langle\nabla_Jf(\BM x),\BM g_J\rangle\\
&=\langle\nabla_Jf(\BM x),\nabla_Jf(\BM x)\rangle-\langle\nabla_Jf(\BM x),\nabla_Jf(\BM x)-\BM g_J\rangle\\
&=\|\nabla_Jf(\BM x)\|_2^2-\langle\nabla_Jf(\BM x),\nabla_Jf(\BM x)-\BM g_J\rangle\\
&\ge\|\nabla_Jf(\BM x)\|_2^2-\|\nabla_Jf(\BM x)\|_2\cdot\|\nabla_Jf(\BM x)-\BM g_J\|_2\\
&\ge\|\nabla_Jf(\BM x)\|_2^2-\|\nabla f(\BM x)\|_2\cdot\|\nabla_Jf(\BM x)-\BM g_J\|_2\\
&\ge\|\nabla_Jf(\BM x)\|_2^2-L_0\cdot\Big(\frac{\sqrt{C_{\rho,\alpha}}L_1}2\sqrt s\epsilon\Big)\\
&=\|\nabla_Jf(\BM x)\|_2^2-\frac{\sqrt{C_{\rho,\alpha}}L_0L_1}2\sqrt s\epsilon\label{eq:grad-succ}\\
&\ge{-\frac{\sqrt{C_{\rho,\alpha}}L_0L_1}2\sqrt s\epsilon}
.}
Furthermore, recall that Lemma~\ref{LEM:cand} succeeds with probability at least $1-\delta$:
\AL{\Prb[\|\nabla_Jf(\BM x)\|_2^2\ge\alpha'\|\nabla f(\BM x)\|_2^2-\lambda_{2,d}\epsilon-\lambda_{1,d}^2\epsilon^2\mid\BM x]\ge1-\delta.}
Then by Eq.~\eqref{eq:grad-succ}, the event
\AL{E:=\bigg[\langle\nabla f(\BM x),\BM g\rangle\ge\alpha'\|\nabla f(\BM x)\|_2^2-\lambda_{2,d}\epsilon-\lambda_{1,d}^2\epsilon^2-\frac{\sqrt{C_{\rho,\alpha}}L_0L_1}2\sqrt s\epsilon\bigg]}
has probability
\AL{\Prb[E\mid\BM x]\ge\Prb\big[\|\nabla_Jf(\BM x)\|_2^2\ge\alpha'\|\nabla f(\BM x)\|_2^2-\lambda_{2,d}\epsilon-\lambda_{1,d}^2\epsilon^2\mathbin{\big|}\BM x\big]\ge1-\delta.}
Finally, by the decomposition $1=1_{E}+1_{E\Cp}$,
\AL{&\Exp[\langle\nabla f(\BM x),\BM g\rangle\mid\BM x]=\Exp[\langle\nabla f(\BM x),\BM g\rangle1_{E}\mid\BM x_t]+\Exp[\langle\nabla f(\BM x),\BM g\rangle1_{E\Cp}\mid\BM x]\\
\ge{}&\Exp\bigg[\bigg(\alpha'\|\nabla f(\BM x)\|_2^2-\lambda_{2,d}\epsilon-\lambda_{1,d}^2\epsilon^2-\frac{\sqrt{C_{\rho,\alpha}}L_0L_1}2\sqrt s\epsilon\bigg)1_{E}\mathbin{\bigg|}\BM x\bigg]+\Exp\bigg[{-\frac{\sqrt{C_{\rho,\alpha}}L_0L_1}2\sqrt s\epsilon}1_{E\Cp}\mathbin{\bigg|}\BM x\bigg]\\
={}&\bigg(\alpha'\|\nabla f(\BM x)\|_2^2-\lambda_{2,d}\epsilon-\lambda_{1,d}^2\epsilon^2-\frac{\sqrt{C_{\rho,\alpha}}L_0L_1}2\sqrt s\epsilon\bigg)\cdot\Prb[E\mid\BM x]-\frac{\sqrt{C_{\rho,\alpha}}L_0L_1}2\sqrt s\epsilon\cdot\Prb[E\Cp\mid\BM x]\\
\ge{}&\bigg(\alpha'\|\nabla f(\BM x)\|_2^2-\lambda_{2,d}\epsilon-\lambda_{1,d}^2\epsilon^2-\frac{\sqrt{C_{\rho,\alpha}}L_0L_1}2\sqrt s\epsilon\bigg)\cdot(1-\delta)-\frac{\sqrt{C_{\rho,\alpha}}L_0L_1}2\sqrt s\epsilon\cdot\delta\\
={}&\alpha'(1-\delta)\|\nabla f(\BM x)\|_2^2-\bigg((1-\delta)\lambda_{2,d}+\frac{\sqrt{C_{\rho,\alpha}}}2L_0L_1\sqrt s\bigg)\epsilon-(1-\delta)\lambda_{1,d}^2\epsilon^2\\
={}&\alpha\|\nabla f(\BM x)\|_2^2-\lambda_{3,d,s}\epsilon-\lambda_{4,d}\epsilon^2,\label{eq:grad-prod}}
with $\lambda_{3,d,s}=O(\lambda_{2,d}+L_0L_1\sqrt s)$ and $\lambda_{4,d}=O(\lambda_{1,d}^2)$.
\end{proof}

\ASSEC{Proof of Theorem~\ref{THM:gd-nc}}{app:pf-opt}

With the help of Theorem~\ref{THM:grad}, a standard argument (e.g., Section~1.2.3, \citealp{nesterov2018lectures}) shows the $O\big(\frac1T\big)$ rate of convergence.

\begin{proof}[Proof of Theorem~\ref{THM:gd-nc}]
%Since $\varDelta>\frac{2L_2f(\BM x_1)-f_*}{\rho^2}$, then we can pick $\alpha$ with $0<\alpha<\rho$ such that $\frac{2L_2f(\BM x_1)-f_*}{\alpha^2}<\varDelta$. 
Pick $0<\alpha<\rho$, and apply Theorem~\ref{THM:grad} with $(\rho,\alpha)$ to get gradient estimates $\{\BM g_t\}_{t\ge1}$. Let $C_{\rho,\alpha}$ denote the constant in the proof of Theorem~\ref{THM:grad}, and let $C_\rho:=C_{\rho,\alpha}$. 

For every $t\ge1$, by Lemma~\ref{LEM:lip-smooth} and Theorem~\ref{THM:grad},
\AL{
f(\BM x_{t+1})={}&f(\BM x_t-\eta\BM g_t)\le f(\BM x_t)-\eta\langle\nabla f(\BM x_t),\BM g_t\rangle+\frac{L_1\eta^2}2\|\BM g_t\|_2^2\\
\le{}&f(\BM x_t)-\eta\langle\nabla f(\BM x_t),\BM g_t\rangle+\frac{L_1\eta^2}2\bigg(\|\nabla f(\BM x_t)\|_2+\frac{\sqrt{C_{\rho,\alpha}}L_1}2\sqrt s\epsilon\bigg)^{\!2}\\
={}&f(\BM x_t)-\eta\langle\nabla f(\BM x_t),\BM g_t\rangle+\frac{L_1\eta^2}2\|\nabla f(\BM x_t)\|_2^2+\frac{\sqrt{C_{\rho}}\eta^2L_1^2\|\nabla f(\BM x_t)\|_2}2\sqrt s\epsilon_t+\frac{C_\rho\eta^2L_1^3}8s\epsilon_t^2\\
\le{}&f(\BM x_t)-\eta\langle\nabla f(\BM x_t),\BM g_t\rangle+\frac{L_1\eta^2}{2}\|\nabla f(\BM x_t)\|_2^2+\frac{\sqrt{C_{\rho}}\eta^2L_0L_1^2}2\sqrt s\epsilon_t+\frac{C_\rho\eta^2L_1^3}8s\epsilon_t^2
%\le{}&f(\BM x_t)-\eta\langle\nabla f(\BM x_t),\BM g_t\rangle+\frac{\sqrt{C_{\rho}}\eta^2L_0L_1^2}2\sqrt s\epsilon_t+\frac{\sqrt{C_{\rho}s}L_1L_2^2\eta^2}2\epsilon_t+\frac{C_\rho sL_2^3\eta^2}8\epsilon_t^2
.}
Rearranging the terms gives
\ALN{eq:opt-term}{\eta\langle\nabla f(\BM x_t),\BM g_t\rangle-\frac{L_1\eta^2}{2}\|\nabla f(\BM x_t)\|_2^2\le(f(\BM x_t)-f(\BM x_{t+1}))+\frac{\sqrt{C_{\rho}}\eta^2L_0L_1^2}2\sqrt s\epsilon_t+\frac{C_\rho\eta^2L_1^3}8s\epsilon_t^2.}
One the one hand, by a telescoping sum of \eqref{eq:opt-term},
\ALN{eq:opt-cum-lhs}{
&\sum_{t=1}^T\bigg(\eta\langle\nabla f(\BM x_t),\BM g_t\rangle-\frac{L_1\eta^2}{2}\|\nabla f(\BM x_t)\|_2^2\bigg)\\
\le{}&\sum_{t=1}^T(f(\BM x_t)-f(\BM x_{t+1}))+\frac{\sqrt{C_{\rho}}\eta^2L_0L_1^2}2\sqrt s\sum_{t=1}^T\epsilon_t+\frac{C_\rho\eta^2L_1^3}8s\sum_{t=1}^T\epsilon_t^2\\
={}&f(\BM x_1)-f(\BM x_{T+1})+\frac{\sqrt{C_{\rho}}\eta^2L_0L_1^2}2\sqrt s\sum_{t=1}^T\epsilon_t+\frac{C_\rho\eta^2L_1^3}8s\sum_{t=1}^T\epsilon_t^2\\
\le{}&f(\BM x_1)-f_*+\frac{\sqrt{C_{\rho}}\eta^2L_0L_1^2}2\sqrt s\sum_{t=1}^T\epsilon_t+\frac{C_\rho\eta^2L_1^3}8s\sum_{t=1}^T\epsilon_t^2
%={}&f(\BM x_1)-f(\BM x_{T+1})+\sum_{t=1}^T\eta_t^2\Big(\frac{\sqrt{C_{\rho}s}L_1L_2^2}2\epsilon_t+\frac{C_\rho sL_2^3}8\epsilon_t^2\Big)\\
%\le{}&f(\BM x_1)-f_*+\sum_{t=1}^T\eta_t^2\Big(\frac{\sqrt{C_{\rho}s}L_1L_2^2}2\epsilon_t+\frac{C_\rho sL_2^3}8\epsilon_t^2\Big)
.}
On the other hand, by the law of total expectation and Eq.~\eqref{eq:grad-prod} in Theorem~\ref{THM:grad},
\ALN{eq:opt-cum-rhs}{&\Exp\bigg[\sum_{t=1}^T\bigg(\eta\langle\nabla f(\BM x_t),\BM g_t\rangle-\frac{L_1\eta^2}{2}\|\nabla f(\BM x_t)\|_2^2\bigg)\bigg]\\
={}&\sum_{t=1}^T\ExpOp\bigg[\eta\ExpOp[\langle\nabla f(\BM x_t),\BM g_t\rangle\mid\BM x_t]-\frac{L_1\eta^2}{2}\|\nabla f(\BM x_t)\|_2^2\bigg]\\
\ge{}&\sum_{t=1}^T\Exp\bigg[\eta(\alpha\|\nabla f(\BM x_t)\|_2^2-\lambda_{3,d,s}\epsilon_t-\lambda_{4,d}\epsilon_t^2)-\frac{L_1\eta^2}{2}\|\nabla f(\BM x_t)\|_2^2\bigg]\\
={}&\bigg(\alpha\eta-\frac{L_1\eta^2}2\bigg)\Exp\bigg[\sum_{t=1}^T\|\nabla f(\BM x_t)\|_2^2\bigg]-\eta\lambda_{3,d,s}\sum_{t=1}^T\epsilon_t-\eta\lambda_{4,d}\sum_{t=1}^T\epsilon_t^2
.}
Let the step size $\eta:=\frac{\alpha}{L_1}$, so $\alpha\eta-\frac{L_1\eta^2}2=\frac{\alpha^2}{2L_1}$. Combining Eqs.~\eqref{eq:opt-cum-lhs} \& \eqref{eq:opt-cum-rhs} gives
\AM{\frac{\alpha^2}{2L_1}\Exp\bigg[\sum_{t=1}^T\|\nabla f(\BM x_t)\|_2^2\bigg]-\eta\lambda_{3,d,s}\sum_{t=1}^T\epsilon_t-\eta\lambda_{4,d}\sum_{t=1}^T\epsilon_t^2\le f(\BM x_1)-f_*+\frac{\sqrt{C_{\rho}}\eta^2L_0L_1^2}2\sqrt s\sum_{t=1}^T\epsilon_t+\frac{C_\rho\eta^2L_1^3}8s\sum_{t=1}^T\epsilon_t^2.}
Rearranging the terms gives
\ALN{eq:opt-sum}{\Exp\bigg[\sum_{t=1}^T\|\nabla f(\BM x_t)\|_2^2\bigg]\le\frac{2L_1}{\alpha^2}(f(\BM x_1)-f_*)+\frac{2L_1\big(\eta\lambda_{3,d,s}+\frac{\sqrt{C_{\rho}}\eta^2L_0L_1^2}2\sqrt s\big)}{\alpha^2}\sum_{t=1}^T\epsilon_t+\frac{2L_1\big(\eta\lambda_{4,d}+\frac{C_\rho\eta^2L_1^3}8s\big)}{\alpha^2}\sum_{t=1}^T\epsilon_t^2.}
For each $t\ge1$, let $\epsilon_t:=\epsilon_1/t^2$ where
\AL{
\epsilon_1:=\frac{-\frac{\uppi^2}6\!\cdot\!\frac{2L_1\big(\eta\lambda_{3,d,s}+\frac{\sqrt{C_{\rho}}\eta^2L_0L_1^2}2\sqrt s\big)}{\alpha^2}\!+\!\!\sqrt{\!\big(\frac{\uppi^2}6\!\cdot\!\frac{2L_1\big(\eta\lambda_{3,d,s}+\frac{\sqrt{C_{\rho}}\eta^2L_0L_1^2}2\sqrt s\big)}{\alpha^2}\big)^{\!2}\!\!+\!4\!\cdot\!\frac{\uppi^4}{90}\!\cdot\!\frac{2L_1\big(\eta\lambda_{4,d}+\frac{C_\rho\eta^2L_1^3s}8\big)}{\alpha^2}\!\cdot\!\varDelta\!}}
{2\!\cdot\!\frac{\uppi^4}{90}\!\cdot\!\frac{2L_1\big(\eta\lambda_{4,d}+\frac{C_\rho\eta^2L_1^3s}8\big)}{\alpha^2}}.}
It is clear that $\epsilon_1>0$. Since $\sum_{t=1}^T\frac1{t^2}\le\frac{\uppi^2}6$, and $\sum_{t=1}^T\frac1{t^4}\le\frac{\uppi^4}{90}$, then by Eq.~\eqref{eq:opt-sum},
\AL{\Exp\bigg[\sum_{t=1}^T\|\nabla f(\BM x_t)\|_2^2\bigg]
&\le\frac{2L_1}{\alpha^2}(f(\BM x_1)-f_*)+\frac{2L_1\big(\eta\lambda_{3,d,s}+\frac{\sqrt{C_{\rho}}\eta^2L_0L_1^2}2\sqrt s\big)}{\alpha^2}\cdot\frac{\uppi^2}6\epsilon_1+\frac{2L_1\big(\eta\lambda_{4,d}+\frac{C_\rho\eta^2L_1^3}8s\big)}{\alpha^2}\cdot\frac{\uppi^4}{90}\epsilon_1^2\label{eq:exp-cum-upper-bound}\\
&=\frac{2L_1}{\alpha^2}(f(\BM x_1)-f_*)+\varDelta
.}
It follows that
\AL{&\Exp\Big[\min_{t=1,\dots,T}\|\nabla f(\BM x_t)\|_2^2\Big]\le\Exp\bigg[\frac1T\sum_{t=1}^T\|\nabla f(\BM x_t)\|_2^2\bigg]\le\frac{\frac{2L_1}{\alpha^2}(f(\BM x_1)-f_*)+\varDelta}T.\qedhere}
\end{proof}

\ASSEC{Proof of Theorem~\ref{THM:gd-prob}}{app:pf-gd-prob}

Note that the proof of Theorem~\ref{THM:gd-nc} implies that $\Exp[\sum_{t=1}^\infty\|\nabla f(\BM x_t)\|_2^2]$ is bounded. Thus, we can give a high-probability convergence analysis through a similar argument with the proof of Theorem~\ref{THM:gd-nc} and by applying Markov's inequality w.r.t.\ a constructed non-negative random variable.

\begin{proof}
Let $\alpha:=L_1\eta<\rho$, and let $C_{\rho,\alpha}$ denote the constant in the proof of Theorem~\ref{THM:grad}. For each $t\ge1$, let $\epsilon_t:=\epsilon_1/t^2$ with $\epsilon_1>0$ to be determined later. 

Let $\mu_1:=\frac{2L_1\big(\eta\lambda_{3,d,s}+\frac{\sqrt{C_{\rho,\alpha}}\eta^2L_0L_1^2}2\sqrt s\big)}{\alpha^2}\cdot\frac{\uppi^2}6$ and $\mu_2:=\frac{2L_1\big(\eta\lambda_{4,d}+\frac{C_{\rho,\alpha}\eta^2L_1^3}8s\big)}{\alpha^2}\cdot\frac{\uppi^4}{90}$. Then by Eq.~\eqref{eq:exp-cum-upper-bound},
\AL{\Exp\bigg[\sum_{t=1}^\infty\|\nabla f(\BM x_t)\|_2^2\bigg]
&\le\frac{2L_1}{\alpha^2}(f(\BM x_1)-f_*)+\mu_1\epsilon_1+\mu_2\epsilon_1^2=\frac{2}{L_1\eta^2}(f(\BM x_1)-f_*)+\mu_1\epsilon_1+\mu_2\epsilon_1^2<\infty.}
Let $J_t$ be the candidate set in step $t$. Then by the Cauchy--Schwarz inequality and Eq.~\eqref{eq:grad-diff},
\AL{\langle\nabla f(\BM x_t),\BM g_t\rangle&=\langle\nabla_{J_t}f(\BM x_t),(\BM g_t)_{J_t}\rangle\\
&=\langle\nabla_{J_t}f(\BM x_t),\nabla_{J_t}f(\BM x_t)\rangle-\nabla_{J_t}f(\BM x_t),\nabla_{J_t}f(\BM x_t)-(\BM g_t)_{J_t}\rangle\\
&=\|\nabla_{J_t}f(\BM x_t)\|_2^2-\langle\nabla_{J_t}f(\BM x_t),\nabla_{J_t}f(\BM x_t)-(\BM g_t)_{J_t}\rangle\\
&\le\|\nabla_{J_t}f(\BM x_t)\|_2^2+\|\nabla_{J_t}f(\BM x_t)\|_2\cdot\|\nabla_{J_t}f(\BM x_t)-(\BM g_t)_{J_t}\|_2\\
&\le\|\nabla f(\BM x_t)\|_2^2+\|\nabla f(\BM x)\|_2\cdot\|\nabla_{J_t}f(\BM x_t)-(\BM g_t)_{J_t}\|_2\\
&\le\|\nabla f(\BM x_t)\|_2^2+L_0\cdot\Big(\frac{\sqrt{C_{\rho,\alpha}}L_1}2\sqrt s\epsilon_t\Big)\\
&=\|\nabla f(\BM x_t)\|_2^2+\frac{\sqrt{C_{\rho,\alpha}}L_0L_1}2\sqrt s\epsilon_t
.}
This implies that
\AL{&\sum_{t=1}^\infty\|\nabla f(\BM x_t)\|_2^2-\sum_{t=1}^\infty\langle\nabla f(\BM x_t),\BM g_t\rangle+\frac{\sqrt{C_{\rho,\alpha}}L_0L_1}2\sqrt s\cdot\frac{\uppi^2}6\epsilon_1\\
={}&\sum_{t=1}^\infty\Big(\|\nabla f(\BM x_t)\|_2^2-\langle\nabla f(\BM x_t),\BM g_t\rangle+\frac{\sqrt{C_{\rho,\alpha}}L_0L_1}2\sqrt s\epsilon_t\Big)\ge0.}
Thus, by Markov's inequality and Eq.~\eqref{eq:grad-prod}, with probability at least $1-\beta$,
\AL{
&\sum_{t=1}^\infty\|\nabla f(\BM x_t)\|_2^2-\sum_{t=1}^\infty\langle\nabla f(\BM x_t),\BM g_t\rangle+\frac{\sqrt{C_{\rho,\alpha}}L_0L_1}2\sqrt s\cdot\frac{\uppi^2}6\epsilon_1\\
\le{}&\frac1\beta\Exp\bigg[\sum_{t=1}^\infty\|\nabla f(\BM x_t)\|_2^2-\sum_{t=1}^\infty\langle\nabla f(\BM x_t),\BM g_t\rangle+\frac{\sqrt{C_{\rho,\alpha}}L_0L_1}2\sqrt s\cdot\frac{\uppi^2}6\epsilon_1\bigg]\\
={}&\frac1\beta\bigg(\Exp\bigg[\sum_{t=1}^\infty\|\nabla f(\BM x_t)\|_2^2\bigg]-\Exp\bigg[\sum_{t=1}^\infty\Exp[\langle\nabla f(\BM x_t),\BM g_t\rangle\mid\BM x_t]\bigg]+\frac{\sqrt{C_{\rho,\alpha}}L_0L_1}2\sqrt s\cdot\frac{\uppi^2}6\epsilon_1\bigg)\\
\le{}&\frac1\beta\bigg(\Exp\bigg[\sum_{t=1}^\infty\|\nabla f(\BM x_t)\|_2^2\bigg]-\Exp\bigg[\sum_{t=1}^\infty(\alpha\|\nabla f(\BM x_t)\|_2^2-\lambda_{3,d,s}\epsilon_t-\lambda_{4,d}\epsilon_t^2)\bigg]+\frac{\sqrt{C_{\rho,\alpha}}L_0L_1}2\sqrt s\cdot\frac{\uppi^2}6\epsilon_1\bigg)\\
={}&\frac1\beta\bigg((1-\alpha)\Exp\bigg[\sum_{t=1}^\infty\|\nabla f(\BM x_t)\|_2^2\bigg]+\Big(\lambda_{3,d,s}+\frac{\sqrt{C_{\rho,\alpha}}L_0L_1}2\sqrt s\Big)\cdot\frac{\uppi^2}6\epsilon_1+\lambda_{4,d}\cdot\frac{\uppi^4}{90}\epsilon_1^2\bigg)\\
\le{}&\frac1\beta\bigg((1-\alpha)\Big(\frac{2}{L_1\eta^2}(f(\BM x_1)-f_*)+\mu_1\epsilon_1+\mu_2\epsilon_1^2\Big)+\Big(\lambda_{3,d,s}+\frac{\sqrt{C_{\rho,\alpha}}L_0L_1}2\sqrt s\Big)\cdot\frac{\uppi^2}6\epsilon_1+\lambda_{4,d}\cdot\frac{\uppi^4}{90}\epsilon_1^2\bigg)\label{eq:cov-prob-bound}
.}
With $\mu_3:=\frac{\mu_1}\beta+\big(\lambda_{3,d,s}+\big(\frac1\beta-1\big)\frac{\sqrt{C_{\rho,\alpha}}L_0L_1}2\sqrt s\big)\cdot\frac{\uppi^2}6$ and $\mu_4:=\frac{\mu_2}\beta+\frac{\lambda_{4,d}}{\beta}\cdot\frac{\uppi^4}{90}$, rearranging the terms in Eq.~\eqref{eq:cov-prob-bound} gives
\ALN{eq:opt-cum-rhs-prob}{
\sum_{t=1}^\infty\langle\nabla f(\BM x_t),\BM g_t\rangle
&\ge\sum_{t=1}^\infty\|\nabla f(\BM x_t)\|_2^2-\frac{2(1-\alpha)}{L_1\eta^2\beta}(f(\BM x_1)-f_*)-\mu_3\epsilon_1-\mu_4\epsilon_1^2\\
&=\sum_{t=1}^\infty\|\nabla f(\BM x_t)\|_2^2-\frac{2(1-L_1\eta)}{L_1\eta^2\beta}(f(\BM x_1)-f_*)-\mu_3\epsilon_1-\mu_4\epsilon_1^2.}
By Eqs.~\eqref{eq:opt-cum-rhs-prob} \& \eqref{eq:opt-cum-lhs},
\ALN{eq:const-times-cum-sq-grad}{
&\Big(\eta-\frac{L_1\eta^2}2\Big)\sum_{t=1}^\infty\|\nabla f(\BM x_t)\|_2^2=\sum_{t=1}^\infty\eta\|\nabla f(\BM x_t)\|_2^2-\sum_{t=1}^\infty\frac{L_1\eta^2}2\|\nabla f(\BM x_t)\|_2^2\\
\le{}&\sum_{t=1}^\infty\eta\langle\nabla f(\BM x_t),\BM g_t\rangle+\frac{2(1-L_1\eta)}{L_1\eta\beta}(f(\BM x_1)-f_*)+\eta\mu_3\epsilon_1+\eta\mu_4\epsilon_1^2-\sum_{t=1}^\infty\frac{L_1\eta^2}2\|\nabla f(\BM x_t)\|_2^2\\
={}&\sum_{t=1}^\infty\Big(\eta\langle\nabla f(\BM x_t),\BM g_t\rangle-\frac{L_1\eta^2}2\|\nabla f(\BM x_t)\|_2^2\Big)+\frac{2(1-L_1\eta)}{L_1\eta\beta}(f(\BM x_1)-f_*)+\eta\mu_3\epsilon_1+\eta\mu_4\epsilon_1^2\\
\le{}&f(\BM x_1)-f_*+\frac{\sqrt{C_{\rho}}\eta^2L_0L_1^2}2\sqrt s\sum_{t=1}^\infty\epsilon_t+\frac{C_\rho\eta^2L_1^3}8s\sum_{t=1}^\infty\epsilon_t^2+\frac{2(1-L_1\eta)}{L_1\eta\beta}(f(\BM x_1)-f_*)+\eta\mu_3\epsilon_1+\eta\mu_4\epsilon_1^2\\
={}&\Big(1+\frac{2(1-L_1\eta)}{L_1\eta\beta}\Big)(f(\BM x_1)-f_*)+\Big(\frac{\sqrt{C_{\rho}}\eta^2L_0L_1^2}2\sqrt s\cdot\frac{\uppi^2}6+\eta\mu_3\Big)\epsilon_1+\Big(\frac{C_\rho\eta^2L_1^3}8s\cdot\frac{\uppi^4}{90}+\eta\mu_4\Big)\epsilon_1^2
.}
With $\mu_5:=\frac{\frac{\sqrt{C_{\rho}}\eta^2L_0L_1^2}2\sqrt s\cdot\frac{\uppi^2}6+\eta\mu_3}{\eta-\frac{L_1\eta^2}2}$ and $\mu_6:=\frac{\frac{C_\rho\eta^2L_1^3}8s\cdot\frac{\uppi^4}{90}+\eta\mu_4}{\eta-\frac{L_1\eta^2}2}$, dividing Eq.~\eqref{eq:const-times-cum-sq-grad} by $\eta-\frac{L_1\eta^2}2$ gives
\AL{\sum_{t=1}^\infty\|\nabla f(\BM x_t)\|_2^2\le\frac{1+\frac{2(1-L_1\eta)}{L_1\eta\beta}}{\eta-\frac{L_1\eta^2}2}(f(\BM x_1)-f_*)+\mu_5\epsilon_1+\mu_6\epsilon_1^2.}
Let $\epsilon_1:=\frac{-\mu_6+\sqrt{\mu_6^2+4\mu_5\varDelta}}{2\mu_5}>0$, which has $\mu_5\epsilon_1+\mu_6\epsilon_1^2=\varDelta$. It follows that for all $T\ge1$ simultaneously,
\AL{
\min_{t=1,\dots,T}\|\nabla f(\BM x_t)\|_2^2&\le\frac1T\sum_{t=1}^T\|\nabla f(\BM x_t)\|_2^2\le\frac1T\sum_{t=1}^\infty\|\nabla f(\BM x_t)\|_2^2\\
&\le\frac{\frac{1+\tfrac{2(1-L_1\eta)}{L_1\eta\beta}}{\eta-\tfrac{L_1\eta^2}2}(f(\BM x_1)-f_*)+\varDelta}T
.\qedhere}\end{proof}
\ASEC{Experiments (Continued)}{app:exp}

\ASSEC{Additional Implementation Details}{app:exp-hyp}
The hyperparameters of all methods are summarized in Table~\ref{tab:app-exp-hyp}. To ensure fair comparison, we let all methods have the query budget $(q+1)T$ at least that of \Ours{}. Since RS, TPGE, ZO-AdaMM, and GLD use $O(1)$ queries per step, we adjust their number $T$ of steps so that their total number of queries matches that of our \Ours{}; for other methods, we use the same number of queries per step as that of our \Ours{}. For each method, we choose the best step size $\eta$ among $\{0.5,0.2,0.1,0.05,0.02,0.01,\dots\}$. 

\begin{table}[t]
\caption{Summary of hyperparameters ($q$: queries per step; N/A means that $q$ is decided by the algorithm).}
\label{tab:app-exp-hyp}
\begin{small}\begin{center}\begin{tabular}{cl|ccc|ccc|ccc}
\toprule
\multicolumn{1}{c}{\multirow{2}*{\textbf{Type}}}&\multicolumn{1}{c|}{\multirow{2}*{\textbf{Method}}}&\multicolumn{3}{c|}{\textsc{Distance}}&\multicolumn{3}{c|}{\textsc{Magnitude}}&\multicolumn{3}{c}{\textsc{Attack}}\\
&&$\eta$&$q$&$T$&$\eta$&$q$&$T$&$\eta$&$q$&$T$\\
\midrule
\multirow{6}*{\makecell[c]{Full\\Gradient}}
 & RS & 0.0001 & N/A & 3000 & 0.0005 & N/A & 775 & 0.002 & N/A & 10000 \\
 & TPGE & 0.0001 & N/A & 2000 & 0.0005 & N/A & 517 & 0.002 & N/A & 6667 \\
 & RSPG & 0.005 & 58 & 100 & 0.02 & 31 & 50 & 0.1 & 199 & 100 \\
 & ZO-signSGD & 0.001 & 58 & 100 & 0.002 & 31 & 50 & 0.001 & 199 & 100 \\
 & ZO-AdaMM & 0.001 & N/A & 3000 & 0.001 & N/A & 775 & 0.002 & N/A & 10000 \\
 & GLD & 0.001 & 4 & 1200 & 0.002 & 4 & 310 & 0.01 & 4 & 5000 \\
\midrule
\multirow{7}*{\makecell[c]{Sparse\\Gradient}}
 & LASSO & 0.005 & 58 & 100 & 0.02 & 31 & 50 & 0.2 & 199 & 100 \\
 & SparseSZO & 0.01 & 58 & 100 & 0.05 & 31 & 50 & 0.2 & 199 & 100 \\
 & TruncZSGD & 0.02 & 58 & 100 & 0.2 & 31 & 50 & 0.02 & 199 & 100 \\
 & ZORO & 0.02 & 58 & 100 & 0.5 & 31 & 50 & 0.0002 & 199 & 100 \\
 & ZO-BCD & 0.5 & 58 & 100 & 0.5 & 31 & 50 & 0.002 & 199 & 100 \\
 & SZOHT & 0.005 & 58 & 100 & 0.5 & 31 & 50 & 0.2 & 199 & 100 \\
 & GraCe (ours) & 0.5 & N/A & 100 & 0.5 & N/A & 50 & 0.5 & N/A & 100 \\
\bottomrule
\end{tabular}\end{center}\end{small}
\end{table}

\ASSEC{Performance under Inexact Sparsity}{app:exp-unknown-s}
In practice, when the true sparsity $s$ is unknown, we may use an estimated sparsity $\HAT s$ instead of the true sparsity $s$. We show in Table~\ref{tab:app-inexact-s} the normalized objective of our method w.r.t.\ various $\HAT s$ on \textsc{Distance} with true sparsity $s=10$. We can see that even when $\HAT s$ is slightly smaller than the true sparsity $s=10$ (especially when $\HAT s\ge 8$), our method still achieves strong performance. Furthermore, slight overestimation of $\HAT s$ can even improve the performance. The results demonstrate that our method is fairly robust w.r.t.\ inexact sparsity. Therefore, as long as the estimated $\HAT s$ is not too far from the true sparsity $s$, we can expect that our method should still have good performance. 

\begin{table}[t]
\caption{Performance of our method w.r.t.\ various $\HAT s$ on \textsc{Distance} with true sparsity $s=10$.}
\label{tab:app-inexact-s}
\begin{small}\begin{center}\begin{tabular}{cccccccccc}
\toprule
$s=6$&$s=7$&$s=8$&$s=9$&$s=10$&$s=11$&$s=12$&$s=13$&$s=14$&$s=15$\\
\midrule
0.0094&0.0097&0.0053&0.0053&0.0051&0.0045&0.0040&0.0040&0.0056&0.0036\\
\bottomrule
\end{tabular}\end{center}\end{small}
\end{table}

\ASSEC{Performance under Non-Sparse Gradients}{app:exp-nonsparse}
We further investigate the performance of our method when $s=\Theta(d)$ (i.e., the function has non-sparse gradients). Note that \Ours{} reduces to the classic method FDSA \cite{kiefer1952stochastic} when $s=d$, so we also include FDSA in our comparison. Here we use \textsc{Distance} with $d=10000$ and $s=2500$ and report the normalized objective under a query budget of 250000. The results are presented in Table~\ref{tab:app-nonsparse}. From the table, we can see that our \Ours{} and FDSA achieve the best performance and significantly outperform all other methods. Therefore, it is actually beneficial that our method reduces to FDSA when $s=\Theta(d)$. The results are not surprising because it has been recently shown that $\Omega(d)$ queries are required if the gradient is not sparse (Theorem~1, \citealp{alabdulkareem2021information}), which implies that any ZOO method should not have significant advantage over FDSA under non-sparse gradients in the worst case. 

\begin{table}[t]
\caption{Performance under non-sparse gradients.}
\label{tab:app-nonsparse}
\begin{small}\begin{center}\begin{tabular}{ccccccc}
\toprule
FDSA&RS&TPGE&RSPG&ZO-signSGD&ZO-AdaMM&GLD\\
\midrule
\textbf{0.0008}&0.0021&0.0056&0.0012&0.1002&0.0023&0.3392\\
\midrule[0.72pt]
LASSO&SparseSZO&TruncZSGD&ZORO&ZO-BCD&SZOHT&GraCe (ours)\\
\midrule
0.0013&0.0448&0.0017&0.0792&0.0187&0.0015&\textbf{0.0007}\\
\bottomrule
\end{tabular}\end{center}\end{small}
\end{table}

\ASSEC{Validation of Query Complexity}{app:exp-num-queries}
To validate the query complexity $O\big(s\log\log\frac ds\big)$, we present the worst-case numbers of queries under various $d$ and $s$ in Table~\ref{tab:app-worstcase}. From the table, we can see that the number of queries grows very slowly w.r.t.\ $d$. For instance, \Ours{} needs at most 19 queries for $d=10^8$ and $s=1$. This provides strong evidence that the number of queries does scale as $O\big(s\log\log\frac ds\big)$. 

\begin{table}[t]
\caption{Worst-case numbers of queries per step under various $d$ and $s$.}
\label{tab:app-worstcase}
\begin{small}\begin{center}\begin{tabular}{c|ccccccc}
\toprule
\textbackslash&$d=10^2$&$d=10^3$&$d=10^4$&$d=10^5$&$d=10^6$&$d=10^7$&$d=10^8$\\
\midrule
$s=1$&11&15&15&17&19&19&19\\
$s=2$&16&19&22&22&28&28&28\\
$s=3$&26&26&36&36&46&46&46\\
$s=4$&25&31&43&43&55&55&55\\
$s=5$&33&41&57&57&73&73&73\\
\bottomrule
\end{tabular}\end{center}\end{small}
\end{table}

\end{document}